\documentclass{article}

\usepackage{arxiv}

\usepackage[utf8]{inputenc} 
\usepackage[T1]{fontenc}    
\usepackage{xr-hyper}
\usepackage{hyperref}       
\usepackage{url}            
\usepackage{booktabs}       
\usepackage{amsfonts}       
\usepackage{nicefrac}       
\usepackage{microtype}      
\usepackage{lipsum}
\usepackage{graphicx}
\usepackage{amsmath}
\usepackage{verbatim}
\usepackage[title]{appendix}
\usepackage{array,multirow,graphicx}
 \usepackage{float}
 \usepackage{caption}
\usepackage{subcaption}
\usepackage[title]{appendix}
\graphicspath{ {./images/} }

\usepackage{cleveref}

\usepackage{amsthm}

\newtheorem{proposition}{Proposition}

\makeatletter
\newcommand*{\addFileDependency}[1]{
  \typeout{(#1)}
  \@addtofilelist{#1}
  \IfFileExists{#1}{}{\typeout{No file #1.}}
}
\makeatother

\title{On Out-of-sample Embedding in UMAP}

\author{
Mohammad Tariqul Islam$^{1,2}$, Jason W. Fleischer$^{2},\thanks{Corresponding author.}$ \\
$^1$Media Lab, Massachusetts Institute of Technology, Cambridge, MA \\
$^2$Electrical and Computer Engineering, Princeton University, Princeton, NJ \\
\texttt{mhdtariq@mit.edu, jasonf@princeton.edu} \\
}

\date{\vspace{-5ex}}

\begin{document}
\maketitle

\begin{abstract}
Neighbor embedding algorithms reveal correlations in high-dimensional data by constructing an equivalent graph representation in a lower-dimensional space. An increasingly popular algorithm is Uniform Manifold Learning and Projection (UMAP), which uses algebraic topology to map distances between the two spaces. While it works well on many types of data sets, UMAP has trouble adding out-of-sample points to a pre-existing mapping. In particular, UMAP often places new points on the periphery of the found clusters, rather than in their interiors with their correlated neighbors. Here, we overcome this ``repulsion effect'' by optimizing pairwise interactions within the original k-nearest-neighbor graph. Moreover, we show that parameterizing UMAP obtains better embeddings than non-parametric algorithms, particularly as the data gets more complex (e.g., medical images). We also show that the repulsion effect is naturally mitigated when a parameterized UMAP is employed to embed the data. We characterize different UMAP approaches using trustworthiness, nearest neighbor classifiers, and by analyzing attractive and repulsive forces in the embeddings.
\end{abstract}


\section{Introduction}\label{sec:intro}

Neighbor embedding algorithms are unsupervised algorithms that identify groups of related data. They are generally regarded as nonlinear dimensionality reduction methods. The methods define a pairwise metric between points in the high-dimensional space that is used to make a similar graph in a lower-dimensional space. Iterative neighborhood embedding methods, like t-distributed stochastic neighbor embedding (t-SNE)~\cite{maaten2008visualizing} and the recently introduced uniform manifold approximation and projection (UMAP)~\cite{mcinnes2018umap}, have become standard methods due to their ability to scale up to millions of data points. In this section, we give a brief description of the neighbor embedding framework. Afterwards, we focus our attention on UMAP, as it has many significant properties that are not present in the other methods.

\subsection{Neighbor Embedding Framework}\label{sec:ne_framework}

The first step in the neighbor embedding framework is to construct a high-dimensional graph~\cite{hinton2002stochastic}. 
For $n$-dimensional data points $\{\mathbf{x}_i\}$ where $\mathbf{x}_i\in\mathbb{R}^n$ and $i = 1, 2, 3, \dots, N$, the high-dimensional weighted graph is given by a pairwise metric
\begin{align}
p_{ij} = f_H(d_H(\mathbf{x}_i,\mathbf{x}_j)|\{\mathbf{x}_i\}), \label{eq:1}
\end{align}
where $d_H(\cdot,\cdot)$ is a distance metric between two points and $f_H(\cdot)$ 
is a function that describes the weighted relation between points.

The low-dimensional graph is formulated in a similar manner by
\begin{align}
q_{ij} = f_L(d_L(\mathbf{x}_i,\mathbf{x}_j)|\{\mathbf{y}_i\}), \label{eq:2}
\end{align}
where the subscript $L$ refers to the lower $d$-dimensional embedding space and
$\mathbf{y}_i\in\mathbb{R}^d$, with $i=1, 2, 3, \dots, N$.

Finally, a relation between the high-dimensional graph and the low-
dimensional graph is established by optimizing a loss function
\begin{align}
\mathcal{L} = \sum_{i,j} l(p_{ij}, q_{ij}).
\end{align}

This embedding framework has been used extensively in sciences, including RNA sequences~\cite{macosko2015highly,kobak2019art,cao2019single, becht2019dimensionality, packer2019lineage}, representation learning~\cite{badamdorj2022contrastive,islam2024deciphering}, medical applications~\cite{hong2018predicting,fleischer2020late,wang2022medclip,islam2024outlier,vazifeh2025manifold} and communication engineering~\cite{anjinappa2021coverage,xu2022boundary}.

\subsection{Related works}
In the base neighbor embedding algorithm, the correlations and corresponding graph are calculated explicitly, without a parameter representation. Classic methods in this regime include Sammon mapping~\cite{sammon1969nonlinear}, Isomap~\cite{tenenbaum2000global}, locally linear embedding~\cite{roweis2000nonlinear}, Laplacian Eigenmaps\cite{belkin2002laplacian}, and t-SNE~\cite{maaten2008visualizing}. Iterative algorithms, such as t-SNE, have excellent visualization performance but slow implementation speed. 

The factors responsible for the speed of the algorithm are twofold: 1) the initial computation of the dense high-dimensional graph and 2) the normalization operation required in every optimization step.
Several methods have emerged to accelerate the calculations. 
In the graph construction step, Tang et al.~\cite{tang2016visualizing} used random projection trees and neighbor exploration to construct an approximate $k$-NN graph. 
In the optimization step, Maaten~\cite{van2014accelerating} and Yang et al.~\cite{yang2013scalable} used restricted sampling via the Burns-Hut approximation~\cite{barnes1986hierarchical}, Mikolov et al. introduced negative sampling~\cite{mikolov2013distributed}, and Linderman et al.~\cite{linderman2019fast} implemented fast Fourier transform interpolation. 

In UMAP, McInnes et al.~\cite{mcinnes2018umap} skipped the normalization step altogether and formulated the low-dimensional embedding by pairwise interactions alone. 
The resulting algorithm uses the cross-entropy loss function that depends explicitly on attractive and repulsive forces between points. 
The attractive forces preserve the local neighborhoods, while the repulsive forces ensure that points that are originally far apart stay far apart. 
A balance of these forces results in optimal dimensionality reduction~\cite{damrich2021umap}. Extending these, Islam and Fleischer~\cite{islam2025shape} showed that the attractive forces and their different forms guide the structure formation, whereas repulsive forces guide the intercluster distances.

A recent demonstration shows that t-SNE and UMAP are samples from a spectrum of neighbor embeddings~\cite{bohm2022attraction}. This resulted in a search for other algorithms (assuming situated under a spectrum) by varying the loss functions and optimization tricks~\cite{damrich2023contrastive, ijcai2023p406, abe2024nonlinear}.
Other recent trends in the literature include methods for extending the embedding for incremental datasets~\cite{ko2020progressive,senanayake2020self}, manifold alignment~\cite{islam2022manifold}, expanding the neighbor embedding framework for triplets (TriMap)~\cite{amid2019trimap}, and considering neighbors at different scales (PaCMAP)~\cite{wang2021understanding}.

The force balance becomes problematic when new points are considered, i.e., out-of-sample points that weren't included in the original embedding. 
Indeed, the formulation in Section~\ref{sec:ne_framework} requires all the weights to be computed again from scratch when new data is added. 
However, if there is a parametric transformation from high-dimensional data to low-dimensional data, such reconstruction can be avoided. 
Unfortunately, parametric embeddings are hard to obtain due to the costly normalization required during each optimization step. 
One remedy is to divide the whole dataset into small subsets of data and optimized all subsets on the same parameters~\cite{van2009learning,bunte2012general}. 
As UMAP forgoes the normalization operation, we can avoid such steps and parameterize it directly.
Approaches in this direction have been considered further down the network pipeline, e.g., as a learned regularizer in the bottleneck of an autoencoder~\cite{duque2020extendable}, directly in the loss function itself~\cite{sainburg2020parametric}, or by computing weight functions in the intermediate layers of the neural network~\cite{zhou2021deep}. Recent studies focus on global vs. local structure preservation, and tuning the repulsive forces extending PaCMAP principles~\cite{huang2024navigating}.
However, their effect on the visualization performance of out-of-sample embeddings has not been discussed.

\subsection{Contribution}

We focus on understanding and mitigating  out-of-sample embedding errors in UMAP. Overall, the contributions of the paper are given below:

\begin{itemize}
    \item We show that during embedding out-of-sample test points in UMAP, repulsive forces can result in test points accumulating in the periphery of clusters. We denote this phenomenon as `Repulsion Effect'.
    
    \item We analyze placements of out-of-sample points through attractive and repulsive forces. We introduce the attractive force ratio (AFR) and the repulsive force ratio (RFR) as metrics that characterize the force balance in an embedding.
    
    \item We develop several deep neural network techniques that learn the embeddings, enabling parameterized versions of UMAP. We show that parameterized UMAP mitigates the adverse effects of repulsion and gives better embeddings.
    
\end{itemize}

To the best of our knowledge, this is the first paper that thoroughly analyzes the test embedding performance of the UMAP algorithm, identifies the core issue with test embeddings, and provides solutions to it. 
We also uncover the conditions in which parameterized embeddings work best, viz., fields with complex data that are generated continuously, such as those in biomedical and healthcare settings. For a general discussion on modern parametric approaches, see~\cite{sainburg2020parametric} and~\cite{huang2024navigating}.

We organize the paper as follows. Section~\ref{sec:UMAP_SEC} describes the UMAP algorithm and the approach it takes to embed out-of-sample points. We also formalize the `repulsion effect'. In Section~\ref{sec:parameterizing}, we describe several methods to parameterize UMAP. Section~\ref{sec:experiments_pumap} provides the experimental setup, evaluation metrics, and the results on three representative datasets. Finally, Section~\ref{sec:conclusions} concludes the paper. Code and trained networks are available at \url{https://github.com/tariqul-islam/OOS_UMAP/}.

\section{Uniform Manifold Approximation and Projection (UMAP)} \label{sec:UMAP_SEC}

In this section, we review the UMAP algorithm~\cite{mcinnes2018umap}. We then focus on embedding out-of-sample points, with an emphasis on the competition between attractive and repulsive forces. 

\subsection{Embedding Algorithm}\label{sec:UMAP_ALGO}\label{sec:umap_def}
In UMAP, a high-dimensional weighted relation between points is constructed using a $k$-NN graph. For each data point $\mathbf{x}_i$, the relation is given by
\begin{align}
    p_{i|j} =
        \begin{cases}
            \exp{\left(-\frac{d_H(\mathbf{x}_i,\mathbf{x}_j)-\rho_i}{\sigma_i}\right)} & \text{if } x_j\in \text{KNN}(\mathbf{x}_i,k) \\
            0 & \text{otherwise}
        \end{cases}, \label{eq:UMAP_HIGH_DIM}
\end{align}
where $\text{KNN}(\mathbf{x}_i,k)$ is the set of $k$-nearest neighbors of the point $\mathbf{x}_i$, $\rho_i = \min_{\mathbf{x}_j\in \text{KNN}(\mathbf{x}_i,k)} d_H(\mathbf{x}_i,\mathbf{x}_j)$ and $\sigma_i$ is a scaling parameter set such that $\sum_j  p_{i|j} = \log_2(k)$. The parameter $\rho_i$ ensures that the point $\mathbf{x}_i$ has strong connectivity ($p_{i|j}=1$) to at least one of the nearest neighbors, and the scaling parameter $\sigma_i$ ensures the uniform manifold approximation. The adjacency matrix of the high-dimensional undirected weighted graph is obtained by (the probabilistic t-conorm), 
\begin{align}
    p_{ij} = p_{i|j} + p_{j|i} - p_{i|j} \times p_{j|i}.
\end{align}

The low-dimensional weight is given by a differentiable function
\begin{align}
    q_{ij} = 
    \frac{1}{1+a(||\mathbf{y}_i-\mathbf{y}_j||_2^2)^b}, \label{eq:low_dim_umap}
\end{align}
where the parameters $a$ and $b$ determine how crowded the low-dimensional points become after embedding. These two parameters are chosen by fitting $q_{ij}$ to 
\begin{align}
\tiny
    \Psi(\mathbf{y}_i,\mathbf{y}_j) = 
    \begin{cases}
    1 & \text{if } ||\mathbf{y}_i-\mathbf{y}_j||_2<m_d \\
    \exp(-(||\mathbf{y}_i-\mathbf{y}_j||_2 - m_d)) & \text{otherwise}
    \end{cases},
\end{align}
where $m_d$ is a user-defined parameter that regulates the minimum distance between two low-dimensional points. 

In order to obtain the final low-dimensional embedding, the UMAP algorithm aims to minimize the cross-entropy loss function
\begin{align}
    \mathcal{L} = 
    \sum_{i,j} p_{ij} \log\left( \frac{p_{ij}}{q_{ij}} \right) +
    \left(1-p_{ij}\right) \log\left( \frac{1-p_{ij}} {1-q_{ij}} \right) \label{eq:crs_loss_fun}
\end{align}
from an initialization of the embedding $\{\mathbf{y_i}\}$. 
The first term of the loss function provides attractive forces and the second term provides repulsive forces. 
However, computing attractive-repulsive forces for all pairs of points is costly. 
Thus, UMAP takes a negative sampling approach~\cite{mikolov2013distributed,tang2016visualizing} (for a detailed treatment of the effect of negative sampling on the loss function see~\cite{damrich2021umap}). 
In the optimization cycle for each edge with $p_{ij}$ in the $k$-NN graph, termed a positive edge, a number of random edges are sampled, termed negative edges. 
The attractive forces are minimized for the positive edges, and the repulsive forces are minimized for the negative edges. 
For the negative edges, the term $\left(1-p_{ij}\right)$ is assumed to be $1$ (this assumption is often justified as most of the $p_{ij}$ entries are zeros). 
In order to obtain a fast convergence, in each iteration, the method samples one positive edge with probability $p_{ij}$ and applies the attractive force. 
After that, it randomly samples $n_s$ negative edges and applies repulsive forces to each of the negative samples. 
The parameter $n_s$, known as the negative sampling rate, is typically set to $5$.

\subsection{Embedding Out-of-sample Points in an Existing Embedding}\label{sec:embed_out_of_sample}

The simplest rule-based method optimizes the low-dimensional embedding of the new point alone while keeping the existing embedding intact~\cite{berman2014mapping,polivcar2019embedding}. 
An adaptation for UMAP, using a $k$-NN graph of the sample concerning the training dataset, appears in the UMAP software itself~\cite{mcinnes2018umap-software}. However, to our knowledge, no description of the method is available. 
For completeness, we give such a description below. 

Let $\mathbf{u}\in\mathbb{R}^n$ be the out-of-sample test point and $\mathbf{v}\in\mathbb{R}^d$ be its UMAP embedding. The high-dimensional weighted graph for this point is given by replacing the in-sample point $x_i$ in Eq.~(\ref{eq:UMAP_HIGH_DIM}) with the point~$u$
\begin{align}
    p_{\mathbf{u}|j} = 
    \begin{cases}
        \exp\left(-\frac{d_H(\mathbf{u},\mathbf{x}_j)- \rho_\mathbf{u}}{\sigma_{\mathbf{u}}}\right) & \text{if } \mathbf{x}_j \in \text{KNN}(\mathbf{u},k)\\
        0 & \text{otherwise}
    \end{cases}, \label{eq:UMAP_REF_HIGH_DIM}
\end{align}
which is normalized to $p_{\mathbf{u}j}$ by
\begin{align}
    p_{\mathbf{u}j} = \frac{p_{\mathbf{u}|j}}{\sum_j  p_{\mathbf{u}|j}}.
\end{align}
Next, we define the low-dimensional weight as 
\begin{align}
    q_{\mathbf{v}j} = 
    \frac{1}{1+a(||\mathbf{v}-\mathbf{y}_j||_2^2)^b}. \label{eq:low_dim_umap_test}
\end{align} 
From an initialization of $\mathbf{v}$, we aim to minimize the cross-entropy loss function with respect to $\mathbf{v}$:
\begin{align}
    \mathcal{L} = 
    \sum_{j} p_{\mathbf{u}j} \log\left( \frac{p_{\mathbf{u}j}}{q_{\mathbf{v}j}} \right) +
    \left(1-p_{\mathbf{u}j}\right) \log\left( \frac{1-p_{\mathbf{u}j}} {1-q_{\mathbf{v}j}} \right). \label{eq:test_loss_function}
\end{align}
Similar to the embedding algorithm in Section~\ref{sec:umap_def}, the attractive and the repulsive forces are applied to a single positive edge and $n_s$ negative edges. The negative sampling parameter, $n_s$, primarily controls the amount of repulsive forces during optimization.

\begin{figure}[t]
    \centering
    \includegraphics[width=0.49\textwidth]{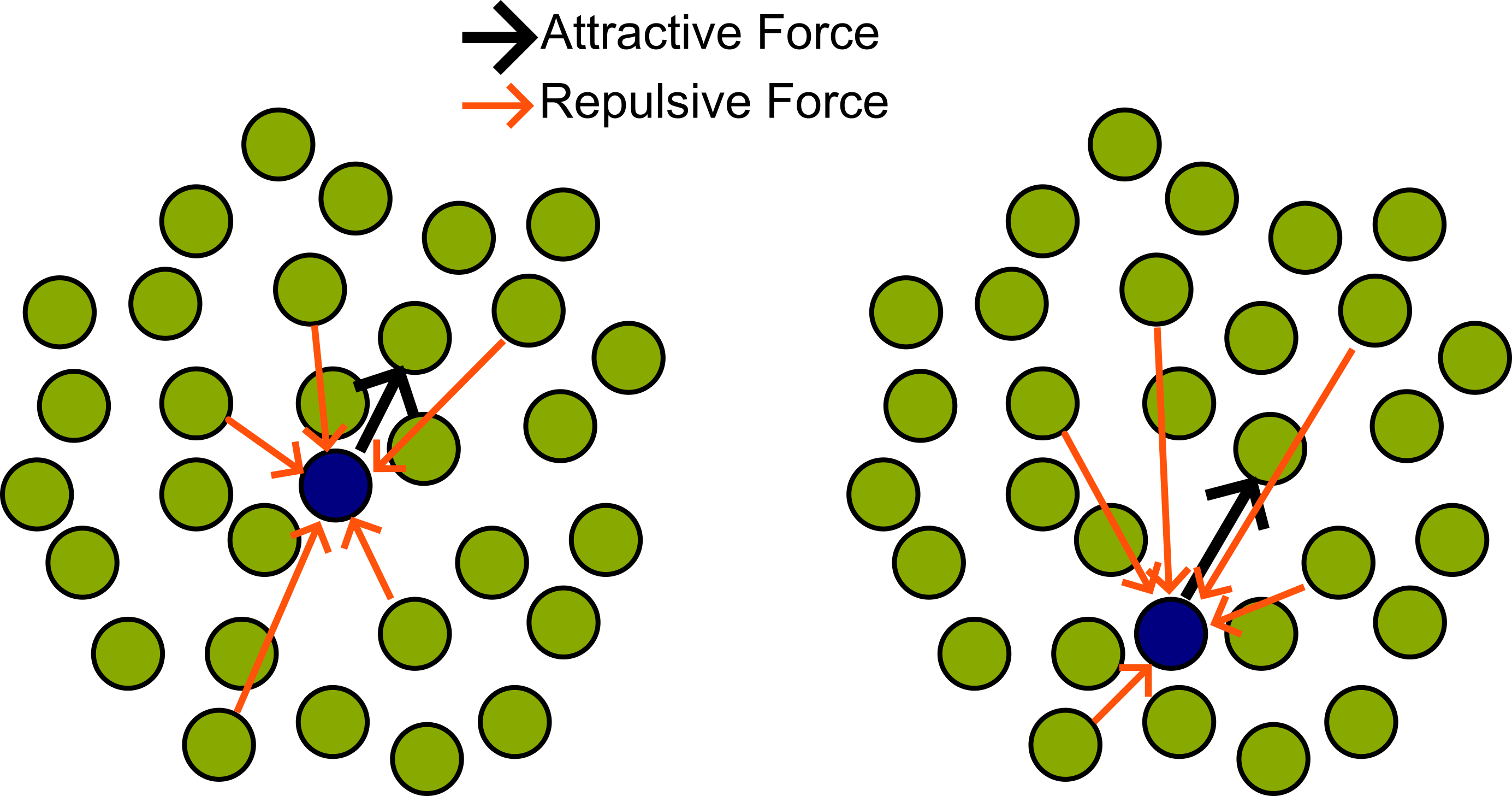} \\
    (a) \hspace{0.24\textwidth} (b) \\
    \caption{Influence of attractive force from a neighbor and repulsive forces from negative samples on an out-of-sample test point (blue). (a) When the point is in the interior of the cluster, the negative samples often surround the test point, balancing the overall forces. (b) However, when the point is near the periphery, the negative samples often provide (an overall) directional force that may cause the test point to be pushed out of the cluster, resulting in what we call the `repulsion effect'.}
    \label{fig:rep_eff_sketch}
\end{figure}

\subsection{Repulsion Effect}\label{sec:repulsion_effect}

Two problems arise during the embedding of out-of-sample points: 1) the lack of normalization while approximating the low-dimensional structure in Eq.~(\ref{eq:low_dim_umap_test}) and 2) the stochastic nature of the optimization from the negative sampling rate parameter $n_s$. 
Since there is no normalization, each update from optimizing Eq.~(\ref{eq:test_loss_function}) works as a freewheel, i.e., there is no mitigating factor that reduces overshooting of the out-of-sample point from a desired location. 
Secondly, as the negative sampling parameter $n_s$ is typically $>1$, the out-of-sample point undergoes more repulsion than attraction.
The update equations of an out-of-sample point $v$ with the intermediate steps are as follows:
\begin{align}
    v^{(t+a)} &= v^{(t)} - \eta f_a(v^{(t)},y^{(a)}), \\
    v^{(t+a+1)} &= v^{(t+a)} + \eta f_r(v^{(t+a)}, y^{(1)}), \\
    \dots, & \nonumber \\
    v^{(t+1)} = v^{(t+a+n_s)} &= v^{(t+a+n_s-1)} + \eta f_r(v^{(t+a+n_s-1)}, y^{(n_s)}),
\end{align}
where $t$ is the iteration number, $v^{(t+a)}$ is the update due to attractive forces from the neighbor $y^{(a)}$,  $\{v^{(t+a+1)},\dots,v^{(t+a+n_s)}\}$ are updates due to repulsive forces from the negative samples $\{y^{(1)},\dots,y^{(n_s)}\}$, $\eta$ is the learning rate, and $f_a$ and $f_r$ define the attractive and repulsive forces, respectively.

These multiple repulsions are less of a problem if the repulsive forces are distributed evenly in all directions, which prevents the overall effect from being directional (Fig.~\ref{fig:rep_eff_sketch}(a)).
If most of the repulsive forces are in the same direction, however (e.g., near a boundary), the resultant repulsive force pushes the point towards the periphery of the original embedding (Fig.~\ref{fig:rep_eff_sketch}(b)). 
Strong directional repulsive forces may even shoot the test point out of the cluster boundary.  
We dub this phenomenon the ``repulsion effect''. 
On the other hand, if $n_s=0$ for an out-of-sample point, the attractive force will dominate and manifest the point in the middle of the embedding, close to its nearest neighbor partner. 
Consequently, to obtain an optimal placement, the values of $k$ and $n_s$ require tuning for each out-of-sample point, which is time-consuming. 
Moreover, the user typically does not know the best placement of the out-of-sample point, making manual tuning a futile effort.  

\section{Parameterizing UMAP}\label{sec:parameterizing}

To parameterize UMAP, we define $\mathbf{y}_i=f_\theta(x_i)$, where $\{\theta\}$ is a set of controlling parameters and $f_\theta:\mathbb{R}^n\to\mathbb{R}^d$ is a transformation from the high-dimensional input to the low-dimensional output. 
We assume that the function $f_\theta(\cdot)$ is differentiable so that it can be optimized using gradient methods. 
There are many approaches to obtaining a parameterized version of UMAP, e.g., optimizing the cross-entropy loss function (as in the original UMAP algorithm) or the mean-squared error. 
In what follows, we denote the class of parameterized versions of UMAP as P-UMAP.

\subsection{Minimizing Cross Entropy Error: P-UMAP-CE}

Here, we formulate the optimization as a `learning UMAP from first principles' concept. First, we define the low-dimensional weights similar to Eq.~(\ref{eq:low_dim_umap}):
\begin{align}
    q^{\{\text{P}\}}_{ij} = 
    \frac{1}{(1+a(||f_\theta(\mathbf{x}_i)-f_\theta(\mathbf{x}_j)||_2^2)^b)}.
\end{align}
Then, we minimize the cross-entropy loss function 
\begin{align}
    \mathcal{L}^{\{\text{P}\}}(\theta) =
    \sum_{i,j} p_{ij} \log\left( \frac{p_{ij}}{q^{\{\text{P}\}}_{ij}} \right) +
    \left(1-p_{ij}\right) \log\left( \frac{1-p_{ij}} {1-q^{\{\text{P}\}}_{ij}} \right)\label{eq:ce_loss_pumap}
\end{align}
with respect to $\theta$ to obtain a low-dimensional embedding. This is analogous to the formulation of~\cite{sainburg2020parametric}.

\subsection{Regression using Mean Squared Error: P-UMAP-MSE}

In this approach, we obtain a reference embedding, $\{y_i\}$, by running the UMAP algorithm. 
After that, we learn a parameterized approximation simply by minimizing the mean-squared error:
\begin{align}
    \mathcal{L}^{\{\text{MSE}\}}(\theta) = \sum_{i} ||f_\theta(\mathbf{x}_i) - \mathbf{y}_i ||_2^2,
\end{align}
yielding a straightforward, inexpensive approach.

\subsection{Combination Approach: P-UMAP-CEMSE}
In this approach, we obtain a parameterized version of UMAP by minimizing both the cross-entropy error and the mean-squared error in the combined loss function
\begin{align}
    \mathcal{L}^{\{\text{P2}\}}(\theta) = \mathcal{L}^{\{\text{P}\}}(\theta) + \mathcal{L}^{\{\text{MSE}\}}(\theta).
\end{align}
This joint function combines the quick learning of MSE error with the data-driven structure of CE.

\section{Experiments}\label{sec:experiments_pumap}

We conduct numerical experiments to evaluate the placement of out-of-sample points in the embeddings.
We first describe the experimental setup and the metrics used to evaluate the embeddings. 
We then describe the visualization performance in different datasets. 
Finally, we characterize the placement of new points in terms of attractive and repulsive forces.

\subsection{Experimental Setup}\label{sec:experimental_setup}\label{sec:neural_network_architecture}
For parameterizing UMAP, we employed a seven-layer, fully-connected, feed-forward neural network. We describe each layer as:
\begin{align}
    f_l(\mathbf{x}) = \sigma(\mathbf{\gamma}_l (\mathbf{W}_l \mathbf{x} + \mathbf{b_l})),
\end{align}
where $l$ is the layer number, $\sigma(\cdot)$ is the sigmoid activation function, and $\mathbf{W}_l$, $\mathbf{b}_l$, and $\mathbf{\gamma}_l$ are learnable weights, biases, and scaling parameters, respectively (forming the parameters, $\theta$, of the network). The number of output neurons of the layers is 500, 300, 200, 100, 100, 100, and $d$, respectively. For visualization, $d=2$ as usual.

The Adam~\cite{kingma2014adam} optimization algorithm is used for its faster and more efficient approach. 
The learning rate is set to $0.001$ initially and then decreased successively by 1/10th every five epochs. 
To optimize the network, we followed the same negative sampling principle of UMAP and used $n_s=5$ in all experiments. 
Depending on the dataset and learning mechanism, we changed the number of epochs for training. 
For P-UMAP-MSE and P-UMAP-CEMSE, optimizing for 20 epochs proved sufficient. 
For P-UMAP-CE, which is a slow learner, we used 40 epochs. 
Any changes from these settings are discussed explicitly. 

In what follows, we performed the experiments in Python 3.6.9 and utilized UMAP software v0.4.6~\cite{mcinnes2018umap-software}  for non-parametric embeddings and PyTorch software package~\cite{paszke2019pytorch} for the deep learning experiments. 
Our codes for different metrics use numpy~\cite{harris2020array}, numba v0.51~\cite{lam2015numba}, and scikit-learn~\cite{scikit-learn}.

\subsection{Evaluation Metrics}\label{sec:evaluation_metrics}

A single evaluation metric often fails to capture the diverse properties of an embedding. 
Therefore, we evaluate in a multitude of ways for a comprehensive assessment. 
Following conventional practice, we assess the dimensionality reduction based on two metrics: the trustworthiness ($T_\kappa$)~\cite{venna2001neighborhood} and the $k$-NN classifier error~\cite{cover1967nearest}. To quantify the repulsion effect, we analyze the accumulation of points at the periphery of a cluster. We characterize this by 1) counting the number of points that accumulate in the periphery (accumulation) and 2) quantifying the balance of attractive and repulsive forces around the periphery in terms of force ratios. These metrics are described in detail below.

\subsubsection{Trustworthiness}
The trustworthiness metric measures whether the local structure in the low-dimensional space conforms with that of the high-dimensional data:
\begin{align}
    T_\kappa = 1 - \frac{2}{n\kappa(2n-3\kappa-1)}  \sum_{i=1}^n \sum_{y_j\in \text{KNN}(y_i,\kappa)} \max(0,(r(i,j)-\kappa)),
\end{align}
where $KNN(y_i,\kappa)$ is the $k$-NN graph in the low-dimensional space and $r(i,j)$ is the rank of $x_j$ in the high-dimensional $k$-NN graph~($KNN(x_i,\kappa)$).  
The parameter $\kappa~(<\frac{N}{2})$ refers to the number of nearest neighbors considered in the low dimension and works as a scale of the local neighborhood. 
To understand the structure preservation at different scales, we consider values of $\kappa$ of $5$, $30$, and $100$. 
Moreover, we consider the trustworthiness of training data alone, which is traditional, and the trustworthiness of the combined training and test data to see the relative rank of the out-of-sample test points. 
Previous literature~\cite{sainburg2020parametric} computed trustworthiness by sampling 10,000 points from the embedding; here, we used all the available samples in each dataset to paint a more accurate picture.

\subsubsection{$k$-NN Classifier}
We employ 1-NN and 5-NN classifiers. The 1-NN classifier assigns the label of the nearest neighbor to the test point. 5-NN classifier considers 5 nearest neighbors and assigns the majority of the labels to the test point.

\subsubsection{Accumulation}
Given a cluster $\mathcal{C}$ and the associated test points $\{\mathbf{v}_\mathcal{C}\}$, we count accumulation $\zeta$ by
\begin{align}
    \zeta = | \{ \mathbf{v} | \mathbf{v}\in\mathbf{v}_\mathcal{C}, \mathbf{v}\notin P \} |,
\end{align}
where $P$ represents a convex polytope within $\mathcal{C}$. To obtain the desired convex polytope, we first compute the convex hull of the cluster. Subsequently, we adjust the vertices of the convex hull by moving them towards the center (i.e., reducing their distance from the center) by 5\%. For further details, see Appendix~\ref{suppsec:accumulation}.

\subsubsection{Force Ratios}\label{sec:force_ratios}

The forces are inherently $d$-dimensional vectors. To define the force ratios, however, we only work with scalars. The total attractive force on an embedded test point $\mathbf{v}$ (and corresponding high-dimensional point $\mathbf{u}$) is given by
\begin{align}
    F_a(\mathbf{v}) = \Bigl|\Bigl|\sum_j p_{\mathbf{u}|j} f_a(\mathbf{v},\mathbf{y}_j)\Bigr|\Bigr| , \text{~~where~~} \mathbf{x}_j\in \text{KNN}(\mathbf{u},k),
\end{align}
where $f_a(\mathbf{v},\mathbf{y}_j)$ is the pairwise attractive force between embedded points $\mathbf{v}$ and $\mathbf{y}_j$, and $p_{\mathbf{u}|j}$ is the weight of the point $\mathbf{u}$ in the high-dimensional graph.
Similarly, the total repulsive force on the point $\mathbf{v}$ is
\begin{align}
    F_r(\mathbf{v}) = \Bigl|\Bigl|\sum_j p_{\mathbf{u}|j} f_r(\mathbf{v},\mathbf{y}_j)\Bigr|\Bigr|, \text{~~where~~} \mathbf{x}_j\in \text{KNN}(\mathbf{u},k),
\end{align}
where $f_r(\mathbf{v},\mathbf{y}_j)$ is the pairwise repulsive force between $\mathbf{v}$ and $\mathbf{y}_j$.

With these averages, we define the force ratio as the ratio of forces of two sets of embedded points $S_1$ and $S_2$. The attractive force ratio (AFR) is the ratio of the average attractive forces on points in $S_1$ and the points in $S_2$ as:
\begin{align}
    \text{AFR}(S_1||S_2) = \frac{\sum_{\mathbf{v}\in S_1} F_a(\mathbf{v})}{\sum_{\mathbf{z}\in S_2} F_a(\mathbf{z})}. \label{eq:afr}
\end{align}
In the ideal case, the AFR should be close to 1.0, indicating the balance of force around the neighborhood of the points of the sets. As a result, it works as a metric of force balance. To match intuition, one can define new metrics such as $1-\text{AFR}$ or $|1-\text{AFR}|$ so that the balance of force occurs at $0$. Note that AFR is not a symmetric function as $\text{AFR}(S_1||S_2)$ and $\text{AFR}(S_2||S_1)$ are not typically equal if $S_1\neq S_2$.

Similarly, we define the repulsive force ratio (RFR) as:
\begin{align}
    \text{RFR}(S_1||S_2) = \frac{\sum_{\mathbf{v}\in S_1} F_r(\mathbf{v})}{\sum_{\mathbf{z}\in S_2} F_r(\mathbf{z})}. \label{eq:rfr}
\end{align}
A value of 1.0 indicates the balance of average forces between the sets. Similar to AFR, RFR is also asymmetric.

The force ratios in Eq.~(\ref{eq:afr}) and~(\ref{eq:rfr}) are abstract and can be implemented for any algorithm that relies on attractive and repulsive forces. For the UMAP algorithm, the attractive force $f_a$ and the repulsive force $f_r$ are obtained by differentiating the attractive and repulsive terms of the cross-entropy loss function (Eq.~(\ref{eq:crs_loss_fun})), respectively:
\begin{align}
    f_a(\mathbf{v},\mathbf{y}) = \frac{2ab(||\mathbf{v}-\mathbf{y}||_2^2)^{b-1}}{1 + a * (||\mathbf{v}-\mathbf{y}||_2^2)^b} (\mathbf{v}-\mathbf{y}), \\
    ~\nonumber\\
    f_r(\mathbf{v},\mathbf{y}) = \frac{2 b}{ ||\mathbf{v}-\mathbf{y}||_2^2 (1 + a (||\mathbf{v}-\mathbf{y}||_2^2)^b) } (\mathbf{v}-\mathbf{y}).
\end{align}

\subsection{Comparing Visualizations}\label{sec:comparing_visualization}

In this section, we assess the visualization performance of different approaches on three datasets: MNIST, chest x-rays, and clinical data from emergency departments. We use trustworthiness ($T_\kappa$), k-NN classifier error, and accumulation as figures of merit.
We compare our results to the original UMAP algorithm by varying the nearest neighbor ($k$) and the negative sampling ($n_s$) parameters.

\subsubsection{MNIST}

\begin{figure} [t]
    {\centering
    \includegraphics[width=0.22\linewidth]{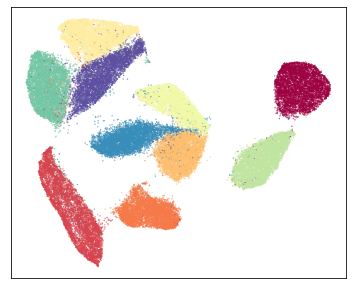}
    \includegraphics[width=0.22\linewidth]{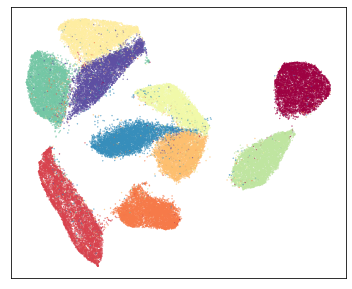}
    \includegraphics[width=0.22\linewidth]{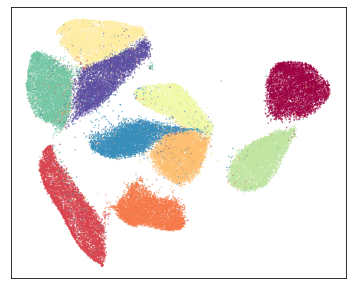}
    \includegraphics[width=0.22\linewidth]{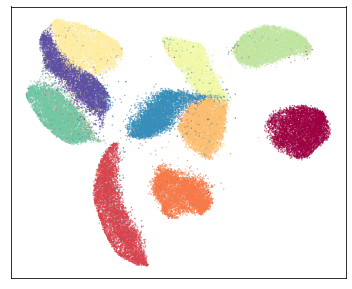} 
    \includegraphics[width=0.037\linewidth]{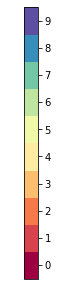} \\}
    \hspace{0.14\textwidth} (a) \hspace{0.19\textwidth} (b) \hspace{0.19\textwidth} (c) \hspace{0.19\textwidth} (d)
    \caption{Two-dimensional embedding of the MNIST dataset using (a) UMAP ($T_5: 0.9523$), (b) P-UMAP-MSE ($T_5: 0.9523)$, (c) P-UMAP-CEMSE ($T_5: 0.9550$), and (d) P-UMAP-CE ($T_5:0.9555$). Note that, (b) and (c) use (a) as the reference embedding. The figures (b)-(d) show that all the parameterized method essentially obtain comparable embedding to the original UMAP method in (a).}
    \label{fig:mnist_embedding}
\end{figure}

\begin{table*}[t]
 \caption{Results on MNIST dataset in terms of Trustworthiness, k-NN classifier, and accumulation.}
 \vspace{0.05in}
  \centering
  \resizebox{\textwidth}{!}{
  \begin{tabular}{l|c|c|c|c|c|c|c|c|c}
    \toprule
             & \multicolumn{3}{|c|}{Trustworthiness (Train)} & \multicolumn{3}{|c|}{Trustworthiness (Train+Test)} & \multicolumn{2}{|c|}{k-NN classifier} & Accumulation, $\zeta$ \\
             \cmidrule(r){2-10}
             & $T_{5}$ & $T_{30}$ & $T_{100}$ & $T_{5}$ & $T_{30}$ & $T_{100}$ & $1$-NN Error & $5$-NN Error & Label 1 \\
    \midrule
    UMAP \tiny{($n_s=5$)} & $0.9523$ & $0.9518$ & $0.9502$ & $0.9523$ & $0.9519$ & $0.9505$ & $7.41\%$ & $4.74\%$ & $104$ \\
    UMAP \tiny{($n_s=3$)} & -- & -- & -- & $0.9523$ & $0.9519$ & $0.9517$ & $7.41\%$ & $4.80\%$ & $75$ \\
    UMAP \tiny{($n_s=1$)} & -- & -- & -- & $0.9523$ & $0.9518$ & $0.9505$ & $7.13\%$ & $4.74\%$ & $48$ \\
    UMAP \tiny{(Train+Test)} & $0.9537$ & $0.9531$ & $0.9517$ & $0.9542$ & $0.9533$ & $0.9521$ & $\mathbf{6.14}\%$ & $\mathbf{3.64}\%$ & $49$ \\
    P-UMAP-MSE    & $0.9523$ & $0.9517$ & $0.9491$ & $0.9467$ & $0.9467$ & $0.9456$ & $10.24\%$ & $7.23\%$ & $53$ \\
    P-UMAP-CEMSE & $0.9550$ & $0.9544$ & $0.9522$ & $0.9538$ & $0.9534$ & $0.9523$ & $6.67\%$ & $4.19\%$ & $46$ \\
    P-UMAP-CE     & $\mathbf{0.9555}$ & $\mathbf{0.9553}$ & $\mathbf{0.9535}$ & $\mathbf{0.9547}$ & $\mathbf{0.9546}$ & $\mathbf{0.9538}$ & $8.20\%$ & $5.02\%$ & $51$ \\
    \bottomrule
  \end{tabular}
  \label{tab:mnist_table}
  }
\end{table*}

\begin{figure*} [t]
    {\centering
    \includegraphics[width=0.3\textwidth]{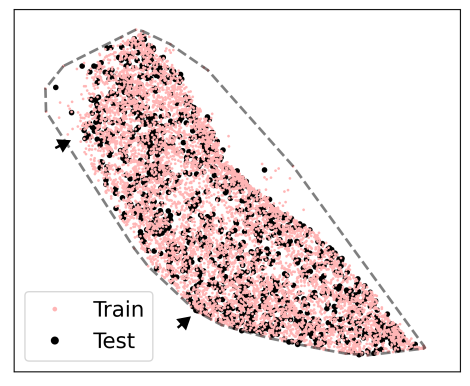}
    \includegraphics[width=0.3\textwidth]{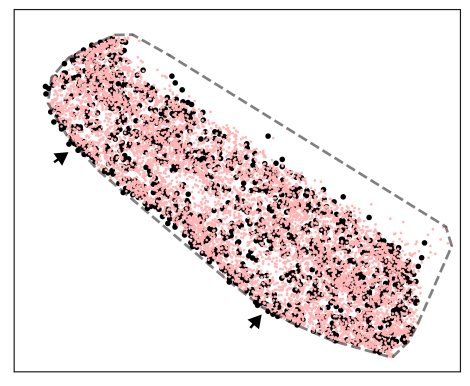}
    \includegraphics[width=0.3\textwidth]{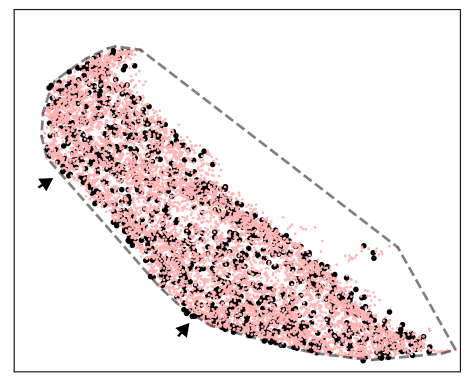}
     \\
    (a) \hspace{0.28\textwidth} (b) \hspace{0.28\textwidth} (c) \\}
    {\centering
    \includegraphics[width=0.3\textwidth]{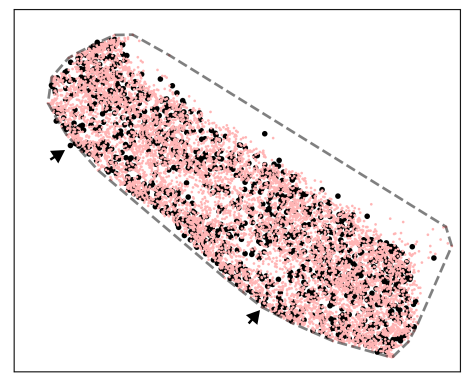}
    \includegraphics[width=0.3\textwidth]{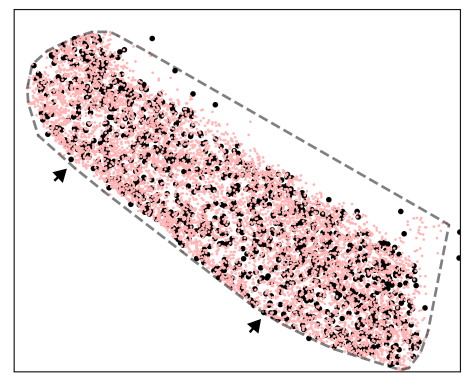}
    \includegraphics[width=0.3\textwidth]{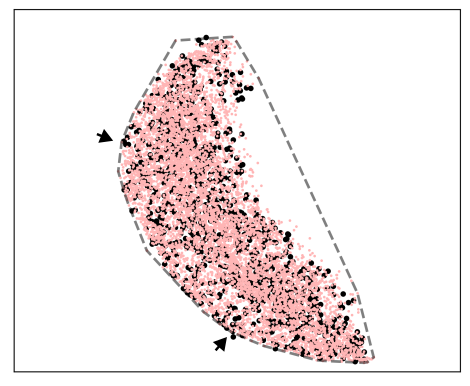} \\
    (d) \hspace{0.28\textwidth} (e) \hspace{0.28\textwidth} (f) \\}
    \caption{Two-dimensional embedding of MNIST data, zoomed to show the cluster labeled $1$. A grey dashed line shows the convex hull of this cluster identified by HDBSCAN. (a) UMAP algorithm with $k=30$ with both training and test data embedded together ($\zeta=49$). (b,c) UMAP algorithm for training data followed by the test data considered as out-of-sample data using the method in Section~\ref{sec:embed_out_of_sample} with (b) $k=30$ ($\zeta=104$) and (c) $k=15$ ($\zeta=105$). (d) UMAP algorithm as in (b), but $n_s$ is set to $3$ for test embedding ($\zeta=75$). (e) P-UMAP-CEMSE ($\zeta=46$), and (f) P-UMAP-CE ($\zeta=51$) for the same settings as (b). Both (b) and (c) show an accumulation of test points (in black) at the periphery of the cluster (indicated using a black arrow) due to the `repulsion effect'. This accumulation is not present in the rest.}
    \label{fig:mnist_repulsion_effect}
\end{figure*}

We start with the standard MNIST dataset of handwritten digit images~\cite{lecun2010mnist}, a commonly used benchmark for evaluating neighbor embedding methods. 
We set the nearest-neighbor parameter $k$ to $30$ and the minimum-distance parameter $m_d$ to $0.25$. 
The training embedding uses the default $n_s=5$ (Fig.~\ref{fig:mnist_embedding}(a)). 
After that, we varied $n_s$ to $5$, $3$, and $1$ for test embedding to see the effect of repulsive forces. 
We also obtained embedding for combined training and test data (UMAP {\tiny{(Train+Test)}}), i.e., the entire dataset, to see how UMAP would ``naturally'' place the out-of-sample test points. 
Table~\ref{tab:mnist_table} summarizes the numerical results for these and the parameterized UMAP embeddings (Fig.~\ref{fig:mnist_embedding}(b)-(d)).
Since MNIST has a clearly defined manifold structure, all the embeddings identify the same clusters.
For the non-parametric UMAP (by varying $n_s$), the trustworthiness of the overall embedding is better for $T_{100}$ when $n_s=3$, indicating better preservation of a larger structure. 
However, the trustworthiness values are even higher when the training and test data are embedded together (UMAP~{\tiny{(Train+Test)}}). 
The gap between this and the UMAP implementation illustrates the misplacements due to the repulsion effect.
For the parameterized embeddings, P-UMAP-MSE achieves a lower trustworthiness than UMAP, while the P-UMAP-CEMSE and P-UMAP-CE embeddings provide higher trustworthiness. 
In terms of the k-NN classifier error, the embedding of UMAP~{\tiny{(Train+Test)}} shows the best performance. 
Interestingly, the P-UMAP-CEMSE provides the 2nd-best k-NN error (we will see later that the CE loss alone achieves the best result as the datasets become more complex). 

The previous metrics reduce the embeddings to a single number and do not characterize the distribution of the points (or the accumulation due to the repulsion effect). While the repulsion effect is not severe or sometimes noticeable in the MNIST dataset, it is present. 
To demonstrate, we first identify the cluster primarily consisting of Label 1 using HDBSCAN~\cite{mcinnes2017hdbscan} and compute the accumulation $\zeta$ (Table~\ref{tab:mnist_table}). UMAP~{\tiny{(Train+Test)}} gives a baseline accumulation $\zeta=49$. 
For UMAP~{\tiny{($n_s=5$)}}, $\zeta=104$ (more than double the nominal value), indicating that the repulsive forces have pushed more test points to the periphery of the cluster. For $n_s=3$, accumulation is higher than the baseline ($\zeta=75>49$), whereas for $n_s=1$, it matches the baseline. On the other hand, all the parameterized UMAP implementations are relatively close to the nominal value, with $\zeta=50$.

\begin{table}[t]
    \centering
    \caption{Time required to embed each out-of-sample point of MNIST data for $k=30$ for UMAP algorithm, and a trained neural network.}
    \label{tab:run_time_mnist}
    \vspace{0.05in}
    \begin{tabular}{c|c|c}
    \hline
        Method  & Sequential Time (ms) & Parallel Time (ms)  \\
        \hline
        UMAP & $489.63$ & $5.79$ \\
        Neural Network & $0.53$ & $0.08$ \\
        \hline
    \end{tabular}
    
\end{table}

\begin{table*}[t]
 \caption{Results on chest x-ray embedding in terms of Trustworthiness, k-NN classifier, and accumulation.} \label{tab:pneumonia_results}
 \vspace{0.05in}
  \centering
\resizebox{\textwidth}{!}{
  \begin{tabular}{l|c|c|c|c|c|c|c|c|c}
    \toprule
             & \multicolumn{3}{|c|}{Trustworthiness (Train)} & \multicolumn{3}{|c|}{Trustworthiness (Train+Test)} & \multicolumn{2}{|c|}{k-NN classifier} & Accumulation, $\zeta$ \\
             \cmidrule(r){2-10}
             & $T_{5}$ & $T_{30}$ & $T_{100}$ & $T_{5}$ & $T_{30}$ & $T_{100}$ & $1$-NN Error & $5$-NN Error  \\
    \midrule
     UMAP \tiny{($n_s=5$)}    & $0.8345$ & $0.8332$ & $0.8315$ & $0.8335$ & $0.8326$ & $0.8314$ &  $23.85\%$ & $19.05\%$ & $245$ \\
     UMAP \tiny{($n_s=3$)}  & -- & -- & -- & $0.8340$ & $0.8330$ & $0.8315$ & $25.05\%$ & $17.85\%$ & $79$\\
     UMAP {\tiny($n_s=1$)} & -- & -- & -- & $0.8321$ & $0.8312$ & $0.8301$ & $22.89\%$ & $17.50\%$ & $30$\\
     UMAP {\tiny(Train+Test)} & $0.8328$ & $0.8324$ & $0.8308$ & $0.8340$ & $0.8330$ & $0.8315$ & $23.90\%$ & $17.85\%$ & $76$\\
     P-UMAP-MSE    & $0.8345$ & $0.8338$ & $0.8315$ & $0.8325$ & $0.8323$ & $0.8305$ & $23.30\%$ & $17.90\%$ & $63$ \\
     P-UMAP-CEMSE & $0.8397$ & $0.8410$ & $0.8411$ & $0.8379$ & $0.8393$ & $0.8398$ & $23.15\%$ & $17.00\%$ & $70$ \\
     P-UMAP-CE     & $\mathbf{0.8430}$ & $\mathbf{0.8430}$ & $\mathbf{0.8429}$ & $\mathbf{0.8400}$ & $\mathbf{0.8403}$ & $\mathbf{0.8409}$ & $\mathbf{22.30}\%$ & $\mathbf{16.15}\%$  & $78$ \\
    \bottomrule
  \end{tabular}
  }
\end{table*}

\begin{figure*} [t]
    {\centering
    \includegraphics[width=0.3\textwidth]{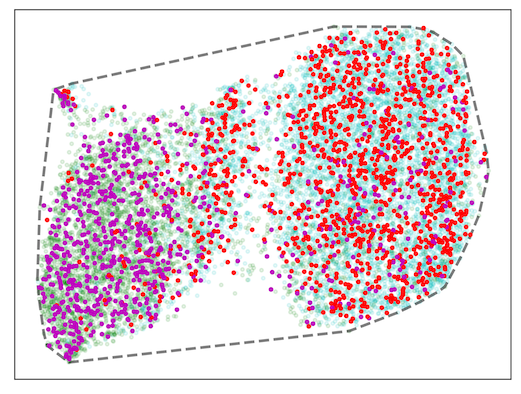} 
    \includegraphics[width=0.3\textwidth]{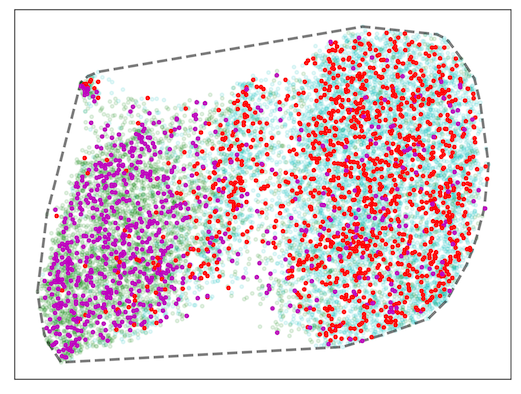} 
    \includegraphics[width=0.3\textwidth]{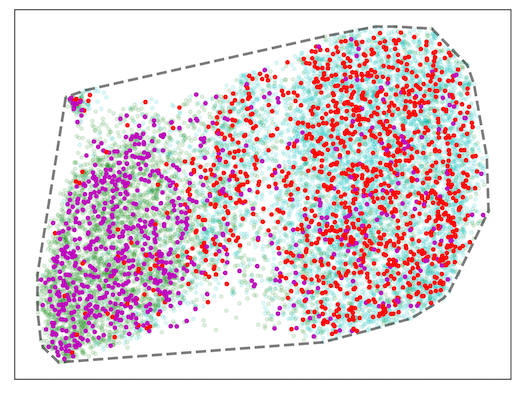} \\
    (a) \hspace{0.28\textwidth} (b) \hspace{0.28\textwidth} (c) \\}
    {\centering
    \includegraphics[width=0.3\textwidth]{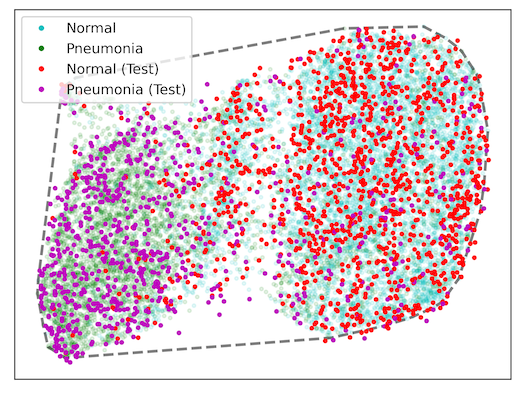} 
    \includegraphics[width=0.3\textwidth]{	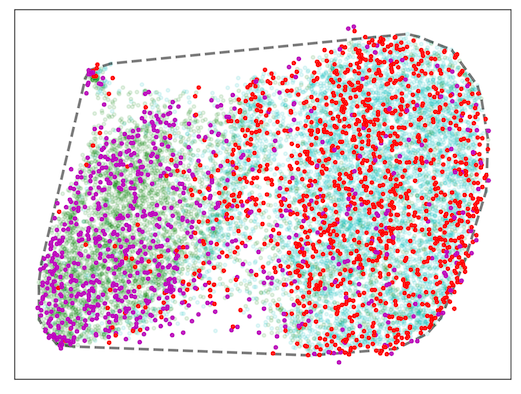} 
    \includegraphics[width=0.3\textwidth]{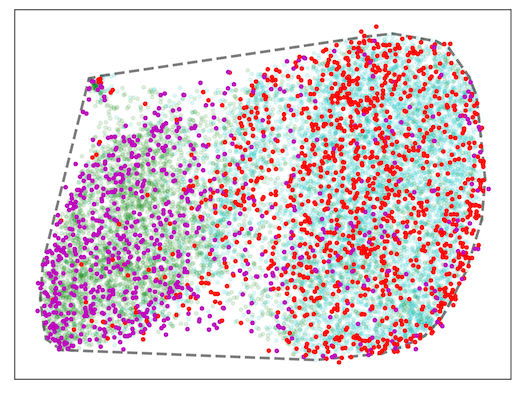} \\
    (d) \hspace{0.28\textwidth} (e) \hspace{0.28\textwidth} (f) \\}
    
    \caption{x-ray features embedded using UMAP. Transparent colors in the background represent the training points, and the solid dots represent the test points. A gray dashed line shows the convex hull of the embedding. Top row: training and test data embedded together using (a) $k=15$ ($\zeta=82$), (b) $k=30$ ($\zeta=76$), and (c) $k=50$ ($\zeta=62$). Bottom row: embedded training followed by out-of-sample test points using (d) $k=15$ ($\zeta=202$), (e) $k=30$ ($\zeta=245$) and (f) $k=50$ ($\zeta=243$) following rule-based algorithm in Section~\ref{sec:embed_out_of_sample}. (a-c) establishes a baseline of embeddings for increasing $k$ (as training and test points are embedded together), where the test embeddings are mostly inside the convex hull. (d-f) show that when test points are embedded afterward, they are placed around the periphery of the mapping and crossing the convex hull due to the `repulsion effect'; as $k$ increases, the repulsion effect increases, causing the number of test points in the periphery to increase.}
    \label{fig:UMAP_problem}
\end{figure*}

Next, we visually analyze some clusters (Fig.~\ref{fig:mnist_repulsion_effect}).
Figure~\ref{fig:mnist_repulsion_effect}(c) is a UMAP for the nearest neighbor parameter $k=15$, whereas the rest are for $k=30$. 
A grey dashed line shows the convex hull of the cluster, and the black arrows point to regions of interest. 
Note that the clusters obtained from the UMAP algorithm are not convex. However, in this example, the convex hull effectively demonstrates the accumulation at the periphery due to the repulsion effect.
For reference, when the UMAP embedding utilizes both the training and test data (Fig.~\ref{fig:mnist_repulsion_effect}(a)), the test points (black dots) are embedded throughout the cluster of training points (red dots), with $\zeta=49$. 
When the test points act as new data after the initial training, many test points accumulate at the boundary of the cluster (Figs.~\ref{fig:mnist_repulsion_effect}(b,c) for $k=30$ ($\zeta=104$) and $15$ ($105$), respectively, showing that the repulsion effect does not necessarily depend on the value of $k$. 
The accumulation is similar to baseline UMAP~{\tiny{(Train+test)}} ($\zeta=49$) for P-UMAP-CEMSE (Fig.~\ref{fig:mnist_repulsion_effect}(e), $\zeta=46$) and P-UMAP-CE (Fig.~\ref{fig:mnist_repulsion_effect}(f), $\zeta=51$).

\begin{figure*} [t]
    {\centering
    \includegraphics[width=0.3\textwidth]{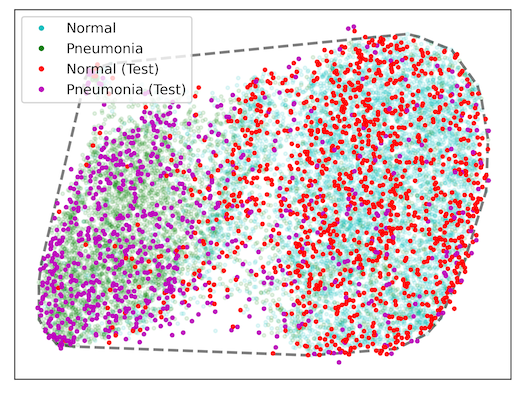} 
    \includegraphics[width=0.3\textwidth]{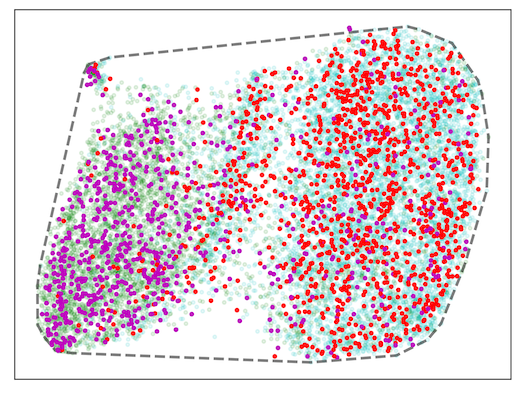} 
    \includegraphics[width=0.3\textwidth]{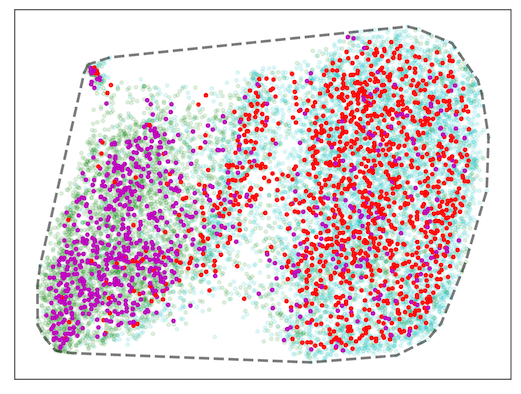} \\
    (a) \hspace{0.28\textwidth} (b) \hspace{0.28\textwidth} (c) \\}
    {\centering
    \includegraphics[width=0.3\textwidth]{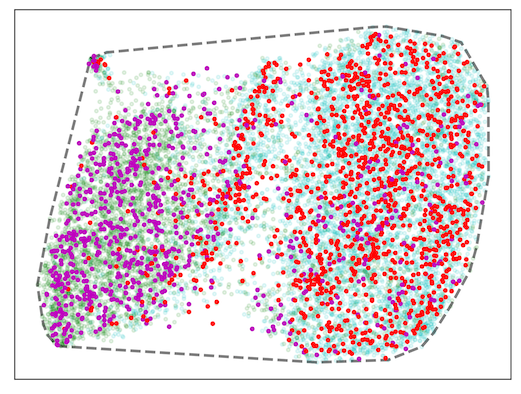} 
    \includegraphics[width=0.3\textwidth]{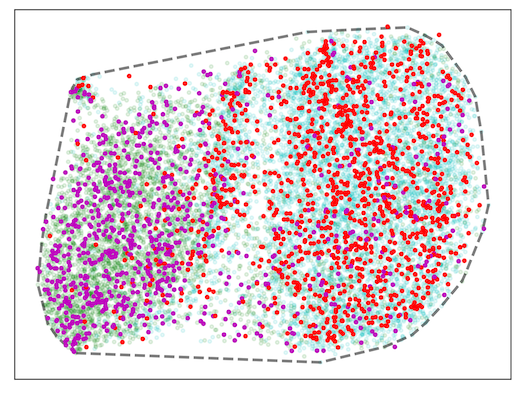} 
    \includegraphics[width=0.3\textwidth]{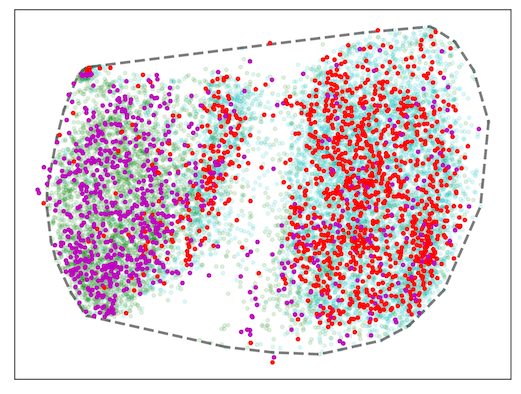} \\
    (d) \hspace{0.28\textwidth} (e) \hspace{0.28\textwidth} (f) \\}
    
    \caption{Embeddings for x-ray data for different algorithms for $k=30$.  Transparent dots in the background represent the training points, and the solid dots represent the test points. The gray dotted line shows the convex hull of the training points. Top row: UMAP for test points with $n_s$ set to (a) $5$ ($\zeta=245$), (b) $3$ ($\zeta=79$), and (c) $1$ ($\zeta=30$). Bottom row: parameterized embeddings using (d) P-UMAP-MSE ($\zeta=63$), (e) P-UMAP-CEMSE ($\zeta=70$), and (f) P-UMAP-CE ($\zeta=78$). The accumulation of the points at the periphery is absent for the modified methods as seen in (b-f) and quantified in Table~\ref{tab:pneumonia_results}.}
    \label{fig:covid_19_umaps}
\end{figure*}

Finally, we look at the time required to embed each out-of-sample test point. 
We computed the running time using $100$ test samples on a CPU. 
(AMD Ryzen Threadripper 2950X 16-Core Processor with 128 GB of RAM).
For both sequential and parallel embeddings (Table~\ref{tab:run_time_mnist}), the neural network is faster by several orders of magnitude than the conventional, rule-based out-of-sample embedding in UMAP.

\subsubsection{Chest x-ray Images}

From the RSNA pneumonia detection dataset~\cite{rsnadataset}, we collected a total of 13,389 chest x-rays:  
the training set consists of 6775 normal x-rays and 4,614 pneumonia x-rays, while  
the test set consists of 1,191 normal x-rays and 809 pneumonia x-rays. 
Because x-ray images have large variations, e.g., x-ray power and orientation, and the subject's condition during image capture, this dataset does not have a well-defined manifold like the images in MNIST. Thus, we used a pre-processing step to obtain one by feature extraction. 
More specifically, we used a DenseNet-121~\cite{huang2017densely} (a deep neural network trained on Imagenet~\cite{russakovsky2015imagenet}) to obtain a set of $1024$ characteristic features. 
The resulting UMAP embedding formed weakly separable clusters of normal vs. pneumonia patients (Fig.~\ref{fig:UMAP_problem}). A gray dotted line shows the convex hull of the training points.

This dataset shows a more obvious instance of the repulsion effect than MNIST. 
The baseline embeddings are established by embedding both the training and test data together (Fig.~\ref{fig:UMAP_problem}(a), (b), and (c) for $k=30$, $50$, and $70$, respectively).
Here, the accumulation decreases as $k$ increases. 
When the out-of-sample embedding is considered, the accumulation is more than three times higher for these cases than the corresponding baseline (Fig.~\ref{fig:UMAP_problem} (d), (e), and (f) for $k=15$, $k=30$, and $k=50$, respectively). Additionally, the accumulation increases as $k$ increases (from 15 to 30).

\begin{table*}[t]
 \caption{Results on the clinical dataset with `shortness of breath' as the chief complaint.} \label{tab:yale_data_table}
 \vspace{0.05in}
  \centering
    \resizebox{\textwidth}{!}{
  \begin{tabular}{l|c|c|c|c|c|c|c|c}
    \toprule
            & \multicolumn{3}{|c|}{Trustworthiness (Train)} & \multicolumn{3}{|c|}{Trustworthiness (Train+Test)} & \multicolumn{2}{|c}{k-NN classifier} \\
             \cmidrule(r){2-9}
             & $T_{5}$ & $T_{30}$ & $T_{100}$ & $T_{5}$ & $T_{30}$ & $T_{100}$ & $1$-NN Error & $5$-NN Error  \\
    \midrule
     UMAP \tiny{($n_s=5$)} & $0.8229$ & $0.7948$ & $0.7819$ & $0.8159$ & $0.7942$ & $0.7817$ & $35.40\%$ & $31.45\%$ \\
     UMAP \tiny{($n_s=3$)} & -- & -- & -- & $0.8179$ & $0.7958$ & $0.7828$ & $35.36\%$ & $30.60\%$\\
     UMAP \tiny{($n_s=1$)} & -- & -- & -- & $0.8166$ & $0.7944$ & $0.7823$ & $35.96\%$ & $30.79\%$\\
     UMAP \tiny{(Train+Test)} & $0.8283$ & $0.7975$ & $0.7829$ & $0.8303$ & $0.8004$ & $0.7853$ & $35.21\%$ & $30.30\%$\\
     P-UMAP-MSE    & $0.8119$ & $0.7858$ & $0.7741$ & $0.7997$ & $0.7803$ & $0.7693$ & $35.75\%$ & $30.95\%$ \\
     P-UMAP-CEMSE  & $0.8121$ & $0.7998$ & $0.7887$ & $0.8070$ & $0.7982$ & $0.7870$ & $34.90\%$ & $30.16\%$ \\
     P-UMAP-CE     & $\mathbf{0.8431}$ & $\mathbf{0.8119}$ & $\mathbf{0.7994}$ & $\mathbf{0.8342}$ & $\mathbf{0.8098}$ & $\mathbf{0.7982}$ & $\mathbf{34.70}\%$ & $\mathbf{29.75}\%$ \\

    \bottomrule
  \end{tabular}
  }
\end{table*}

\begin{figure*}[t]
    \centering
    \includegraphics[width=0.24\textwidth]{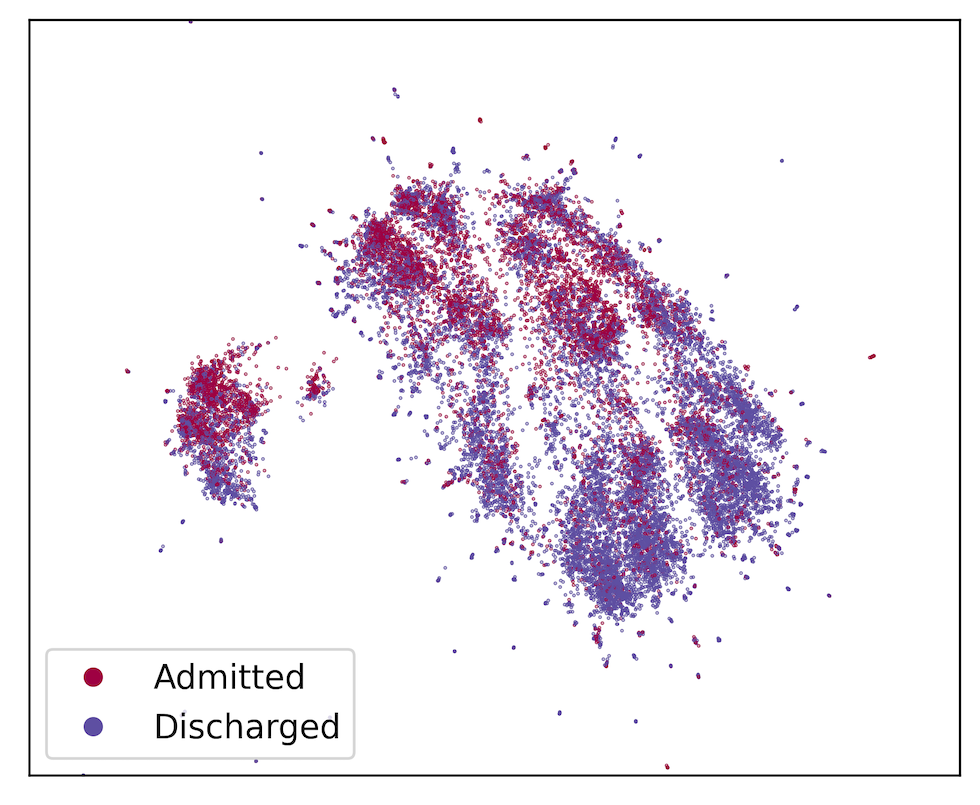} 
    \includegraphics[width=0.24\textwidth]{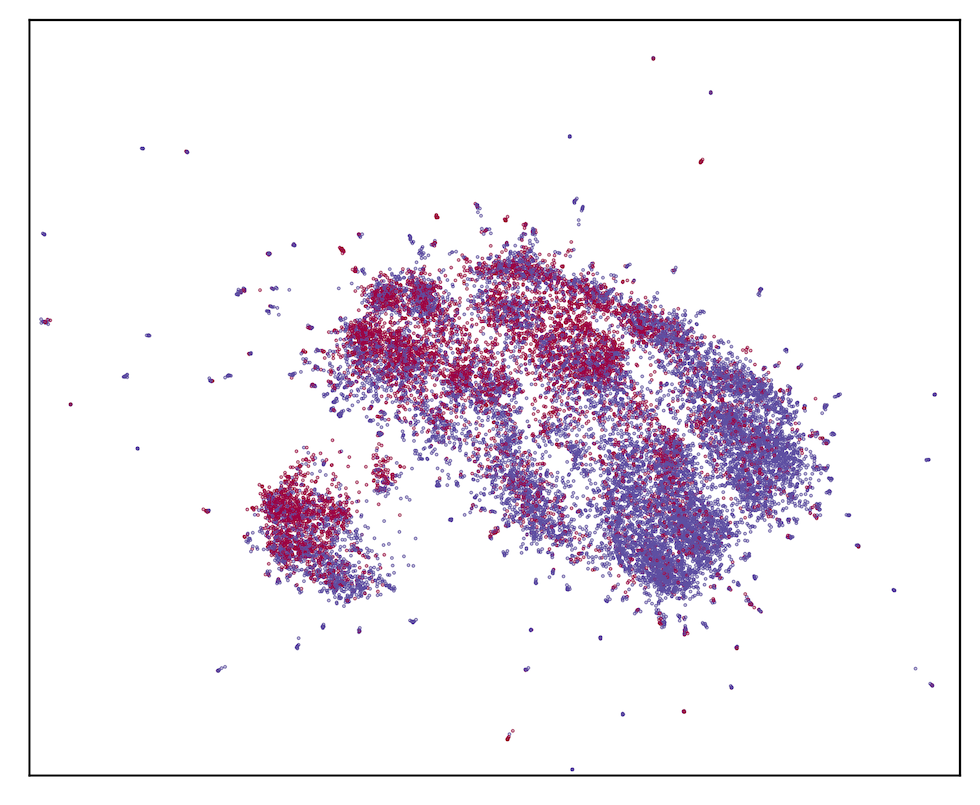}  
    \includegraphics[width=0.24\textwidth]{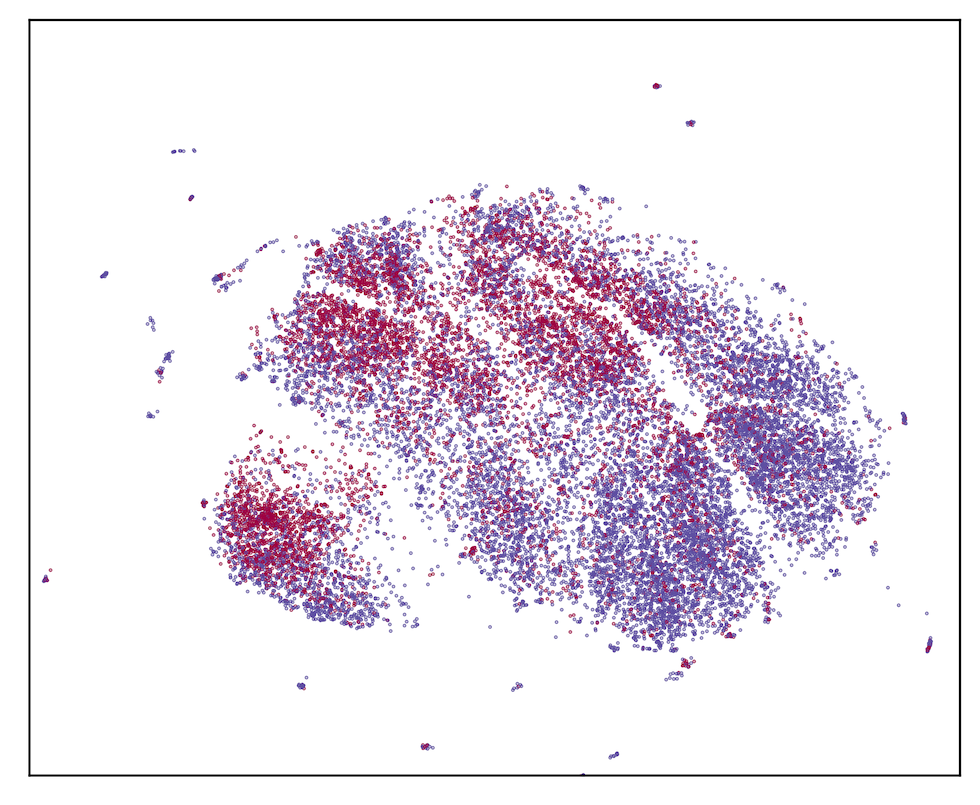} 
    \includegraphics[width=0.24\textwidth]{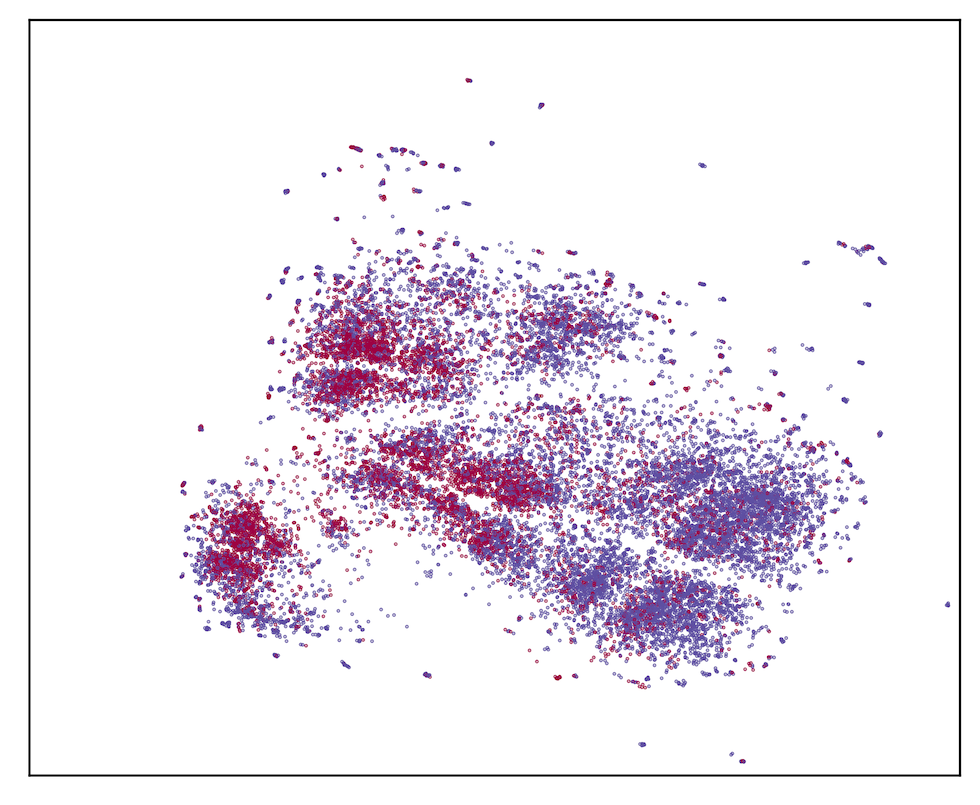} \\
     (a) \hspace{0.2\textwidth} (b) \hspace{0.2\textwidth} (c) \hspace{0.2\textwidth} (d) \\
    
    \caption{Embedding of patient features with `shortness of breath' as the chief complaint. (a) UMAP, (b) P-UMAP-MSE, (c) P-UMAP-CEMSE, (d) P-UMAP-CE.}
    \label{fig:shortness_of_breath}
\end{figure*}

\begin{figure*}[t]
    \centering
    \includegraphics[width=0.24\textwidth]{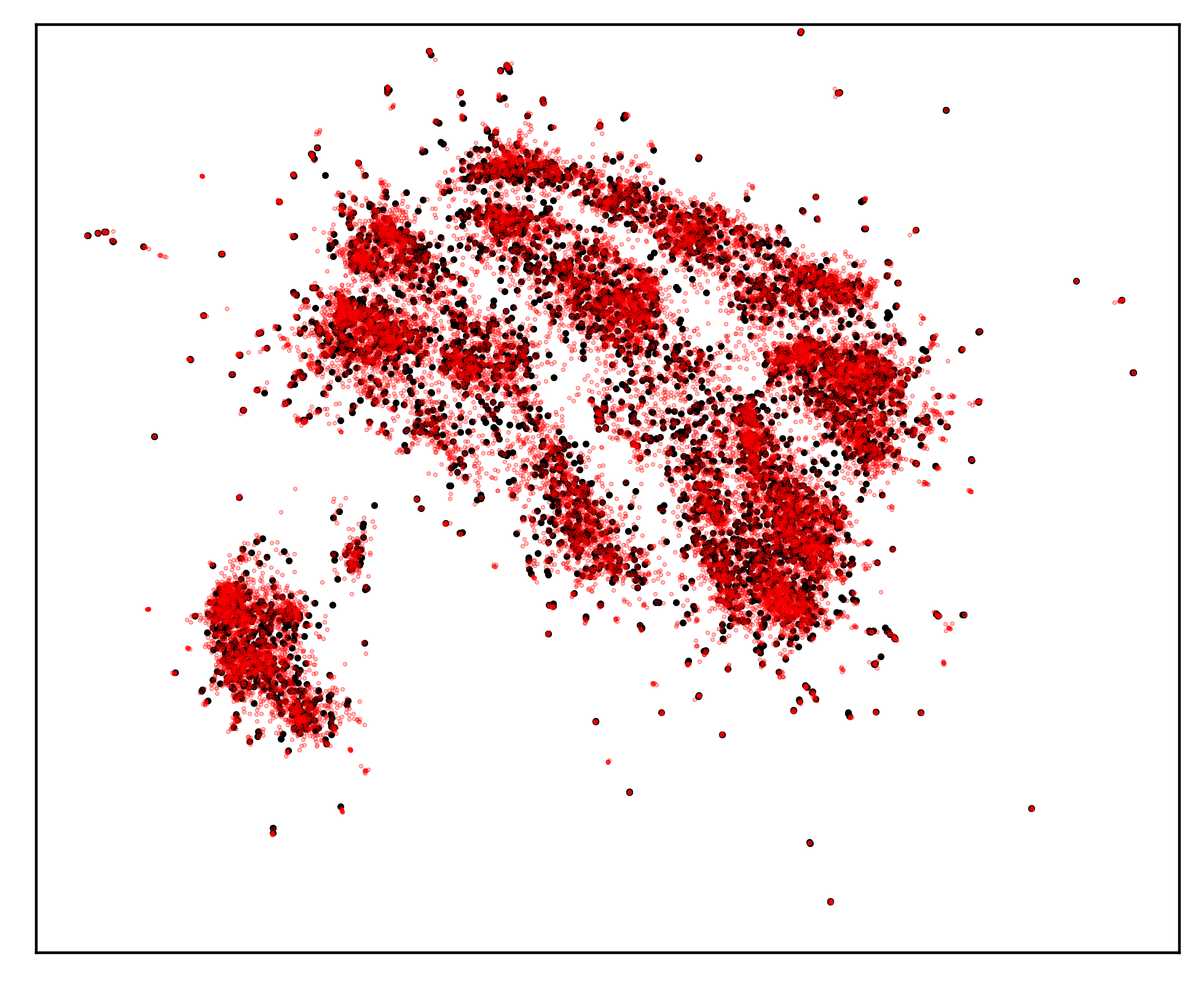}
    \includegraphics[width=0.24\textwidth]{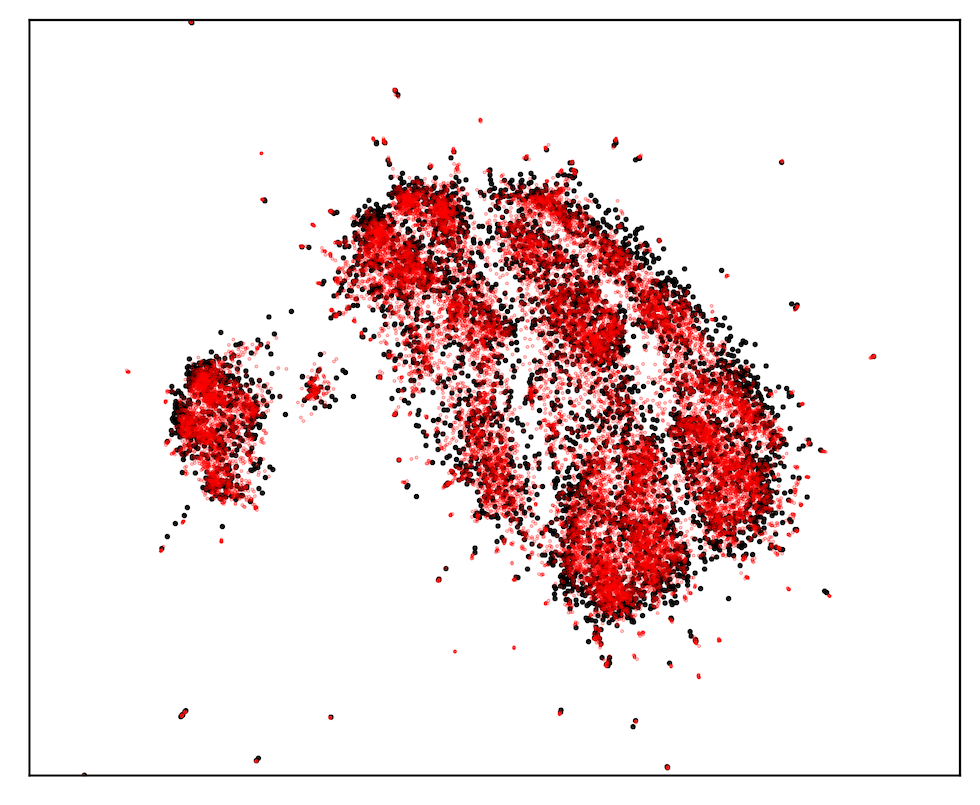}
    \includegraphics[width=0.24\textwidth]{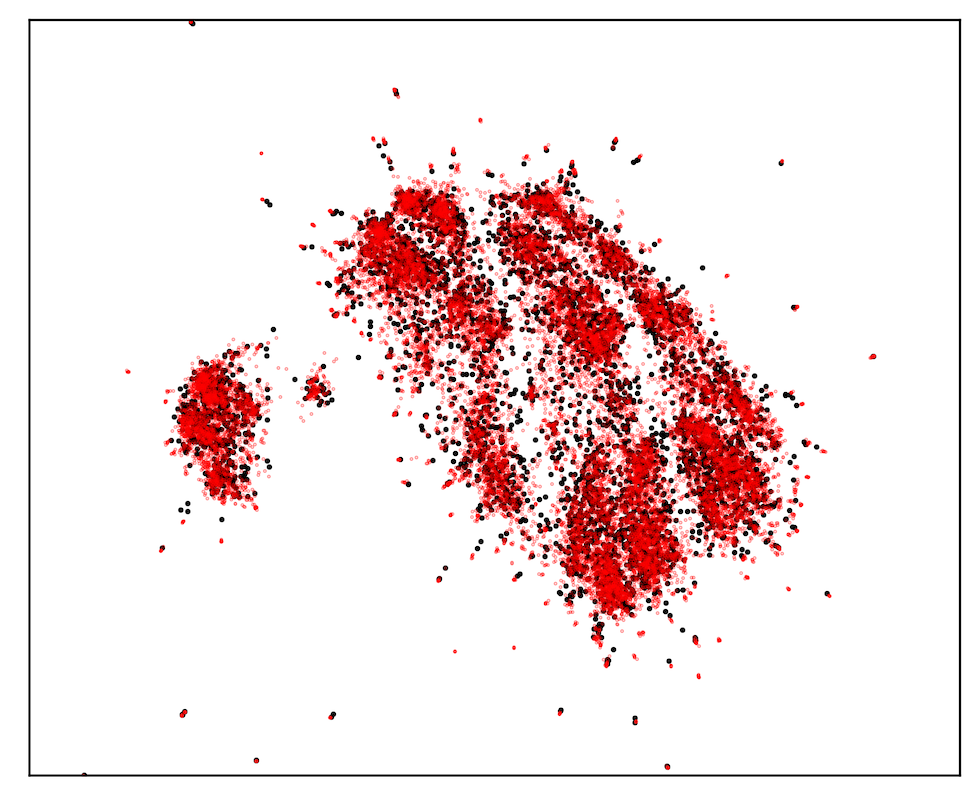}
    \includegraphics[width=0.24\textwidth]{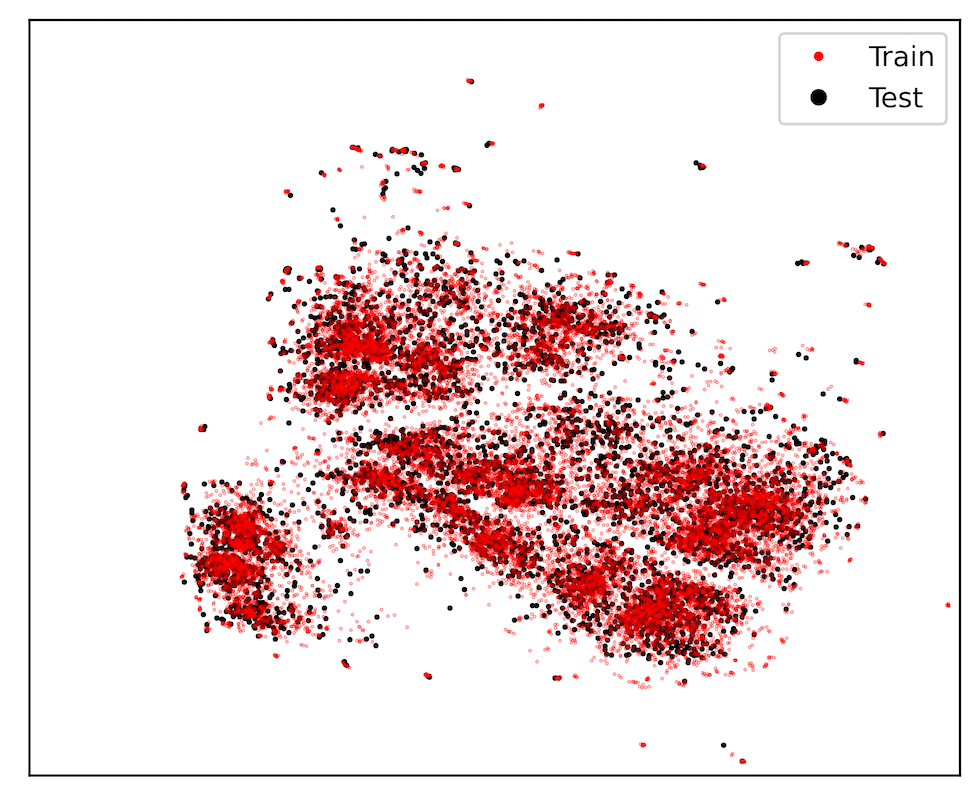} \\
     (a) \hspace{0.2\textwidth} (b) \hspace{0.2\textwidth} (c) \hspace{0.2\textwidth} (d) \\
        
    \caption{Repulsion effect in the clinical dataset for patients with shortness of breath. (a) UMAP when training and test data are embedded together. (b-d) Embedded using training data followed by test data as out-of-sample points using (b) original UMAP embedding with $n_s=5$, (c) UMAP embedding with $n_s=3$,  and (d) P-UMAP-CE. The repulsion effect leads to the accumulation of points in (b) through the presence of a noticeable dark shadow at the peripheries, which is absent in both (c) and (d).}
    \label{fig:repulsion_effect_yale}
\end{figure*}

Figure~\ref{fig:covid_19_umaps} shows different embeddings for $k=30$. 
We used a negative sampling rate of $n_s=5$ for all the training cases, but for testing purposes, we also considered $n_s=3$ and $n_s=1$  (Figs.~\ref{fig:covid_19_umaps}(a-c)). 
The accumulation of points at the border remains the same for $n_s=3$ ($\zeta=79$; baseline: 78 for UMAP~{\tiny{(Train+Test)}}), but a distinct avoidance of the border is apparent when $n_s=1$ ($\zeta=30$). This result again supports the hypothesis that the repulsion effect is due to the collective effect of multiple negative edges. 
Similarly, the parameterized versions of UMAP do not require negative sampling after training and thus do not show signs of the repulsion effect either (Figs.~\ref{fig:covid_19_umaps}(d-f)). 

This overall trend is also reflected in the trustworthiness and $k$-NN classifier error (Table~\ref{tab:pneumonia_results}).
P-UMAP-CE gives the highest scores, with the performance of P-UMAP-CEMSE a close second. 
For the non-parametric UMAP algorithms, the trustworthiness on the complete dataset (Train+Test) is the highest when $n_s = 3$, a lower value than that used in the original construction. 
However, this value cannot be too low, as with $n_s=1$, the trustworthiness decreases.

\subsubsection{Clinical Data}

Here, we examine the dataset compiled by Hong et al.~\cite{hong2018predicting} aimed to automate hospital admissions in emergency departments.  
The data originates from three departments from March 2013 to July 2017 and consists of patient history, demographics, chief complaints, and vital signs. 
It includes patient triage information of 190,000 patients and 560,486 hospital visits. 
The data is primarily labeled by the feature `disposition', indicating whether a patient was admitted into the hospital or discharged. 
In addition, 971 triage and demographic features include categorical and numerical data. 
For analysis, we employ the method of patient phenotyping used by Hurley et al.~\cite{hurley2019visualization} and consider 'Shortness of Breath' as the chief complaint (analysis of 'Abdominal Pain' is provided in Appendix~\ref{suppsec:more_clinical_data}).
We discarded the `disposition' and `emergency severity index' features, as they are decision variables, while one-hot encoded the categorical data and mean-imputed the missing values. 
After that, the data was divided into train and test sets (with an 80\% and 20\% split, respectively). 
We obtained the embeddings by setting the k-NN parameter $k$ to $30$. 
For the UMAP embedding, we set the minimum distance parameters $m_d$ to $0$, $0.0001$, $0.001$, $0.01$, $0.1$, $0.5$, $0.8$, $1.0$, and reported the one that produced the highest trustworthiness (for additional details, see Fig.~\ref{suppfig:trustworthienss_yale_shortness} in Appendix~\ref{suppsec:more_clinical_data}). 
For the parameterized embeddings, we set $m_d$ to $0.1$.

There were $24,652$ hospital visits pertaining to shortness of breath, with $15,791$ of those visits resulting in admissions. 
After converting labels to one-hot encoded data and dropping features with missing entries, we obtained $1003$-dimensional vectors. 
The number of visits in the training (test) split is $19,721$ ($4,931$). 
Figure~\ref{fig:shortness_of_breath} shows the two-dimensional embeddings obtained from the bare and parameterized UMAPs. 
All the embeddings have the same general structure, showing the separation of admitted patients from the discharged. 
The trustworthiness values (Table~\ref{tab:yale_data_table}) show that P-UMAP-CE provides the highest trustworthiness values for all nearest neighbors considered, with P-UMAP-CEMSE the second highest. 
Interestingly, for training embedding at a very local level (5 nearest neighbors), the original UMAP has a higher trustworthiness score than the versions containing MSE (agreeing with the other datasets).
The second-highest trustworthiness is for $n_s=3$ (in a non-parametric setting). 
The 2-NN and 5-NN classifiers are comparable for all the methods, while P-UMAP-CE and P-UMAP-CEMSE are marginally better.

Figure~\ref{fig:repulsion_effect_yale} shows the training and test points in the embedding. As before, the test points distribute uniformly within the clusters for UMAP~{\tiny{(Train+Test)}} (Fig.~\ref{fig:repulsion_effect_yale}(a)) but accumulate at the periphery for UMAP~{\tiny{($n_s=5$)}} (Fig.~\ref{fig:repulsion_effect_yale}(b)). This accumulation at the boundary disappears when UMAP employs a lower value of $n_s=3$ (Fig.~\ref{fig:repulsion_effect_yale}(c)) and in the parametric implementation P-UMAP-CE (Fig.~\ref{fig:repulsion_effect_yale}(d)).

\subsection{Analyzing Attractive and Repulsive Forces}\label{sec:afr_rfr_results}

In this section, we analyze the attractive and repulsive forces of the test points around the periphery of clusters and assess the quality based on the force ratios (AFR and RFR) described in Section~\ref{sec:force_ratios}. 
Both AFR and RFR require two sets of points from the dataset. 
To this end, from each dataset, we defined two regions near each other in the UMAP~{\tiny{($n_s=5$)}} embeddings: one near/outside the periphery of the clusters and the other inside. 
The set $S_1$ originates from the test points of the former region, while $S_2$ originates from the latter. 
We ensured that $|S_1|=|S_2|$  (for explicit locations of $S_1$ and $S_2$ for each dataset, see Fig.~\ref{suppfig:choiceofpoints} in the Appendix.)
Note that the absolute value of the forces is less important, as by scaling the embedding, one can change them without changing the embedding's characteristics in terms of trustworthiness or $k$-NN classifier. 

\clearpage
\begin{figure*}[t]
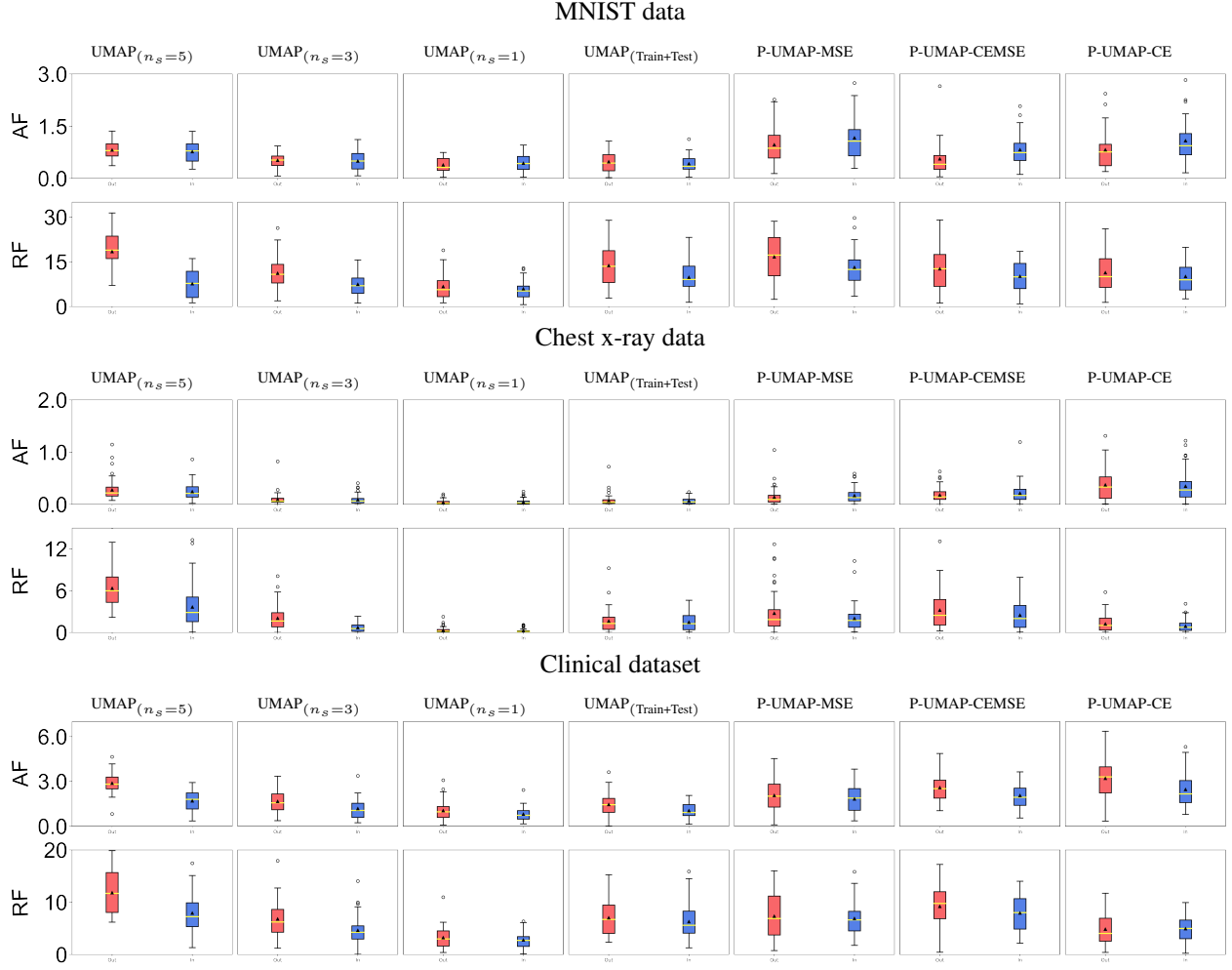

    {\centering
    MNIST data \\}
    \vspace{0.1in}
    {\raggedright
    \hspace{0.06\linewidth} \tiny{UMAP$_{(n_s=5)}$} \hspace{0.045\linewidth} \tiny{UMAP$_{(n_s=3)}$} \hspace{0.045\linewidth} \tiny{UMAP$_{(n_s=1)}$} \hspace{0.04\linewidth} \tiny{UMAP$_{(\text{Train+Test})}$} \hspace{0.04\linewidth} \tiny{P-UMAP-MSE} \hspace{0.04\linewidth} \tiny{P-UMAP-CEMSE} \hspace{0.045\linewidth} \tiny{P-UMAP-CE}  \\}
    {\centering
    \includegraphics[width=\linewidth]{mnist_boxplot.png} \\
    }
    
    {\centering
    Chest x-ray data \\}
    \vspace{0.1in}
    {\raggedright
    \hspace{0.06\linewidth} \tiny{UMAP$_{(n_s=5)}$} \hspace{0.045\linewidth} \tiny{UMAP$_{(n_s=3)}$} \hspace{0.045\linewidth} \tiny{UMAP$_{(n_s=1)}$} \hspace{0.04\linewidth} \tiny{UMAP$_{(\text{Train+Test})}$} \hspace{0.04\linewidth} \tiny{P-UMAP-MSE} \hspace{0.04\linewidth} \tiny{P-UMAP-CEMSE} \hspace{0.045\linewidth} \tiny{P-UMAP-CE}  \\}
    {\centering
    \includegraphics[width=\linewidth]{pneu_boxplot.png} \\
    }

    {\centering
    Clinical dataset \\}
    \vspace{0.1in}
    {\raggedright
    \hspace{0.06\linewidth} \tiny{UMAP$_{(n_s=5)}$} \hspace{0.045\linewidth} \tiny{UMAP$_{(n_s=3)}$} \hspace{0.045\linewidth} \tiny{UMAP$_{(n_s=1)}$} \hspace{0.04\linewidth} \tiny{UMAP$_{(\text{Train+Test})}$} \hspace{0.04\linewidth} \tiny{P-UMAP-MSE} \hspace{0.04\linewidth} \tiny{P-UMAP-CEMSE} \hspace{0.045\linewidth} \tiny{P-UMAP-CE}  \\}
    {\centering
    \includegraphics[width=\linewidth]{short_boxplot.png} \\
    }
    
    \caption{Distribution of attractive forces (AF) and repulsive forces (RF) for MNIST, chest x-ray, and clinical datasets. For each dataset, the top row shows the distribution of attractive forces, and the bottom row shows the distribution of repulsive forces for sets $S_1$ (outside the periphery, red) and $S_2$ (inside the periphery, blue), respectively. The boxplots (whiskers) are within the (1.5) interquartile range. The yellow bar and the black triangle indicate the median and mean, respectively.}
    \label{fig:afr_rfr_figures}
\end{figure*}

\begin{table*}[t]
    \centering
    \caption{Force ratios, $\mathcal{F}_a=|1-\text{AFR}|$ and $\mathcal{F}_r=|1-\text{RFR}|$, for different datasets.}\label{tab:afr_rfr}
    \vspace{0.05in}
    \resizebox{\textwidth}{!}{
    \begin{tabular}{l|lr|lr||lr|lr||lr|lr}
        \toprule
     & \multicolumn{4}{|c||}{MNIST dataset} & \multicolumn{4}{|c||}{Chest x-ray dataset} & \multicolumn{4}{|c}{Clinical dataset} \\
        \midrule
         & $\mathcal{F}_a$ & (\%I) & $\mathcal{F}_r$ & (\%I) & $\mathcal{F}_a$ & (\%I) & $\mathcal{F}_r$ & (\%I) & $\mathcal{F}_a$ & (\%I) & $\mathcal{F}_r$ & (\%I) \\
         \midrule
        UMAP \tiny{($n_s=5$)}     & 0.06 & (0.0\%)  & 1.40 & (0.0\%)    & 0.07 & (0.0\%)    &  0.39 & (0.0\%)     & 0.69 & (0.0\%)    & 0.49 & (0.0\%) \\
        UMAP \tiny{($n_s=3$)}     & 0.04 & (33.33\%)  & 0.50 & (64.29\%)  & 0.08 & (-14.29\%)  &  0.78 & (-100\%)  & 0.40 & (42.03\%)  & 0.47 & (4.08\%) \\
        UMAP \tiny{($n_s=1$)}     & 0.11 & (-83.33\%)  & 0.13 & (90.71\%)   & 0.08 & (-14.29\%) &  0.15 & (61.53\%)    & 0.29 & (57.97\%)  & 0.16 & (67.35\%) \\
        UMAP \tiny{(Train+Test)}  & 0.13 & (-116.67\%)  & 0.41 & (70.71\%) & 0.01 & (85.71\%)  &   0.08 & (79.49\%)    & 0.39 & (43.47\%)  & 0.12 & (75.51\%) \\
        P-UMAP-MSE               & 0.17 & (-183.33\%)  & 0.27 & (80.71\%) & 0.11 & (-57.14\%)  &  0.13 & (66.67\%)    & 0.13 & (81.16\%)  & 0.06 & (87.76\%) \\
        P-UMAP-CEMSE              & 0.32 & (-433.33\%)  & 0.24 & (82.86\%) & 0.07 & (0.0\%)  &    0.15 & (61.53\%)    & 0.26 & (62.31\%)  & 0.16 & (67.35\%) \\
        P-UMAP-CE                 & 0.24  &(-300\%)  & 0.13 & (90.71\%) & 0.05 & (28.57\%)      & 0.19 & (51.28\%)  & 0.24 & (66.22\%)  & 0.02 & (95.92\%) \\
        \bottomrule
    \end{tabular}
    }
\end{table*}
\clearpage

Figure~\ref{fig:afr_rfr_figures} shows the distribution of attractive (top row) and repulsive forces (bottom row) for each of the datasets. 
Overall, we can observe that the attractive forces are dispersed equally in both sets for all the methods. The distributions of repulsive forces have a significant difference between $S_1$ and $S_2$ for UMAP~{\tiny{($n_s=5$)}}, particularly for the MNIST and Clinical datasets. 
The distribution of repulsive forces is much broader (the mean is higher, and the medians are further from the mean) for the points outside the periphery ($S_1$ in red) than for the points inside the periphery ($S_2$ in blue). 

Table~\ref{tab:afr_rfr} shows the AFR and RFR of these regions of each dataset.
For ease of understanding, we tabulated the values in terms of $\mathcal{F}_a=|1-\text{AFR}|$ and $\mathcal{F}_r=|1-\text{RFR}|$, so that the minimum achievable value of 0 indicates a `lower is better' metric. 
We also tabulate the percent change (\%I) of these values from UMAP~{\tiny{($n_s=5$)}}.
A positive \%I indicates improvement of force balance from that of UMAP~{\tiny{($n_s=5$)}}. 
We considered 30 nearest neighbors to compute $\mathcal{F}_a$ and $\mathcal{F}_r$.

Overall, $\mathcal{F}_r$ values have improved over UMAP~{\tiny{($n_s=5$)}} for all the embeddings (except x-ray data for UMAP~{\tiny{($n_s=3$)}}) indicating a balance of repulsive forces.
For chest x-ray data, P-UMAP-CE improves $\mathcal{F}_a$ by 28.57\% and $\mathcal{F}_r$ by 51.28\%. 
For clinical data the improvement is 66.22\% in $\mathcal{F}_a$ and 95.92\% in $\mathcal{F}_r$. 
For MNIST data, $\mathcal{F}_r$ improves by 90.71\%.
This analysis shows numerically that the repulsion effect is reduced in parameterized UMAPs, supporting the results in Section~\ref{sec:comparing_visualization}.

\section{Conclusion}\label{sec:conclusions}
The original UMAP algorithm is not an online algorithm, in the sense that it cannot accommodate new data points without re-running the embedding from scratch. When it tries to embed out-of-sample test points, they are placed on the periphery of existing clusters. 
We have demonstrated that this behavior arises from dominant repulsive forces during optimization, regardless of the number of nearest neighbors considered. 
Reducing this “repulsion effect” (e.g., by making the negative sampling parameter $n_s$ lower than that of training embedding) leads to better embeddings, but the best performance is obtained by parameterizing the mapping. 
We parameterized UMAP using cross-entropy (CE) loss and mean-squared error (MSE), individually and in combination, and demonstrated its behavior on MNIST digits, chest x-ray images, and clinical data. 
P-UMAP-CE consistently outperformed the other algorithms, in both visualization and trustworthiness, with the benefits of parameterization increasing as the data became more complex. 
Analyzing repulsive forces, we further showed that the points that are pushed outside the periphery have disproportionately higher repulsive forces than the points that are just inside. 
We characterized this phenomenon using the repulsive force ratio (RFR). 
Similarly, we also characterized attractive forces of such points using the attractive force ratio (AFR). 
We showed that parameterized UMAPs provide better ARF and RFR values for out-of-sample points compared to those of the original UMAP. 
The force decomposition used here is a general analysis tool for any dimensionality reduction algorithm, while the parameterized algorithms that result from it should find application in any field that must be updated continuously, e.g., those involving clinical and biomedical data.

\section*{Acknowledgment}

This work was supported by AFOSR grant FA9550-21-1-0317 and the Schmidt DataX Fund at Princeton University, made possible through a major gift from the Schmidt Futures Foundation. Mohammad Tariqul Islam is supported by MIT-Novo Nordisk Artificial Intelligence Fellowship.

We would like to thank Hong et el.~\cite{hong2018predicting} for making the clinical data and corresponding documentation for parsing the dataset publicly available, which accelerated setting up the experiments with clinical dataset.


\section*{Data Availability}
All the data used in the experiments is publicly available. MNIST is publicly available at \url{http://yann.lecun.com/exdb/mnist/}. X-Ray data can be publicly obtained from  \url{https://www.kaggle.com/c/rsna-pneumonia-detection-challenge}. 
Clinical data is publicly available at \url{https://github.com/yaleemmlc/admissionprediction}.


%
%

\bibliographystyle{unsrt}  
\bibliography{references} 

\begin{thebibliography}{10}

\bibitem{maaten2008visualizing}
Laurens van~der Maaten and Geoffrey Hinton.
\newblock Visualizing data using t-sne.
\newblock {\em Journal of machine learning research}, 9(Nov):2579--2605, 2008.

\bibitem{mcinnes2018umap}
Leland McInnes, John Healy, and James Melville.
\newblock {UMAP}: Uniform manifold approximation and projection for dimension
  reduction.
\newblock {\em arXiv preprint arXiv:1802.03426}, 2018.

\bibitem{hinton2002stochastic}
Geoffrey Hinton and Sam~T Roweis.
\newblock Stochastic neighbor embedding.
\newblock In {\em Advances in Neural Information Processing Systems},
  volume~15, pages 833--840, 2002.

\bibitem{macosko2015highly}
Evan~Z Macosko, Anindita Basu, Rahul Satija, James Nemesh, Karthik Shekhar,
  Melissa Goldman, Itay Tirosh, Allison~R Bialas, Nolan Kamitaki, Emily~M
  Martersteck, et~al.
\newblock Highly parallel genome-wide expression profiling of individual cells
  using nanoliter droplets.
\newblock {\em Cell}, 161(5):1202--1214, 2015.

\bibitem{kobak2019art}
Dmitry Kobak and Philipp Berens.
\newblock The art of using t-sne for single-cell transcriptomics.
\newblock {\em Nature communications}, 10(1):1--14, 2019.

\bibitem{cao2019single}
Junyue Cao, Malte Spielmann, Xiaojie Qiu, Xingfan Huang, Daniel~M Ibrahim,
  Andrew~J Hill, Fan Zhang, Stefan Mundlos, Lena Christiansen, Frank~J
  Steemers, et~al.
\newblock The single-cell transcriptional landscape of mammalian organogenesis.
\newblock {\em Nature}, 566(7745):496--502, 2019.

\bibitem{becht2019dimensionality}
Etienne Becht, Leland McInnes, John Healy, Charles-Antoine Dutertre,
  Immanuel~WH Kwok, Lai~Guan Ng, Florent Ginhoux, and Evan~W Newell.
\newblock Dimensionality reduction for visualizing single-cell data using umap.
\newblock {\em Nature biotechnology}, 37(1):38--44, 2019.

\bibitem{packer2019lineage}
Jonathan~S Packer, Qin Zhu, Chau Huynh, Priya Sivaramakrishnan, Elicia Preston,
  Hannah Dueck, Derek Stefanik, Kai Tan, Cole Trapnell, Junhyong Kim, et~al.
\newblock A lineage-resolved molecular atlas of c. elegans embryogenesis at
  single-cell resolution.
\newblock {\em Science}, 365(6459):eaax1971, 2019.

\bibitem{badamdorj2022contrastive}
Taivanbat Badamdorj, Mrigank Rochan, Yang Wang, and Li~Cheng.
\newblock Contrastive learning for unsupervised video highlight detection.
\newblock In {\em Proceedings of the IEEE/CVF Conference on Computer Vision and
  Pattern Recognition}, pages 14042--14052, 2022.

\bibitem{islam2024deciphering}
Md~Tauhidul Islam and Lei Xing.
\newblock Deciphering the feature representation of deep neural networks for
  high-performance ai.
\newblock {\em IEEE Transactions on Pattern Analysis and Machine Intelligence},
  2024.

\bibitem{hong2018predicting}
Woo~Suk Hong, Adrian~Daniel Haimovich, and R~Andrew Taylor.
\newblock Predicting hospital admission at emergency department triage using
  machine learning.
\newblock {\em PloS one}, 13(7):e0201016, 2018.

\bibitem{fleischer2020late}
Jason Fleischer and Mohammad~Tariqul Islam.
\newblock Late breaking abstract-identifying and phenotyping {COVID-19}
  patients using machine learning on chest x-rays.
\newblock {\em European Respiratory Journal}, 2020.

\bibitem{wang2022medclip}
Zifeng Wang, Zhenbang Wu, Dinesh Agarwal, and Jimeng Sun.
\newblock {MedCLIP}: Contrastive learning from unpaired medical images and
  text.
\newblock In {\em Proceedings of the 2022 Conference on Empirical Methods in
  Natural Language Processing}, pages 3876--3887, 2022.

\bibitem{islam2024outlier}
Mohammad~Tariqul Islam and Jason~W Fleischer.
\newblock Outlier detection in large radiological datasets using umap.
\newblock In {\em International Workshop on Topology-and Graph-Informed Imaging
  Informatics}, pages 111--121. Springer, 2024.

\bibitem{vazifeh2025manifold}
Amir~Reza Vazifeh and Jason~W Fleischer.
\newblock Manifold learning for personalized and label-free detection of
  cardiac arrhythmias.
\newblock {\em arXiv preprint arXiv:2506.16494}, 2025.

\bibitem{anjinappa2021coverage}
Chethan~K Anjinappa and Ismail G{\"u}ven{\c{c}}.
\newblock Coverage hole detection for mmwave networks: An unsupervised learning
  approach.
\newblock {\em IEEE Communications Letters}, 2021.

\bibitem{xu2022boundary}
Zhendong Xu, Baoqi Huang, Bing Jia, Wuyungerile Li, and Hui Lu.
\newblock A boundary aware wifi localization scheme based on umap and knn.
\newblock {\em IEEE Communications Letters}, 2022.

\bibitem{sammon1969nonlinear}
John~W Sammon.
\newblock A nonlinear mapping for data structure analysis.
\newblock {\em IEEE Transactions on computers}, 100(5):401--409, 1969.

\bibitem{tenenbaum2000global}
Joshua~B Tenenbaum, Vin De~Silva, and John~C Langford.
\newblock A global geometric framework for nonlinear dimensionality reduction.
\newblock {\em science}, 290(5500):2319--2323, 2000.

\bibitem{roweis2000nonlinear}
Sam~T Roweis and Lawrence~K Saul.
\newblock Nonlinear dimensionality reduction by locally linear embedding.
\newblock {\em science}, 290(5500):2323--2326, 2000.

\bibitem{belkin2002laplacian}
Mikhail Belkin and Partha Niyogi.
\newblock Laplacian eigenmaps and spectral techniques for embedding and
  clustering.
\newblock In {\em Advances in neural information processing systems}, pages
  585--591, 2002.

\bibitem{tang2016visualizing}
Jian Tang, Jingzhou Liu, Ming Zhang, and Qiaozhu Mei.
\newblock Visualizing large-scale and high-dimensional data.
\newblock In {\em Proceedings of the 25th international conference on world
  wide web}, pages 287--297, 2016.

\bibitem{van2014accelerating}
Laurens Van Der~Maaten.
\newblock Accelerating t-sne using tree-based algorithms.
\newblock {\em The Journal of Machine Learning Research}, 15(1):3221--3245,
  2014.

\bibitem{yang2013scalable}
Zhirong Yang, Jaakko Peltonen, and Samuel Kaski.
\newblock Scalable optimization of neighbor embedding for visualization.
\newblock In {\em International Conference on Machine Learning}, pages
  127--135, 2013.

\bibitem{barnes1986hierarchical}
Josh Barnes and Piet Hut.
\newblock A hierarchical o (n log n) force-calculation algorithm.
\newblock {\em nature}, 324(6096):446--449, 1986.

\bibitem{mikolov2013distributed}
Tomas Mikolov, Ilya Sutskever, Kai Chen, Greg~S Corrado, and Jeff Dean.
\newblock Distributed representations of words and phrases and their
  compositionality.
\newblock In {\em Advances in neural information processing systems}, pages
  3111--3119, 2013.

\bibitem{linderman2019fast}
George~C Linderman, Manas Rachh, Jeremy~G Hoskins, Stefan Steinerberger, and
  Yuval Kluger.
\newblock Fast interpolation-based t-sne for improved visualization of
  single-cell rna-seq data.
\newblock {\em Nature methods}, 16(3):243--245, 2019.

\bibitem{damrich2021umap}
Sebastian Damrich and Fred~A Hamprecht.
\newblock On {UMAP's} true loss function.
\newblock {\em Advances in Neural Information Processing Systems},
  34:5798--5809, 2021.

\bibitem{islam2025shape}
Mohammad~Tariqul Islam and Jason~W Fleischer.
\newblock The shape of attraction in {UMAP}: Exploring the embedding forces in
  dimensionality reduction.
\newblock {\em Transactions of Machine Learning Research}, 2025.

\bibitem{bohm2022attraction}
Jan~Niklas B{\"o}hm, Philipp Berens, and Dmitry Kobak.
\newblock Attraction-repulsion spectrum in neighbor embeddings.
\newblock {\em Journal of Machine Learning Research}, 23(95):1--32, 2022.

\bibitem{damrich2023contrastive}
Sebastian Damrich, Niklas B{\"o}hm, Fred~A Hamprecht, and Dmitry Kobak.
\newblock From $ t $-{SNE} to {UMAP} with contrastive learning.
\newblock In {\em The Eleventh International Conference on Learning
  Representations}, 2023.

\bibitem{ijcai2023p406}
Andrew Draganov, Jakob Jørgensen, Katrine Scheel, Davide Mottin, Ira Assent,
  Tyrus Berry, and Cigdem Aslay.
\newblock {ActUp}: Analyzing and consolidating {tSNE} and {UMAP}.
\newblock In {\em Proceedings of the Thirty-Second International Joint
  Conference on Artificial Intelligence, {IJCAI-23}}, pages 3651--3658, 8 2023.

\bibitem{abe2024nonlinear}
Motoshi Abe, Yuichiro Nomura, and Takio Kurita.
\newblock Nonlinear dimensionality reduction with q-gaussian distribution.
\newblock {\em Pattern Analysis and Applications}, 27(1):1--12, 2024.

\bibitem{ko2020progressive}
Hyung-Kwon Ko, Jaemin Jo, and Jinwook Seo.
\newblock Progressive uniform manifold approximation and projection.
\newblock In {\em EuroVis (Short Papers)}, pages 133--137, 2020.

\bibitem{senanayake2020self}
Damith~A Senanayake, Wei Wang, Shalin~H Naik, and Saman Halgamuge.
\newblock Self-organizing nebulous growths for robust and incremental data
  visualization.
\newblock {\em IEEE Transactions on Neural Networks and Learning Systems},
  32(10):4588--4602, 2020.

\bibitem{islam2022manifold}
Mohammad~Tariqul Islam and Jason~W Fleischer.
\newblock Manifold-aligned neighbor embedding.
\newblock In {\em ICLR 2022 Workshop on Geometrical and Topological
  Representation Learning}, 2022.

\bibitem{amid2019trimap}
Ehsan Amid and Manfred~K Warmuth.
\newblock {TriMap}: Large-scale dimensionality reduction using triplets.
\newblock {\em arXiv preprint arXiv:1910.00204}, 2019.

\bibitem{wang2021understanding}
Yingfan Wang, Haiyang Huang, Cynthia Rudin, and Yaron Shaposhnik.
\newblock Understanding how dimension reduction tools work: an empirical
  approach to deciphering {t-SNE}, {UMAP}, {TriMAP}, and {PaCMAP} for data
  visualization.
\newblock {\em The Journal of Machine Learning Research}, 22(1):9129--9201,
  2021.

\bibitem{van2009learning}
Laurens Van Der~Maaten.
\newblock Learning a parametric embedding by preserving local structure.
\newblock In {\em Artificial Intelligence and Statistics}, pages 384--391,
  2009.

\bibitem{bunte2012general}
Kerstin Bunte, Michael Biehl, and Barbara Hammer.
\newblock A general framework for dimensionality-reducing data visualization
  mapping.
\newblock {\em Neural Computation}, 24(3):771--804, 2012.

\bibitem{duque2020extendable}
Andr{\'e}s~F Duque, Sacha Morin, Guy Wolf, and Kevin Moon.
\newblock Extendable and invertible manifold learning with geometry regularized
  autoencoders.
\newblock In {\em 2020 IEEE International Conference on Big Data (Big Data)},
  pages 5027--5036. IEEE, 2020.

\bibitem{sainburg2020parametric}
Tim Sainburg, Leland McInnes, and Timothy~Q Gentner.
\newblock Parametric umap embeddings for representation and semisupervised
  learning.
\newblock {\em Neural Computation}, 33(11):2881--2907, 2021.

\bibitem{zhou2021deep}
Zixia Zhou, Xinrui Zu, Yuanyuan Wang, Boudewijn~PF Lelieveldt, and Qian Tao.
\newblock Deep recursive embedding for high-dimensional data.
\newblock {\em IEEE Transactions on Visualization and Computer Graphics},
  28(2):1237--1248, 2021.

\bibitem{huang2024navigating}
Haiyang Huang, Yingfan Wang, and Cynthia Rudin.
\newblock Navigating the effect of parametrization for dimensionality
  reduction.
\newblock {\em Advances in Neural Information Processing Systems},
  37:12977--13019, 2024.

\bibitem{berman2014mapping}
Gordon~J Berman, Daniel~M Choi, William Bialek, and Joshua~W Shaevitz.
\newblock Mapping the stereotyped behaviour of freely moving fruit flies.
\newblock {\em Journal of The Royal Society Interface}, 11(99):20140672, 2014.

\bibitem{polivcar2019embedding}
Pavlin~G Poli{\v{c}}ar, Martin Stra{\v{z}}ar, and Bla{\v{z}} Zupan.
\newblock Embedding to reference t-sne space addresses batch effects in
  single-cell classification.
\newblock In {\em International Conference on Discovery Science}, pages
  246--260. Springer, 2019.

\bibitem{mcinnes2018umap-software}
Leland McInnes, John Healy, Nathaniel Saul, and Lukas Grossberger.
\newblock {UMAP}: Uniform manifold approximation and projection.
\newblock {\em The Journal of Open Source Software}, 3(29):861, 2018.

\bibitem{kingma2014adam}
Diederik~P Kingma and Jimmy Ba.
\newblock Adam: A method for stochastic optimization.
\newblock In {\em International Conference on Learning Representations}, 2015.

\bibitem{paszke2019pytorch}
Adam Paszke, Sam Gross, Francisco Massa, Adam Lerer, James Bradbury, Gregory
  Chanan, Trevor Killeen, Zeming Lin, Natalia Gimelshein, Luca Antiga, et~al.
\newblock Pytorch: An imperative style, high-performance deep learning library.
\newblock {\em Advances in neural information processing systems}, 32, 2019.

\bibitem{harris2020array}
Charles~R Harris, K~Jarrod Millman, St{\'e}fan~J Van Der~Walt, Ralf Gommers,
  Pauli Virtanen, David Cournapeau, Eric Wieser, Julian Taylor, Sebastian Berg,
  Nathaniel~J Smith, et~al.
\newblock Array programming with numpy.
\newblock {\em Nature}, 585(7825):357--362, 2020.

\bibitem{lam2015numba}
Siu~Kwan Lam, Antoine Pitrou, and Stanley Seibert.
\newblock Numba: A llvm-based python jit compiler.
\newblock In {\em Proceedings of the Second Workshop on the LLVM Compiler
  Infrastructure in HPC}, pages 1--6, 2015.

\bibitem{scikit-learn}
F.~Pedregosa, G.~Varoquaux, A.~Gramfort, V.~Michel, B.~Thirion, O.~Grisel,
  M.~Blondel, P.~Prettenhofer, R.~Weiss, V.~Dubourg, J.~Vanderplas, A.~Passos,
  D.~Cournapeau, M.~Brucher, M.~Perrot, and E.~Duchesnay.
\newblock Scikit-learn: Machine learning in {P}ython.
\newblock {\em Journal of Machine Learning Research}, 12:2825--2830, 2011.

\bibitem{venna2001neighborhood}
Jarkko Venna and Samuel Kaski.
\newblock Neighborhood preservation in nonlinear projection methods: An
  experimental study.
\newblock In {\em International Conference on Artificial Neural Networks},
  pages 485--491. Springer, 2001.

\bibitem{cover1967nearest}
Thomas Cover and Peter Hart.
\newblock Nearest neighbor pattern classification.
\newblock {\em IEEE transactions on information theory}, 13(1):21--27, 1967.

\bibitem{lecun2010mnist}
Yann LeCun, Corinna Cortes, and CJ~Burges.
\newblock Mnist handwritten digit database.
\newblock {\em ATT Labs [Online]. Available: http://yann.lecun.com/exdb/mnist},
  2, 2010.

\bibitem{mcinnes2017hdbscan}
Leland McInnes, John Healy, and Steve Astels.
\newblock hdbscan: Hierarchical density based clustering.
\newblock {\em The Journal of Open Source Software}, 2(11):205, 2017.

\bibitem{rsnadataset}
RSNA.
\newblock {\em RSNA pneumonia detection challenge}, 2018.

\bibitem{huang2017densely}
Gao Huang, Zhuang Liu, Laurens Van Der~Maaten, and Kilian~Q Weinberger.
\newblock Densely connected convolutional networks.
\newblock In {\em Proceedings of the IEEE conference on computer vision and
  pattern recognition}, pages 4700--4708, 2017.

\bibitem{russakovsky2015imagenet}
Olga Russakovsky, Jia Deng, Hao Su, Jonathan Krause, Sanjeev Satheesh, Sean Ma,
  Zhiheng Huang, Andrej Karpathy, Aditya Khosla, Michael Bernstein, et~al.
\newblock Imagenet large scale visual recognition challenge.
\newblock {\em International journal of computer vision}, 115(3):211--252,
  2015.

\bibitem{hurley2019visualization}
Nathan~C Hurley, Adrian~D Haimovich, R~Andrew Taylor, and Bobak~J Mortazavi.
\newblock Visualization of emergency department clinical data for interpretable
  patient phenotyping.
\newblock In {\em 2019 4th IEEE/ACM Conference on Connected Health:
  Applications, Systems and Engineering Technologies}, 2019.

\end{thebibliography}


\newpage
\appendix
\onecolumn


\section*{Appendix}

In this appendix, first, we further discuss accumulation and the details of accumulated points from different dataset (\ref{suppsec:accumulation}). After that, we provide another perspective of the repulsion effect through the repulsion strength parameter (\ref{suppsec:scalerep}). Then, we describe additional results from the clinical data (\ref{suppsec:more_clinical_data}). Finally, we give details of the data used for force analysis (\ref{suppsec:supp_afr_rfr}).

\section{Additional discussion on accumulation} \label{suppsec:accumulation}
In the main text, we defined the accumulation $\zeta$ as
\begin{align}
    \zeta = | \{ \mathbf{v} | \mathbf{v}\in\mathbf{v}_\mathcal{C}, \mathbf{v}\notin P \} |,
\end{align}
where $\mathbf{v}_\mathcal{C}$ is the set of test points associated with cluster $\mathcal{C}$ and $P$ is a convex polytope. To count accumulated points in the periphery, we require $P$ to be inside the cluster and count the number of points that lie outside this polytope. We computed $P$ from the convex hull, $H$, of the cluster $\mathcal{C}$.

\begin{proposition}
    The polytope $P$, obtained by transforming the vertices $\{\mathbf{h_i}\}_{i=1}^{k}\in\mathbb{R}^{n}$ of convex polytope $H$ using the transformation $\mathbf{h_i}'=\lambda \mathbf{c} + (1-\lambda) \mathbf{h_i}$, where $\lambda\in\mathbb{R}$ and $\mathbf{c}\in\mathbb{R}^{n}$, is also convex.
\end{proposition}
\begin{proof}
    The corresponding interior angles of the polytopes $H$ are preserved in $P$ under this transformation.
\end{proof}
In the main text, we used $\mathbf{c}=\frac{1}{k} \sum_{i=1}^{k}h_i$ and $\lambda=0.05$ to compute $\zeta$. The test points of MNIST and chest x-ray datasets for various methods, along with the convex hull $H$ and the polytope $P$, are shown in Figs.~\ref{suppfig:mnist_repulsion_effect}-\ref{suppfig:pneumonia_umaps}.


\begin{figure*} [h]
    {\centering
    \includegraphics[width=0.24\textwidth]{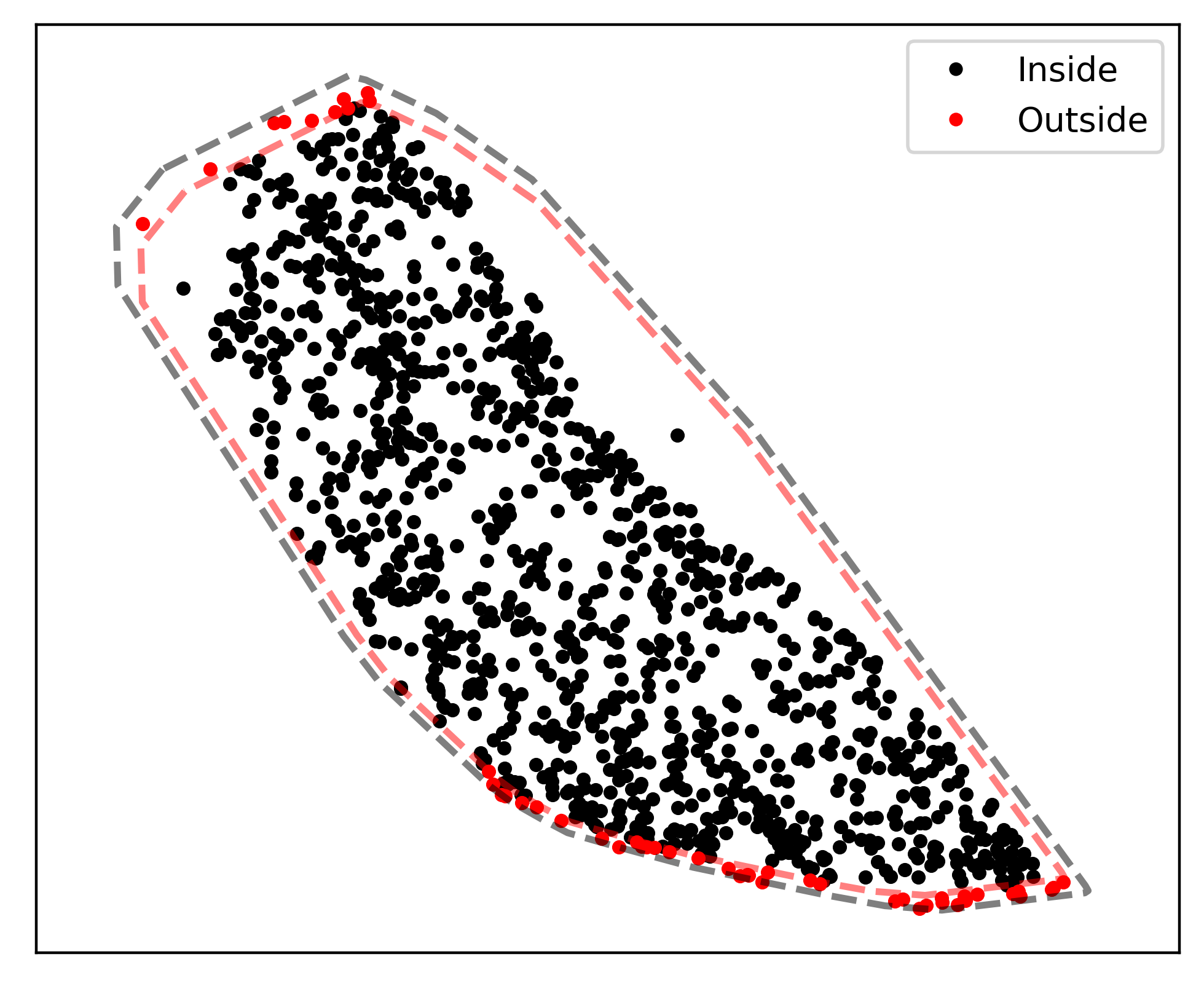}
    \includegraphics[width=0.24\textwidth]{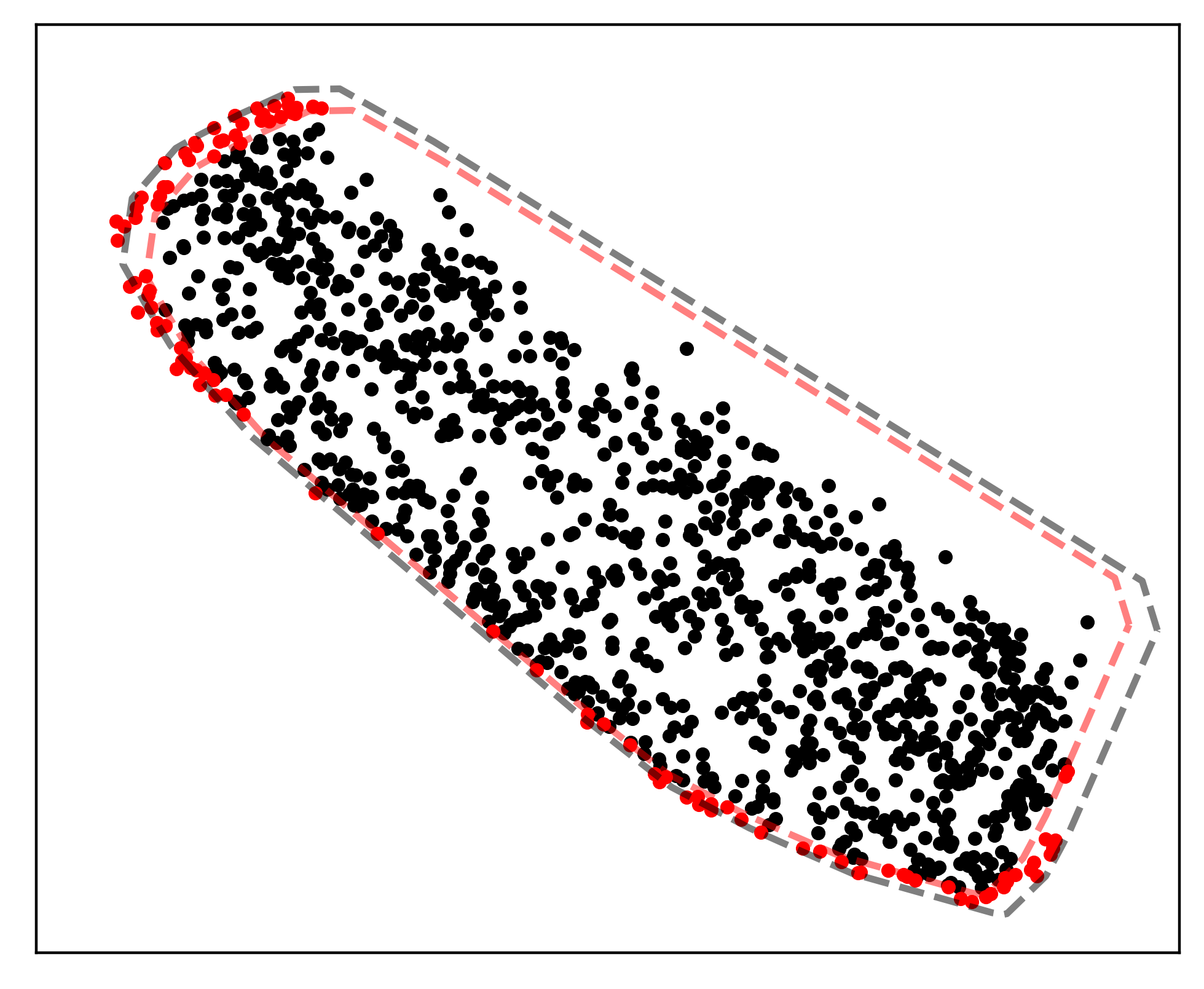}
    \includegraphics[width=0.24\textwidth]{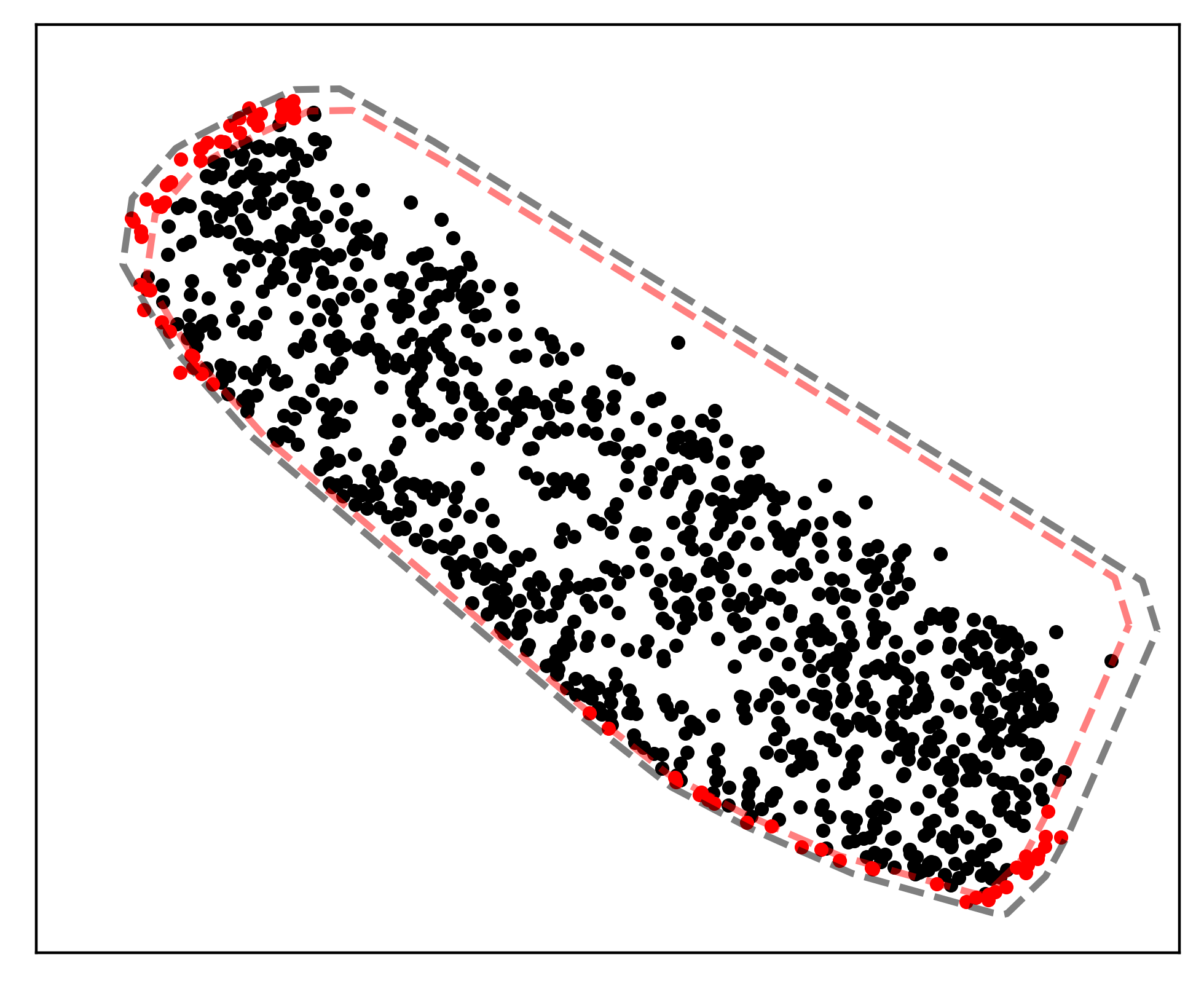}
    \includegraphics[width=0.24\textwidth]{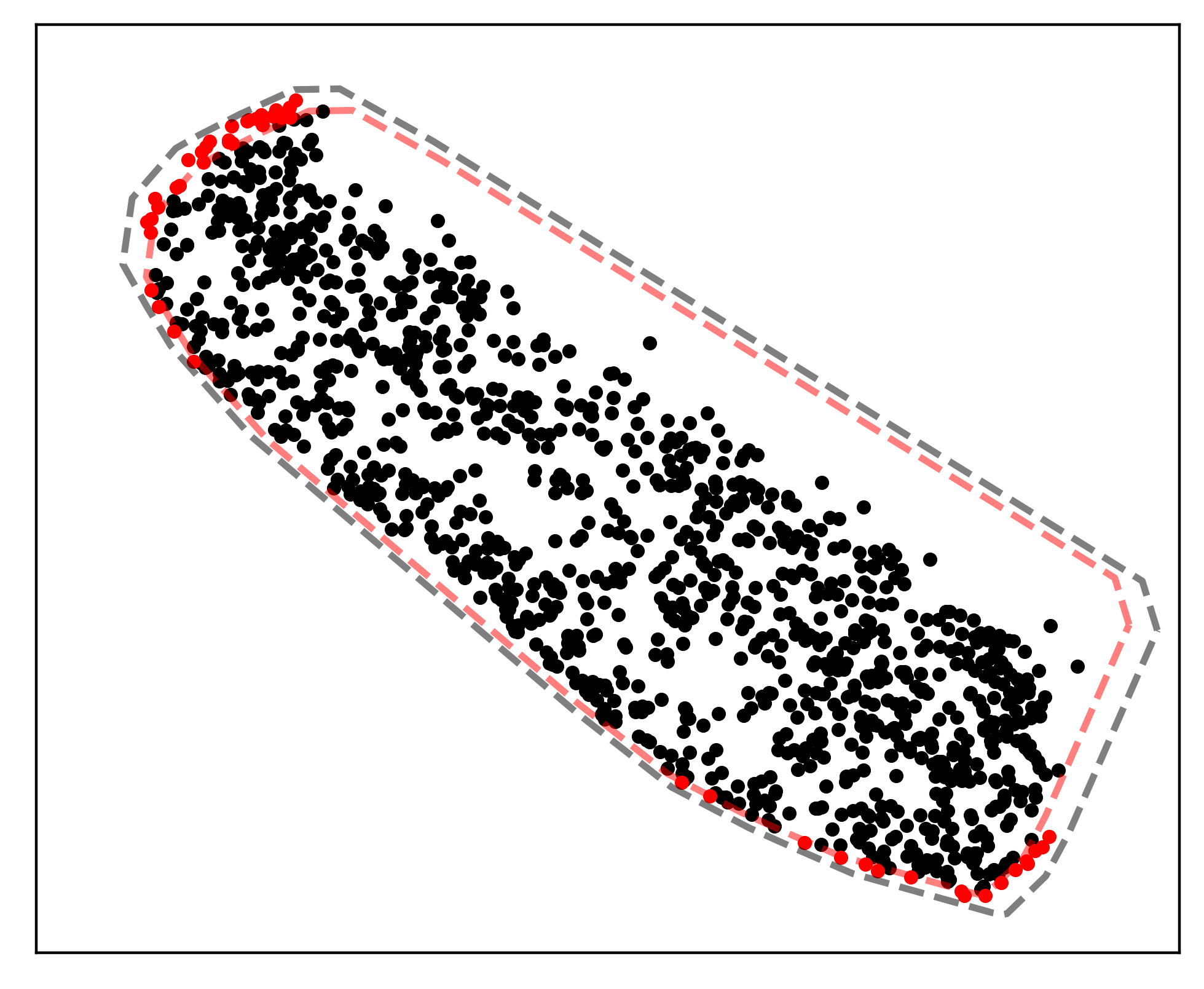}
     \\
    (a) \hspace{0.22\textwidth} (b) \hspace{0.22\textwidth} (c) \hspace{0.22\textwidth} (d) \\}
    {\centering
    \includegraphics[width=0.24\textwidth]{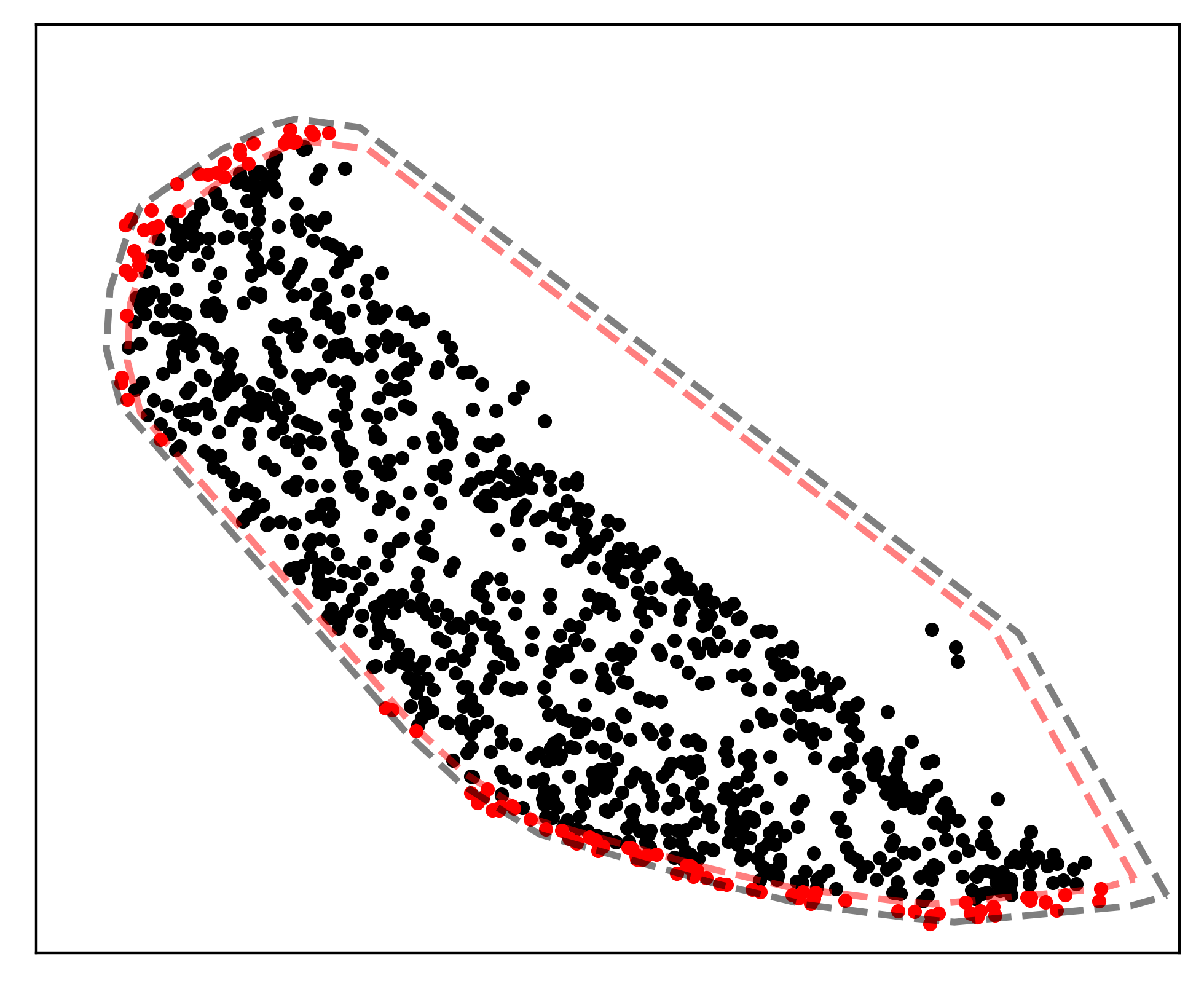}
    \includegraphics[width=0.24\textwidth]{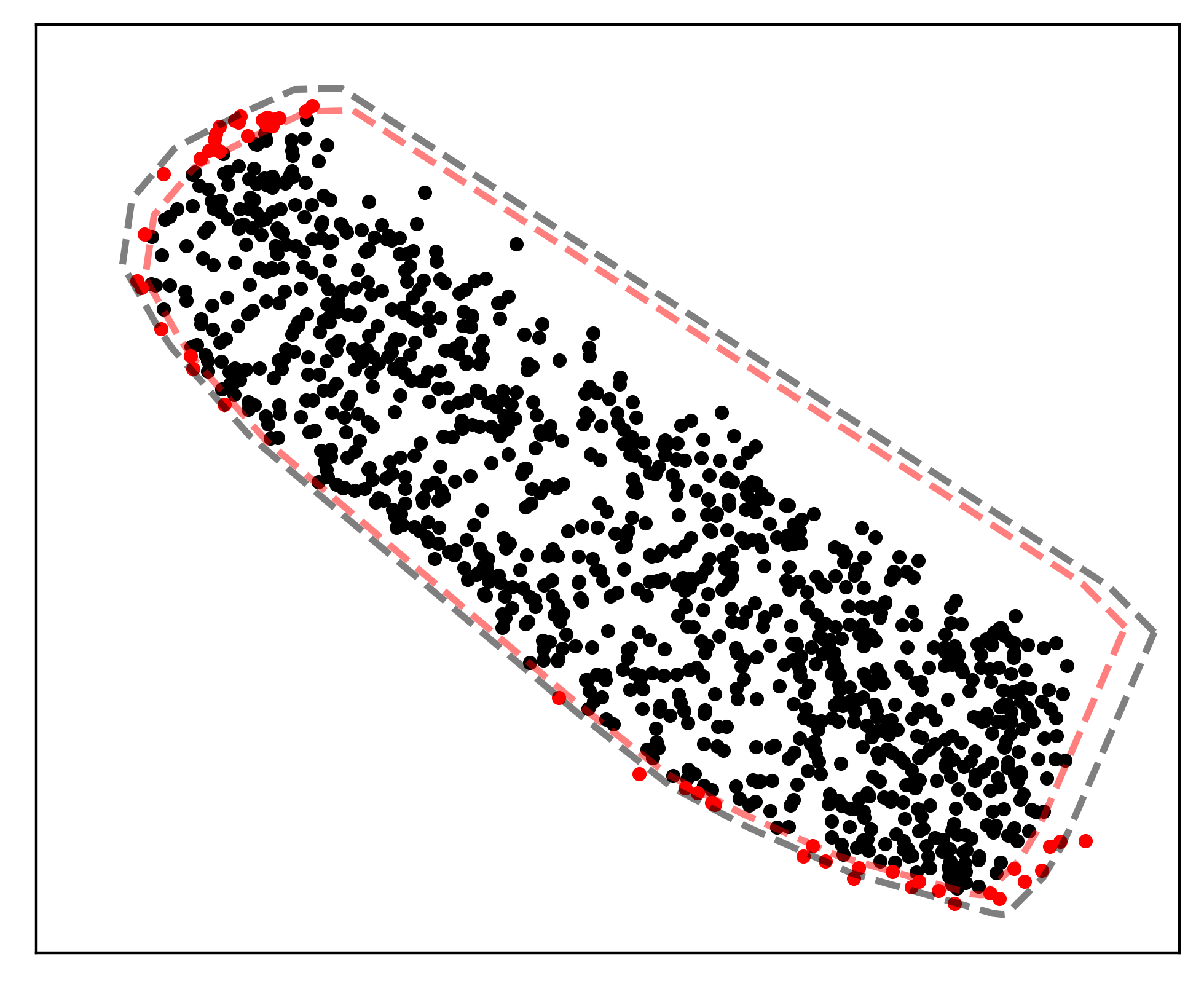}
    \includegraphics[width=0.24\textwidth]{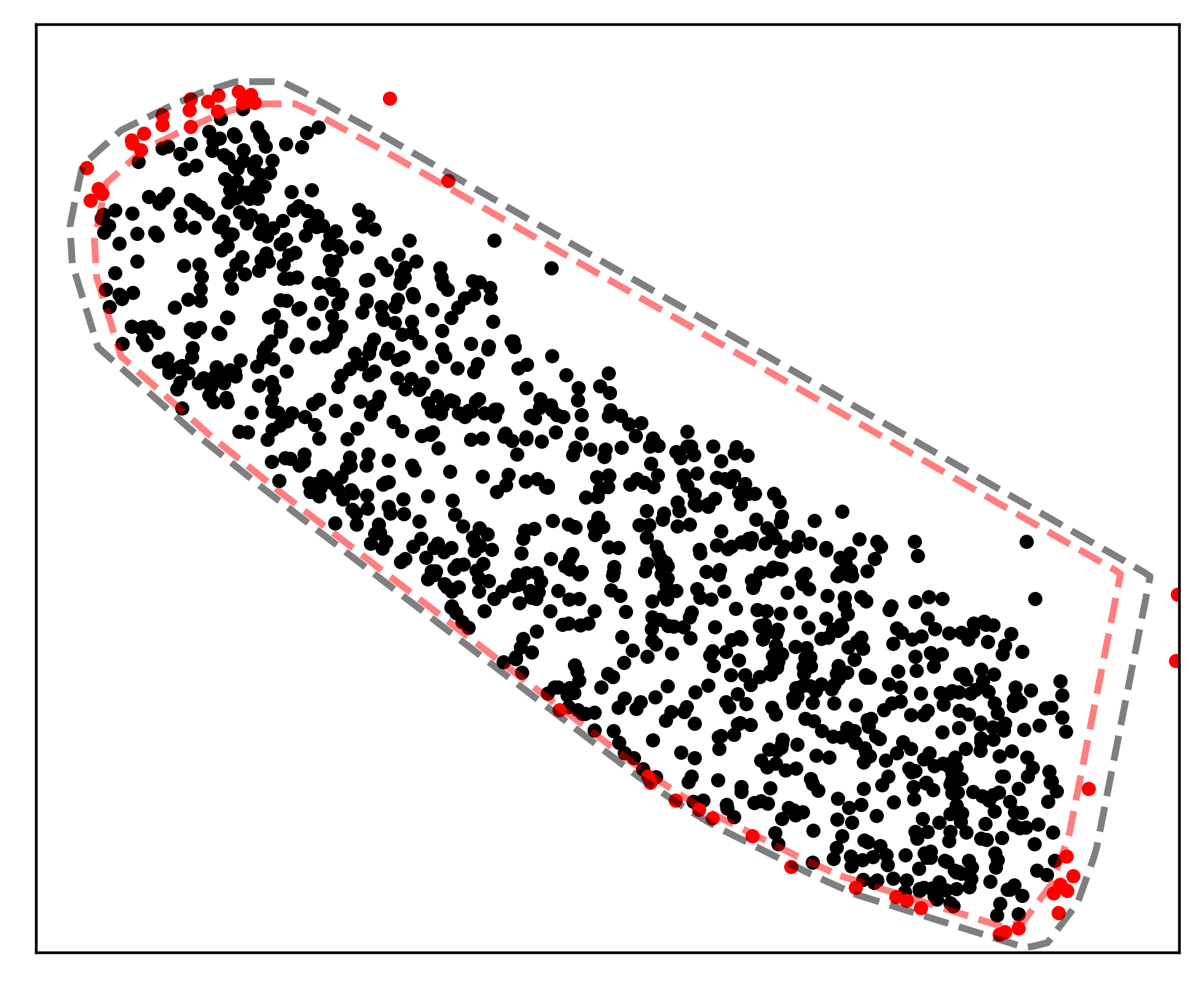}
    \includegraphics[width=0.24\textwidth]{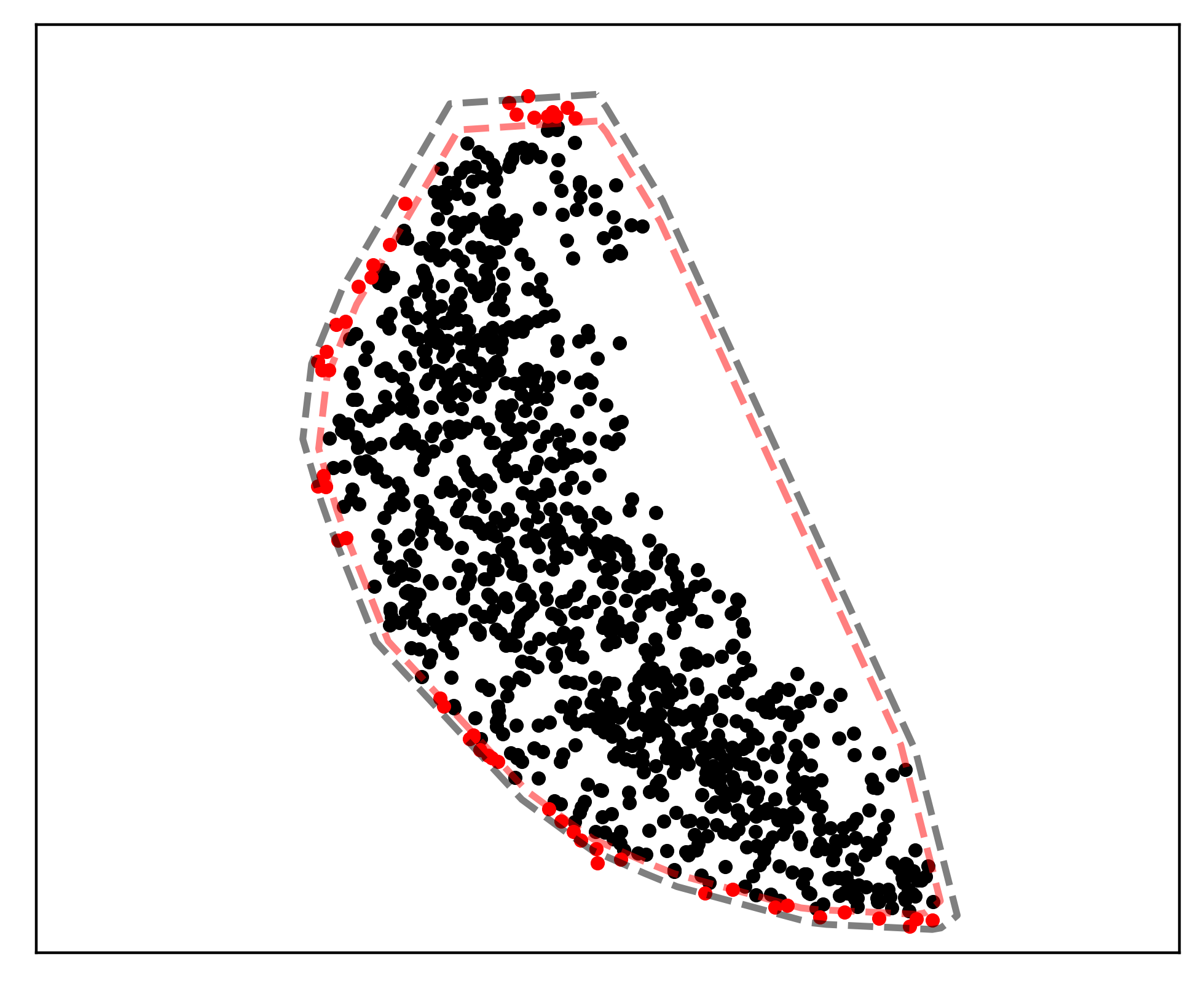} \\
    (e) \hspace{0.22\textwidth} (f) \hspace{0.22\textwidth} (g) \hspace{0.22\textwidth} (h) \\}
    
    \caption{Test points of the cluster labeled 1 of MNIST data for different embeddings. Boundary of polytopes $H$ and $P$ are shown in grey and red dashed lines, respectively. The points outside $P$ are shown using red dots. The methods are: (a) UMAP {\tiny{(Train+test)}} ($k=30$, $n_s=5$, $\zeta=49$), (b) UMAP ($k=30$, $n_s=5$, $\zeta=104$), (c) UMAP ($k=30$, $n_s=3$, $\zeta=75$), (d) UMAP ($k=30$, $n_s=1$, $\zeta=48$),  (e) UMAP ($k=15$, $n_s=5$, $\zeta=105$), (f) UMAP-MSE ($\zeta=53$),  (g) UMAP-CEMSE ($\zeta=46$), and (h) UMAP-CE ($\zeta=51$).}
    \label{suppfig:mnist_repulsion_effect}
\end{figure*}

\begin{figure*} [t]
    {\centering
    \includegraphics[width=0.3\textwidth]{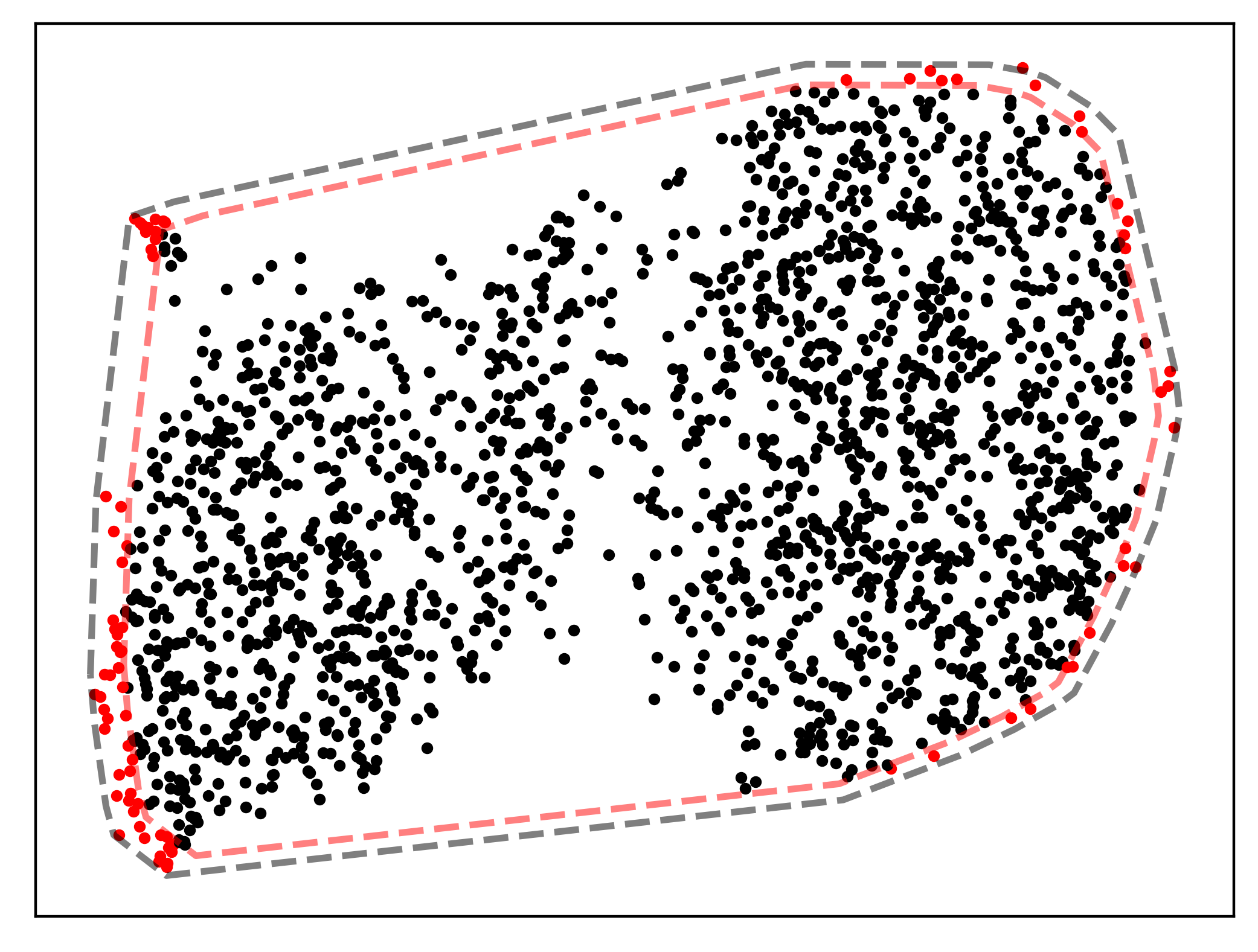} 
    \includegraphics[width=0.3\textwidth]{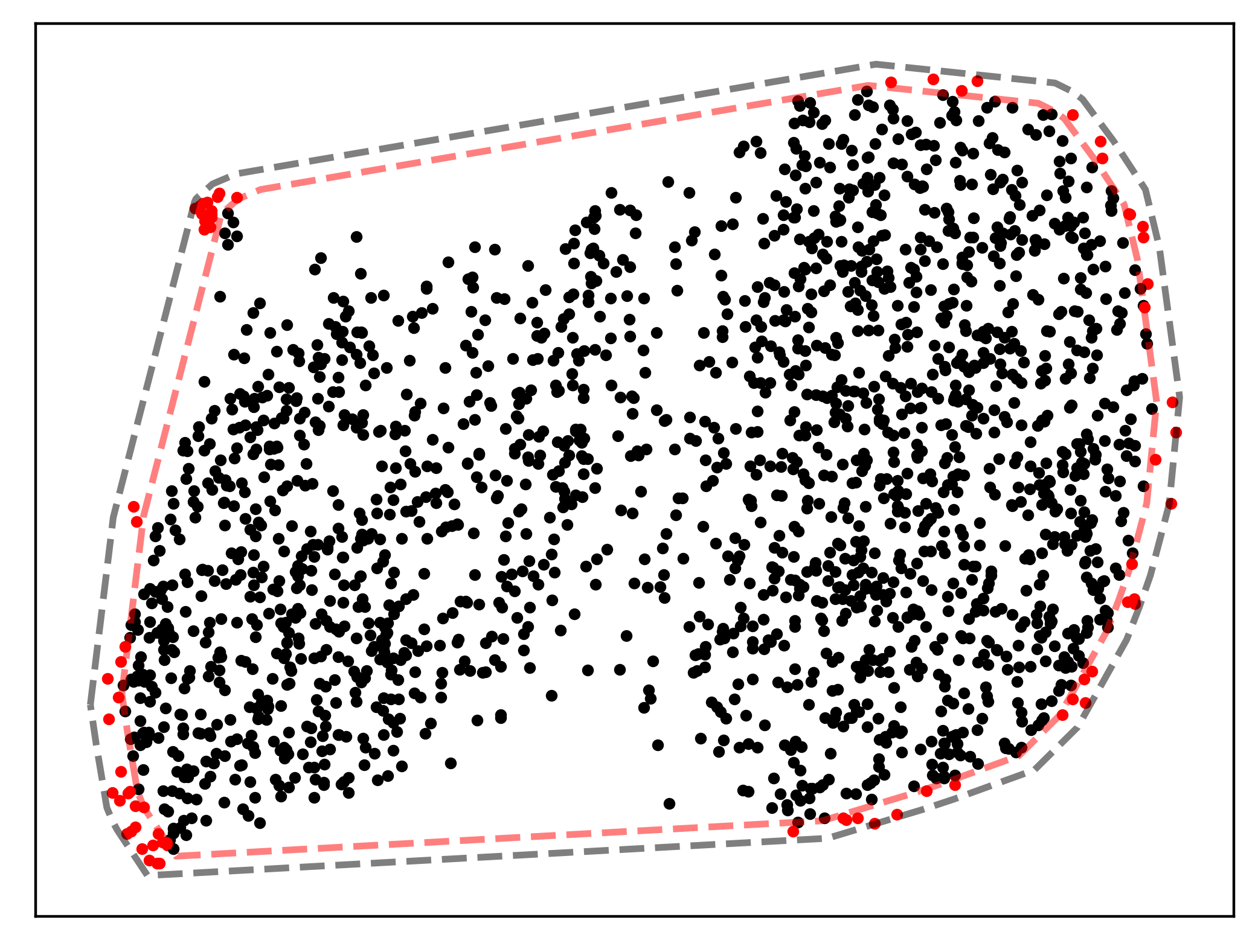} 
    \includegraphics[width=0.3\textwidth]{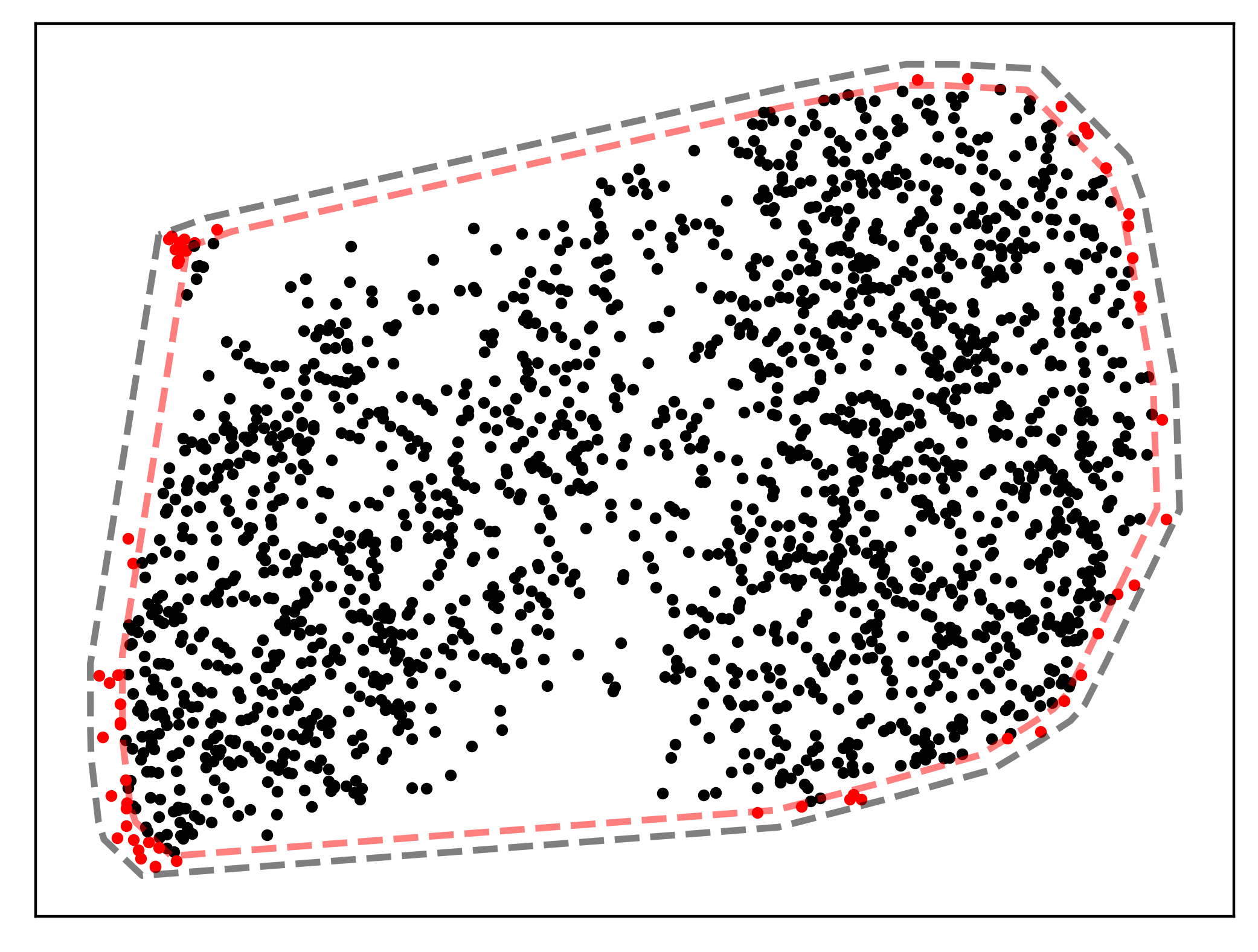} \\
    (a) \hspace{0.28\textwidth} (b) \hspace{0.28\textwidth} (c) \\}
    {\centering
    \includegraphics[width=0.3\textwidth]{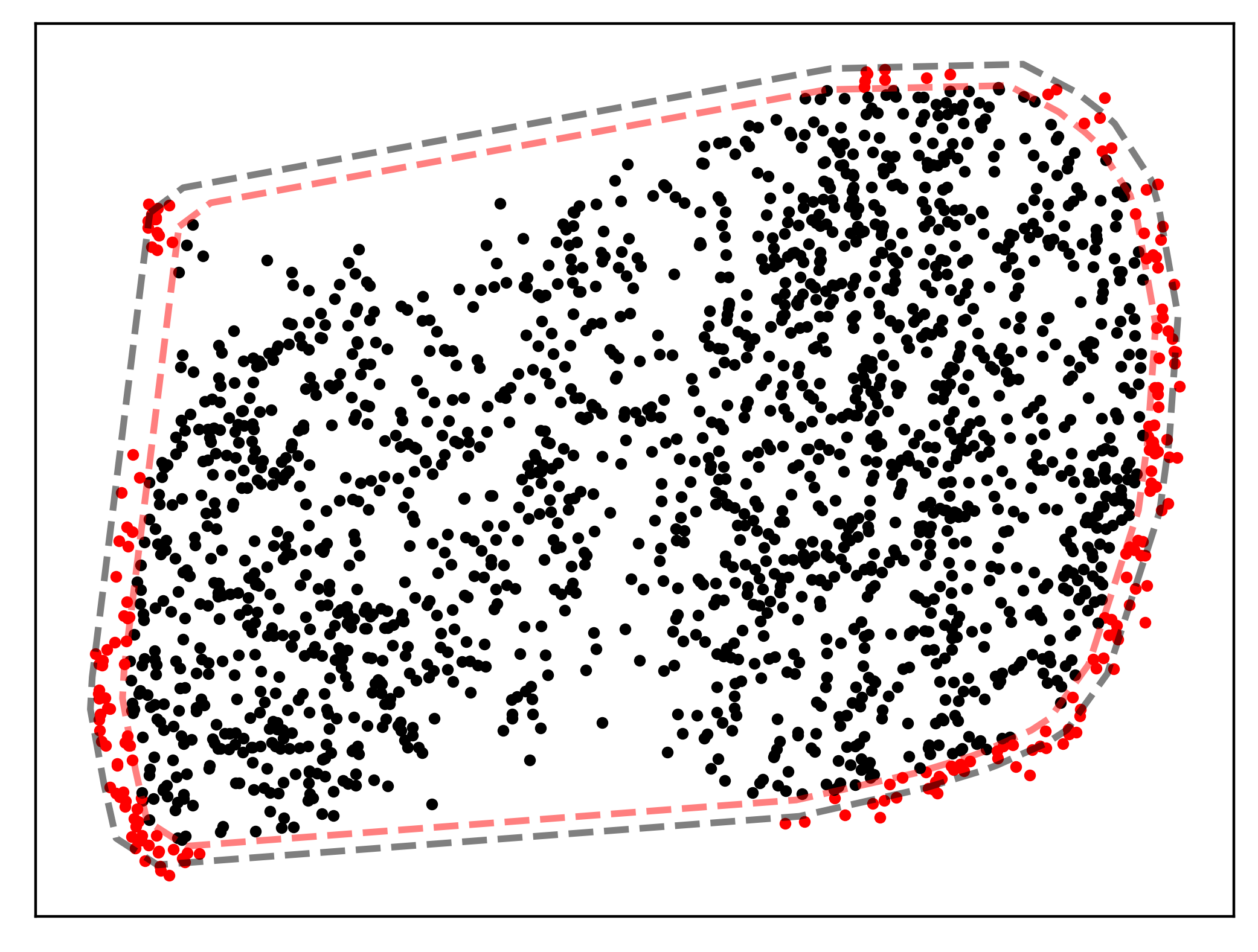} 
    \includegraphics[width=0.3\textwidth]{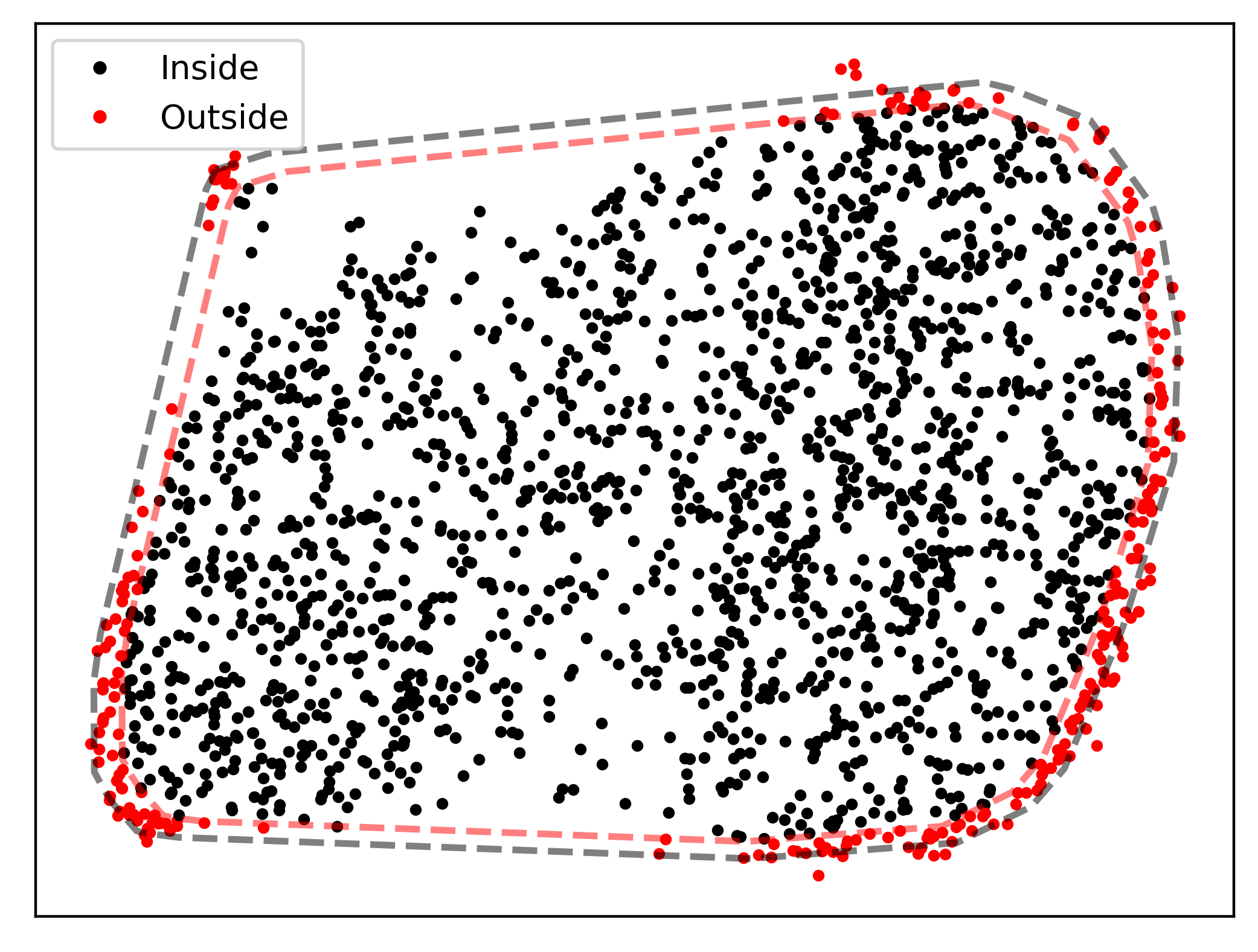} 
    \includegraphics[width=0.3\textwidth]{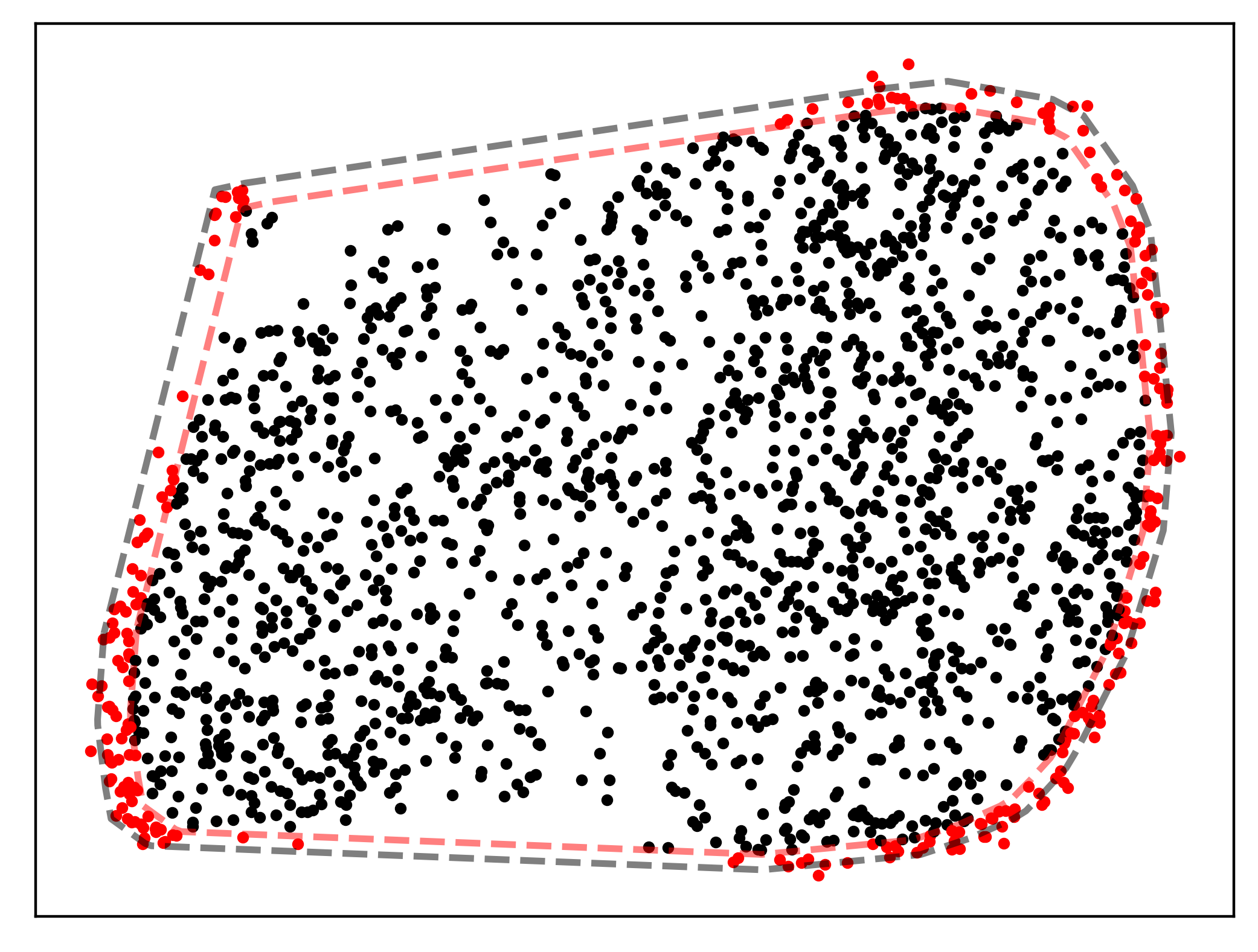} \\
    (d) \hspace{0.28\textwidth} (e) \hspace{0.28\textwidth} (f) \\}
    
    \caption{Test points of chest x-ray embeddings from Fig.~4 using UMAP. Embeddings for x-ray data for different algorithms for $k=30$. Boundary of polytopes $H$ and $P$ are shown in grey and red dashed lines, respectively. The points outside $P$ are shown using red dots. Top row: training and test data embedded together using (a) $k=15$ ($\zeta=82$), (b) $k=30$ ($\zeta=76$) and (c) $k=50$ ($\zeta=62$). Bottom row: Embedded training followed by out-of-sample test points using (d) $k=15$ ($\zeta=202$), (e) $k=245$ ($\zeta=76$) and (f) $k=50$ ($\zeta=243$).}
    \label{suppfig:UMAP_problem}
\end{figure*}

\begin{figure*} [t]
    {\centering
    \includegraphics[width=0.3\textwidth]{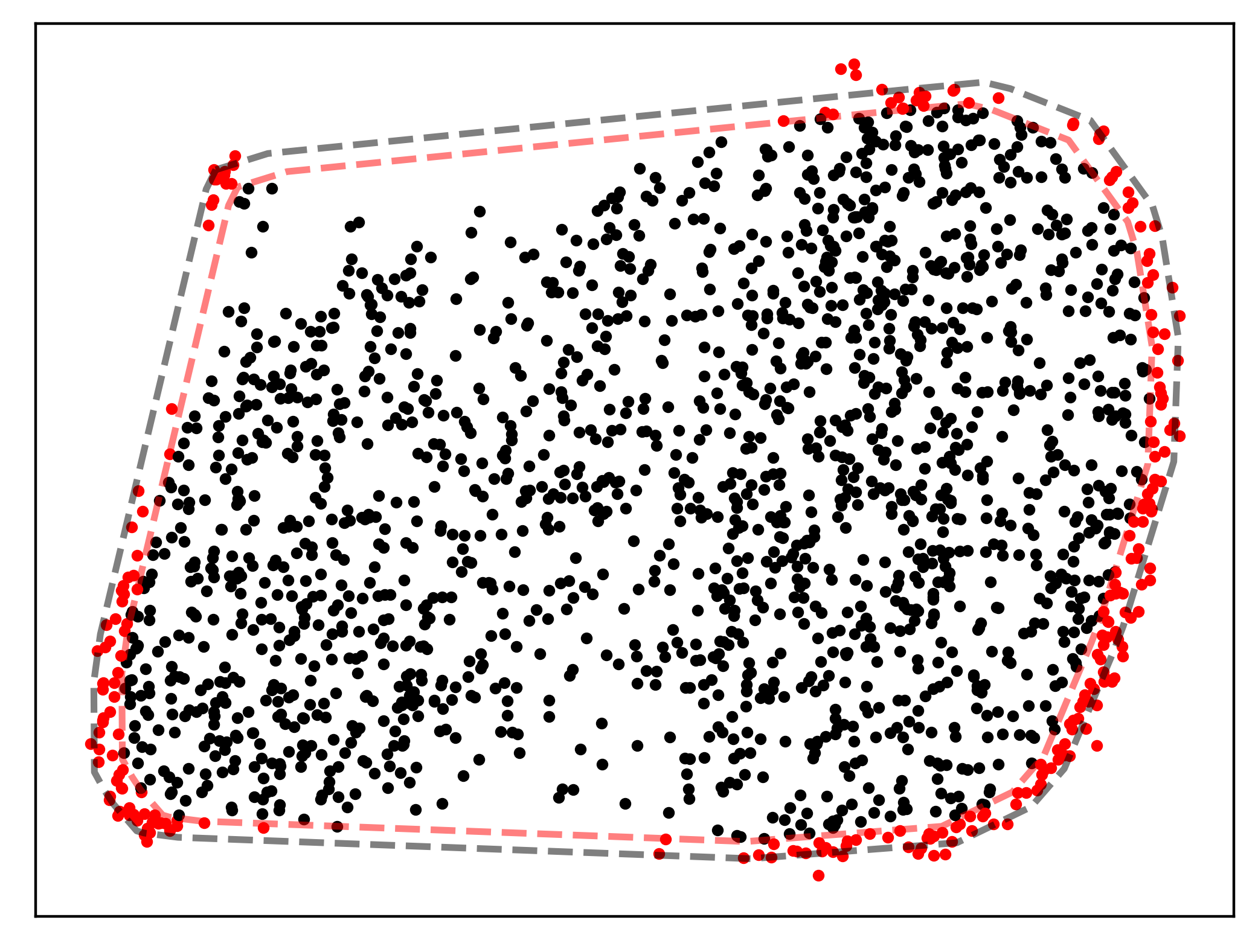} 
    \includegraphics[width=0.3\textwidth]{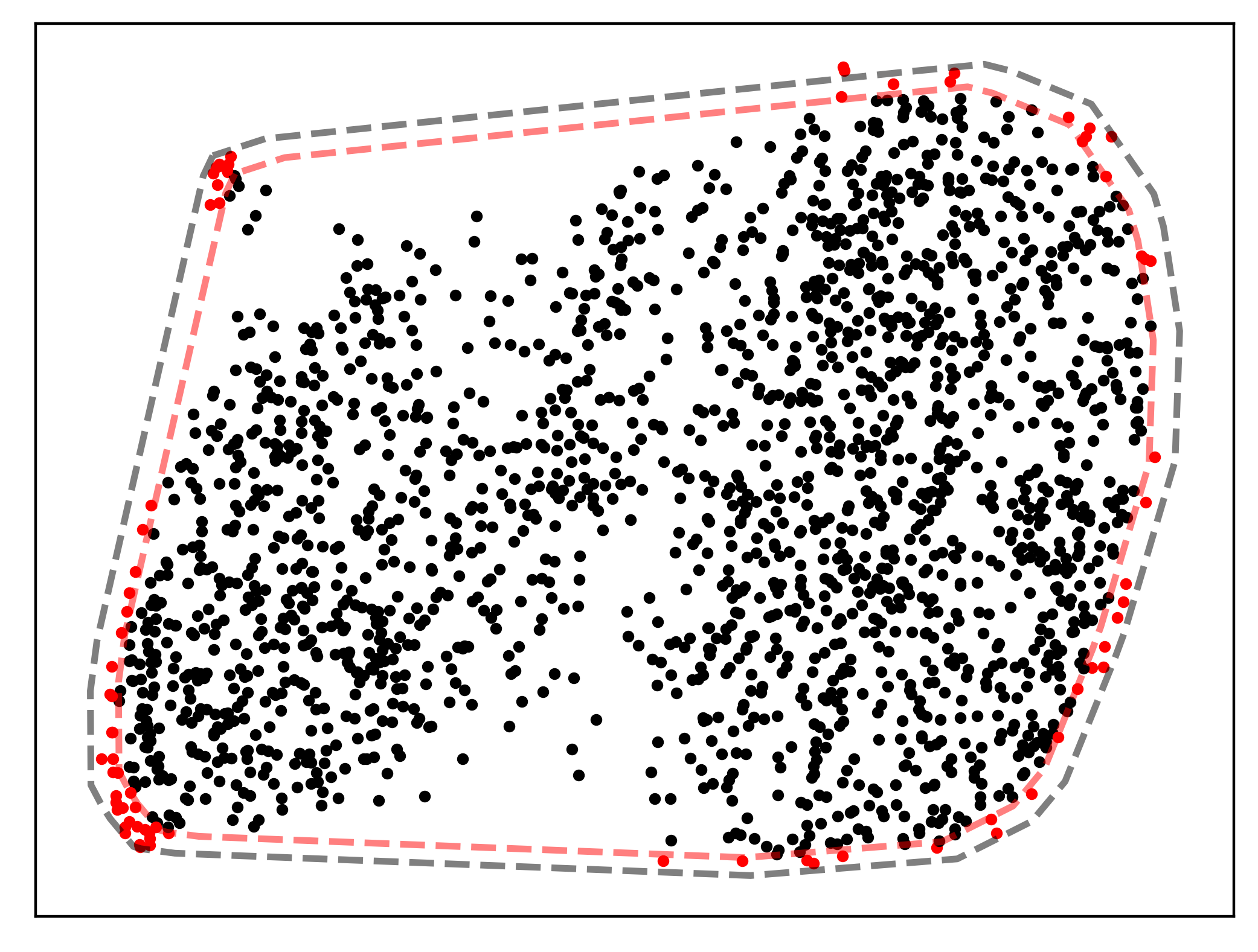} 
    \includegraphics[width=0.3\textwidth]{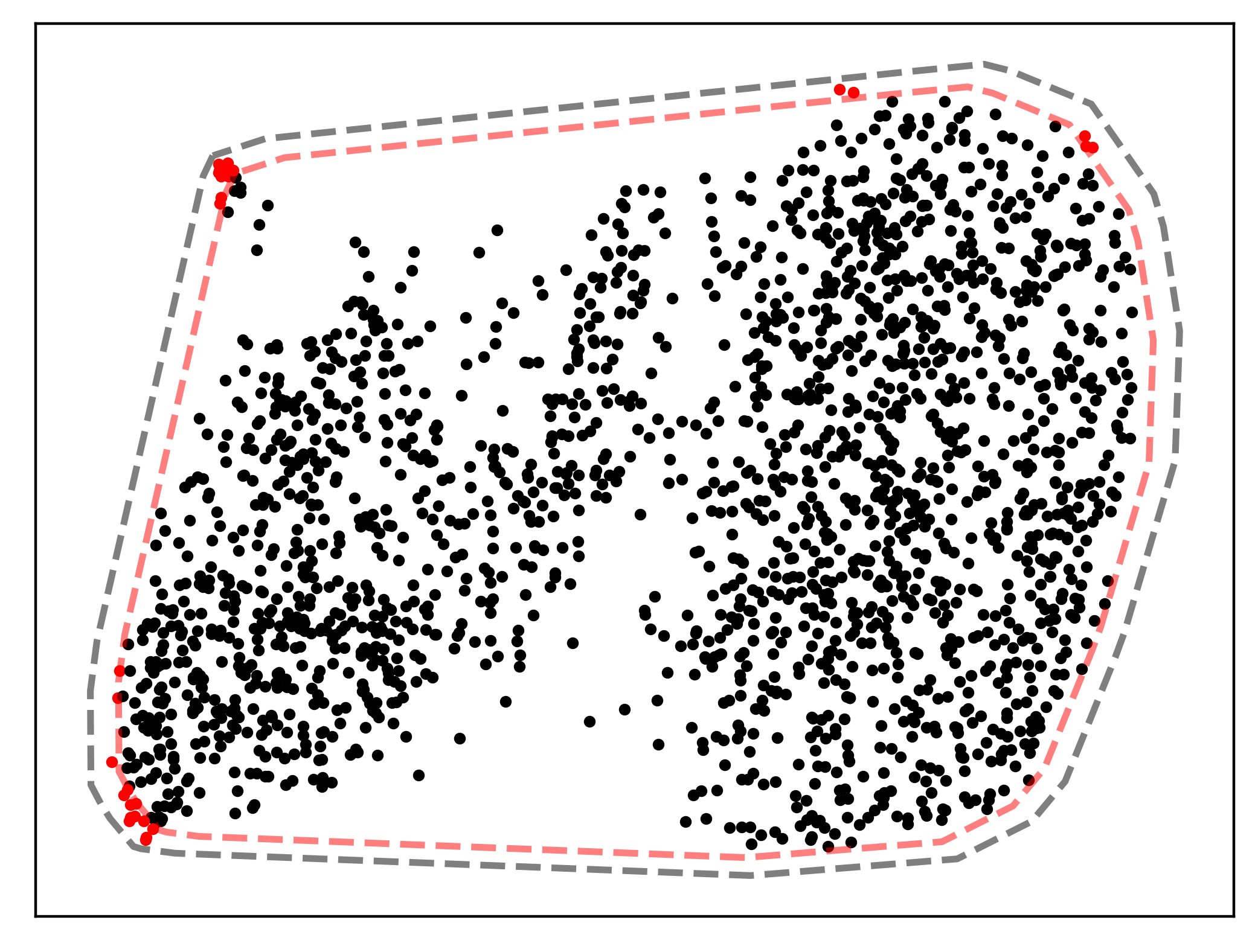} \\
    (a) \hspace{0.28\textwidth} (b) \hspace{0.28\textwidth} (c) \\}
    {\centering
    \includegraphics[width=0.3\textwidth]{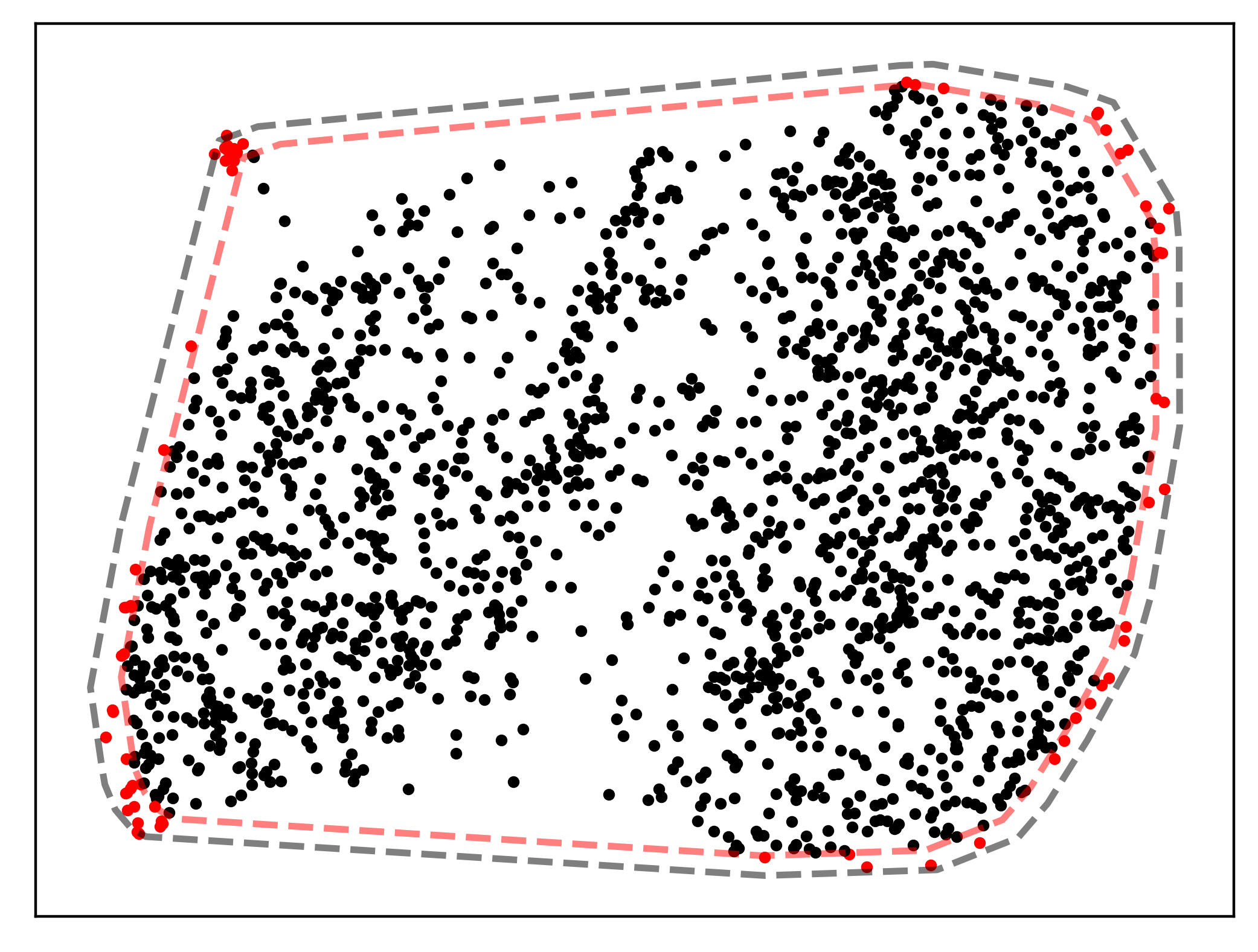} 
    \includegraphics[width=0.3\textwidth]{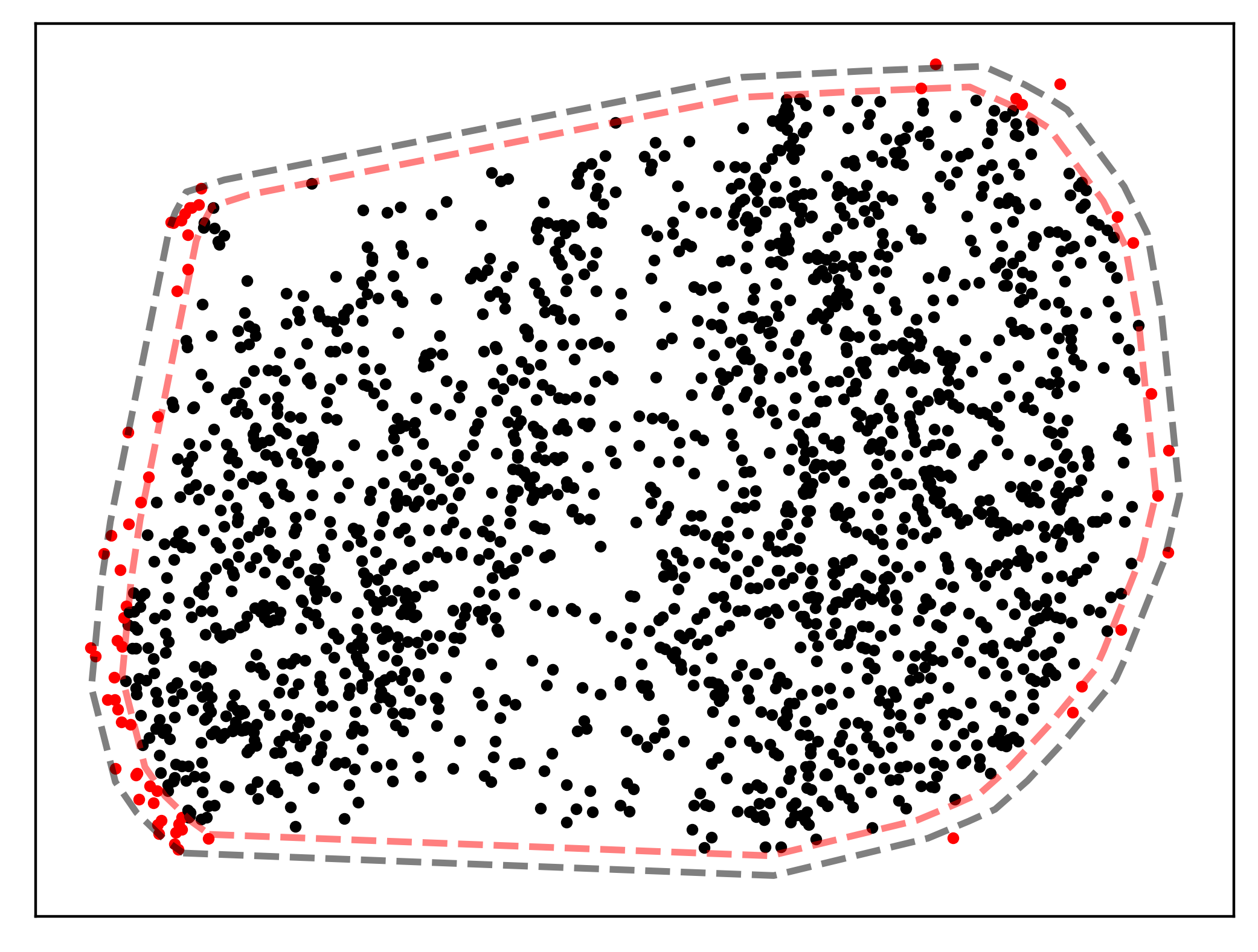} 
    \includegraphics[width=0.3\textwidth]{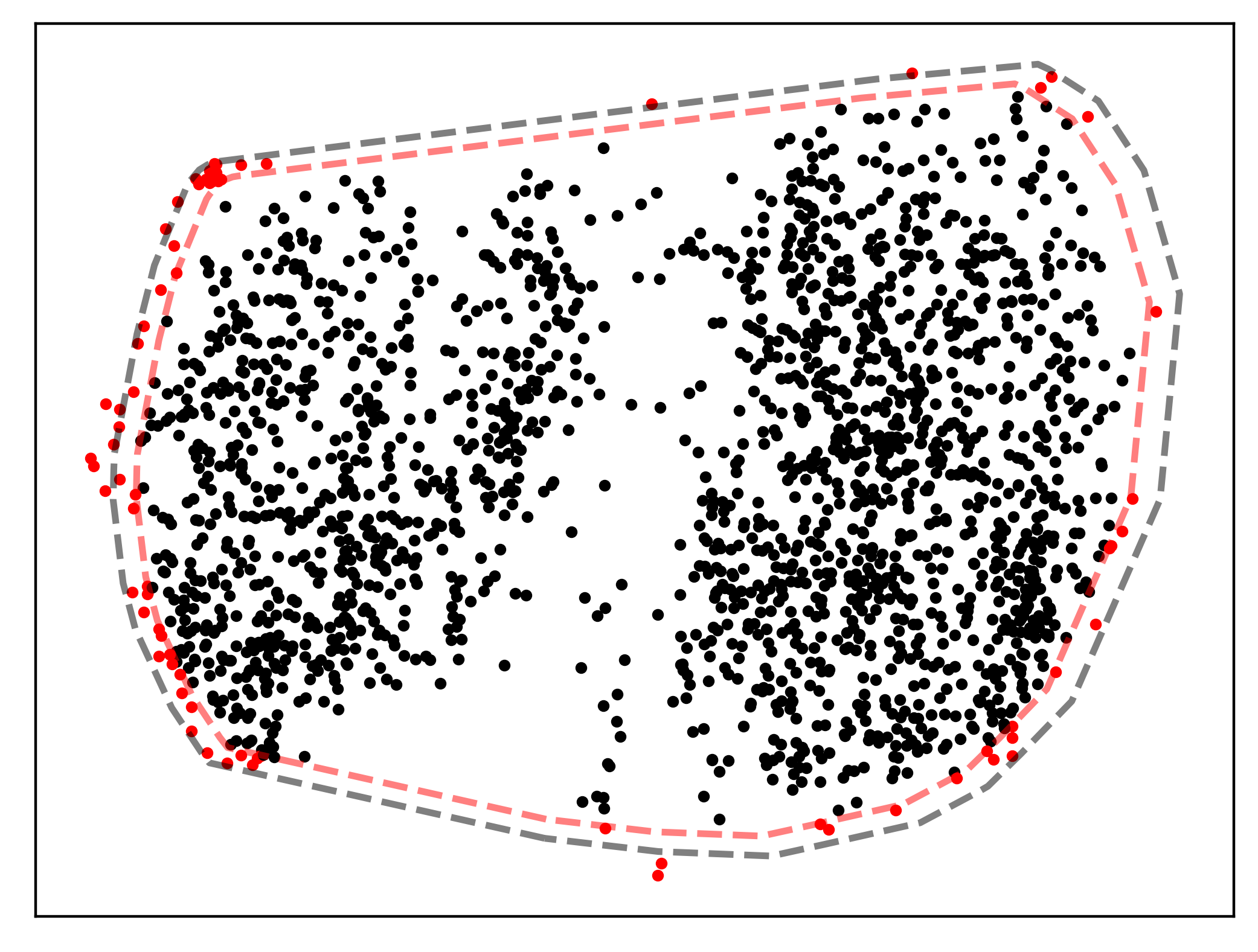} \\
    (d) \hspace{0.28\textwidth} (e) \hspace{0.28\textwidth} (f) \\}
    
    \caption{Test points of chest x-ray embeddings from Fig.~5 using different algorithms. Embeddings for x-ray data for different algorithms for $k=30$. Boundary of polytopes $H$ and $P$ are shown in grey and red dashed lines, respectively. The points outside $P$ are shown using red dots. Top row: UMAP for test points with $n_s$ set to (a) $5$ ($\zeta=245$), (b) $3$ ($\zeta=79$), and (c) $1$ ($\zeta=30$). Bottom row: parameterized embeddings using (d) P-UMAP-MSE ($\zeta=63$), (e) P-UMAP-CEMSE ($\zeta=70$), and (f) P-UMAP-CE ($\zeta=78$).}
    \label{suppfig:pneumonia_umaps}
\end{figure*}

\clearpage

\section{Scaling Repulsive Force}\label{suppsec:scalerep}
In the main text, we showcase the repulsion effect in the non-parametric UMAP by employing a lower negative sampling parameter $n_s$ during testing. Alternatively, the repulsion effect can be demonstrated by introducing the parameter repulsion strength $r_s$ in the update equation by
\begin{align}
v^{(t+a+i)} &= v^{(t+a)} + r_s\eta f_r(v^{(t+a)}, y^{(i)}).
\end{align}
This update equation incorporates the repulsion strength $r_s$ to modify the repulsive force during the iterative process.
The default settings of the UMAP algorithm set $r_s$ to $1$ during both training and testing. 
This value can be lowered during testing to reduce the repulsion effect.
In Fig.~\ref{suppfig:repulsion_strength_figure}, we compare UMAP embeddings for different values of $r_s$ during test time, specifically $0.6$, $1.0$, and $1.2$. 
Accumulation increases when $r_s>1.0$, and decreases when $r_s<1.0$. 
Notably, at $r_s=0.6$, we achieve an accumulation value close to the nominal one ($\zeta=76$).

\begin{figure*} [h]
    {\centering
    \includegraphics[width=0.3\textwidth]{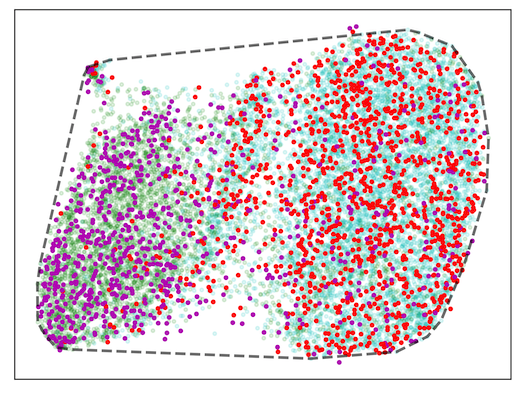} 
    \includegraphics[width=0.3\textwidth]{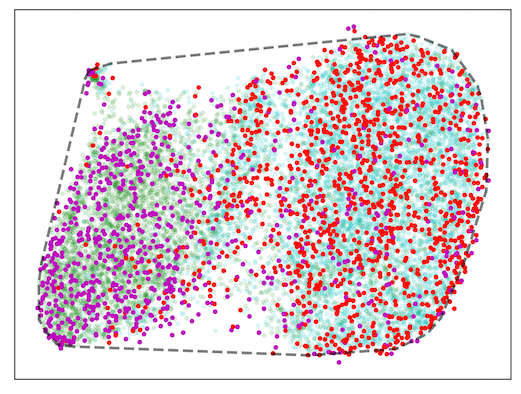} 
    \includegraphics[width=0.3\textwidth]{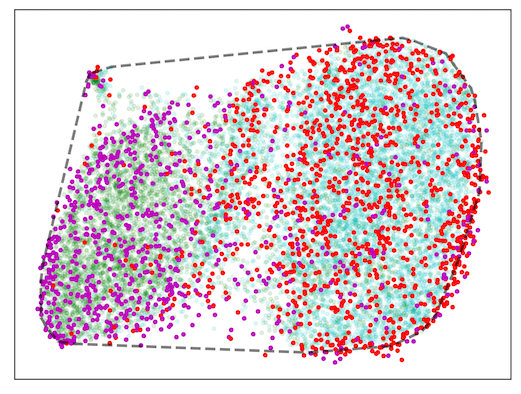} \\
    (a) \hspace{0.28\textwidth} (b) \hspace{0.28\textwidth} (c) \\}
    
    \caption{UMAP embeddings obtained by varying $r_s$ to (a) $0.6$ ($\zeta=81$), (b) $1.0$ ($\zeta=245$), and (c) $1.2$ ($\zeta=390$). The values of $\zeta$ indicate that as $r_s$ is increased, the accumulation of points at the periphery of the clusters also increases.}
    \label{suppfig:repulsion_strength_figure}
\end{figure*}

\section{Additional Discussion regarding regarding clinical data}~\label{suppsec:more_clinical_data}

\begin{figure*}[t]
    \centering
    \includegraphics[width=0.5\linewidth]{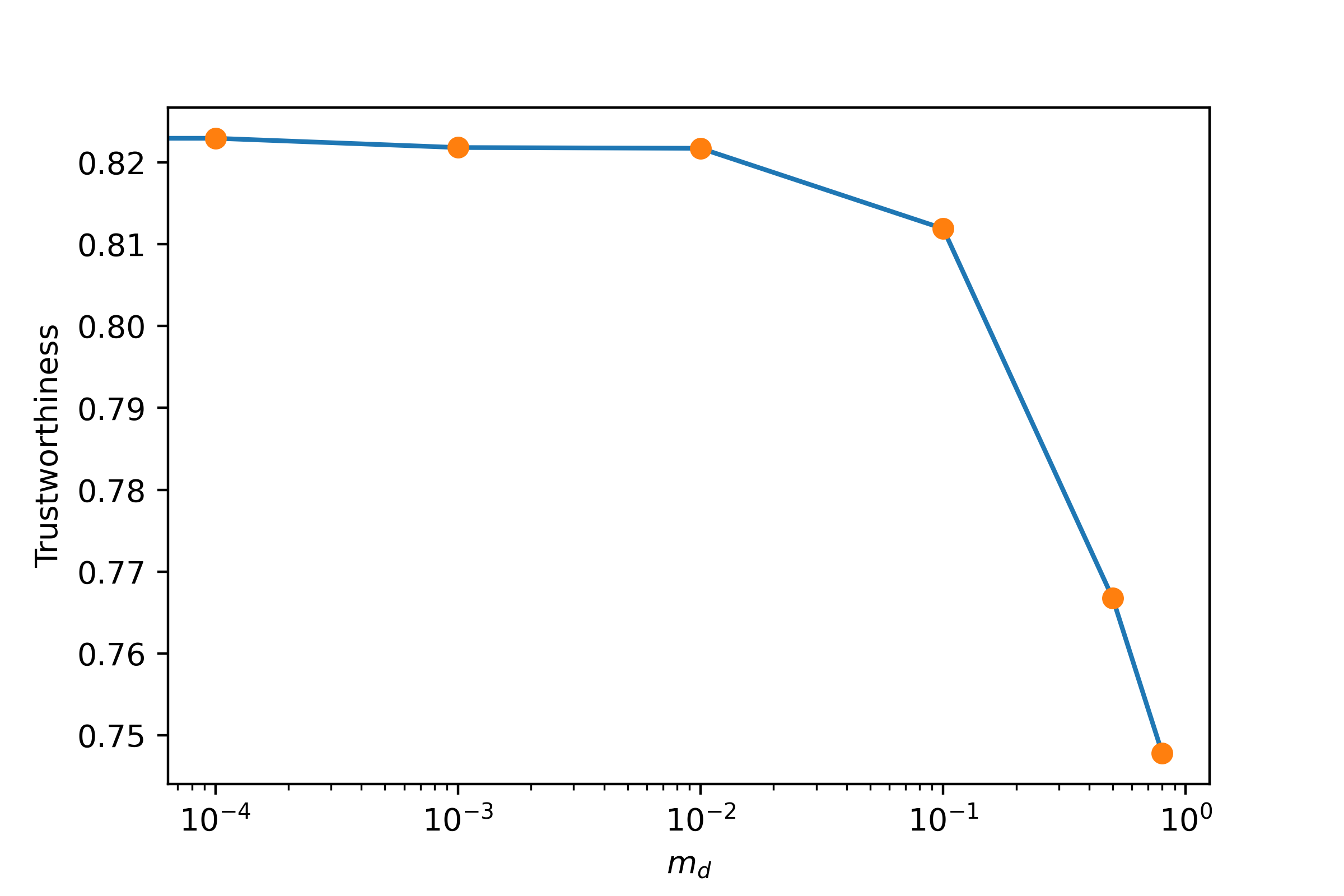}
    \caption{Trustworthiness of training embedding for `shortness of breath' for different values of minmum distance $m_d$.}
    \label{suppfig:trustworthienss_yale_shortness}
\end{figure*}

\begin{figure*}[t]
    \centering
    
    \includegraphics[width=1.6in]{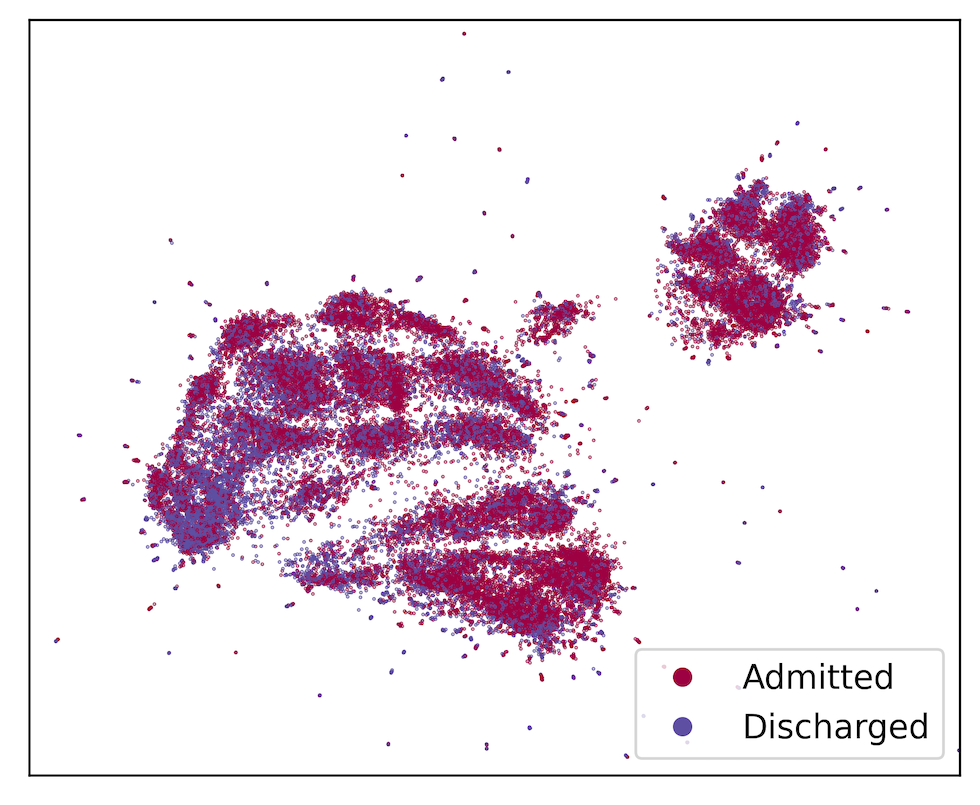} 
    \includegraphics[width=1.6in]{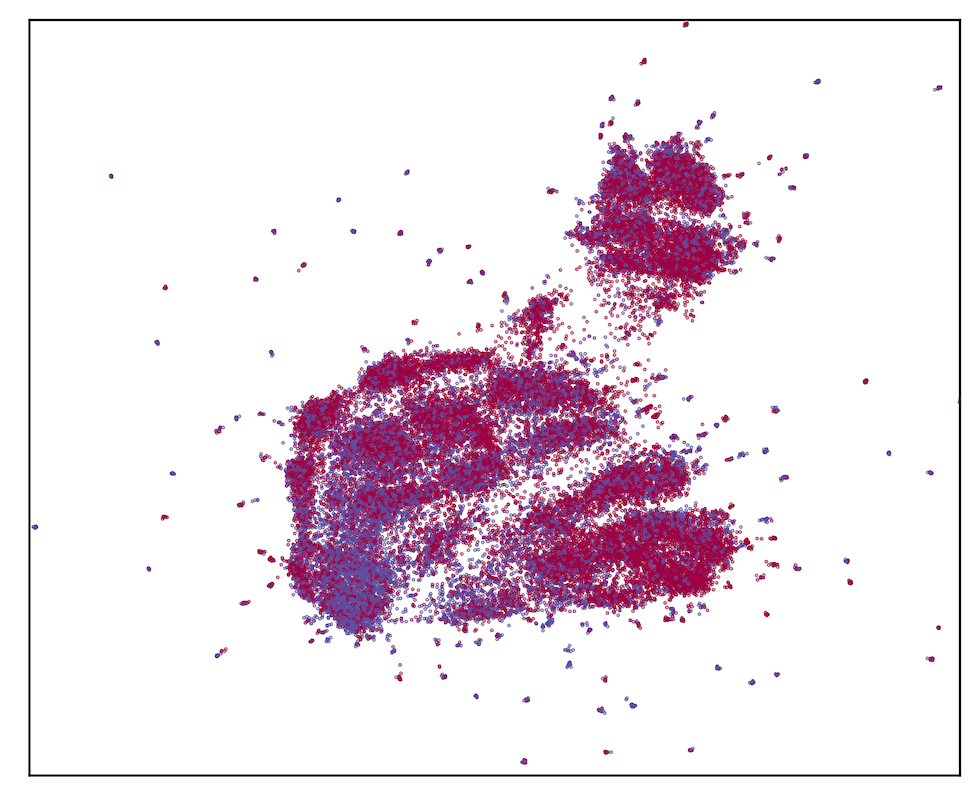} 
    \includegraphics[width=1.6in]{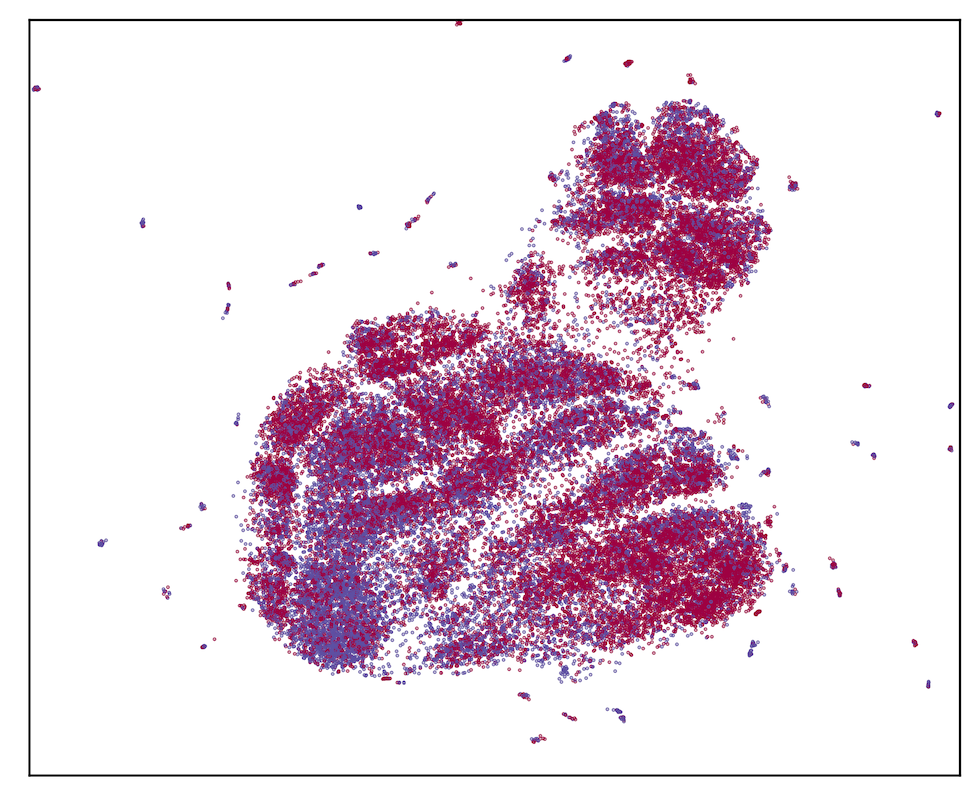} 
    \includegraphics[width=1.6in]{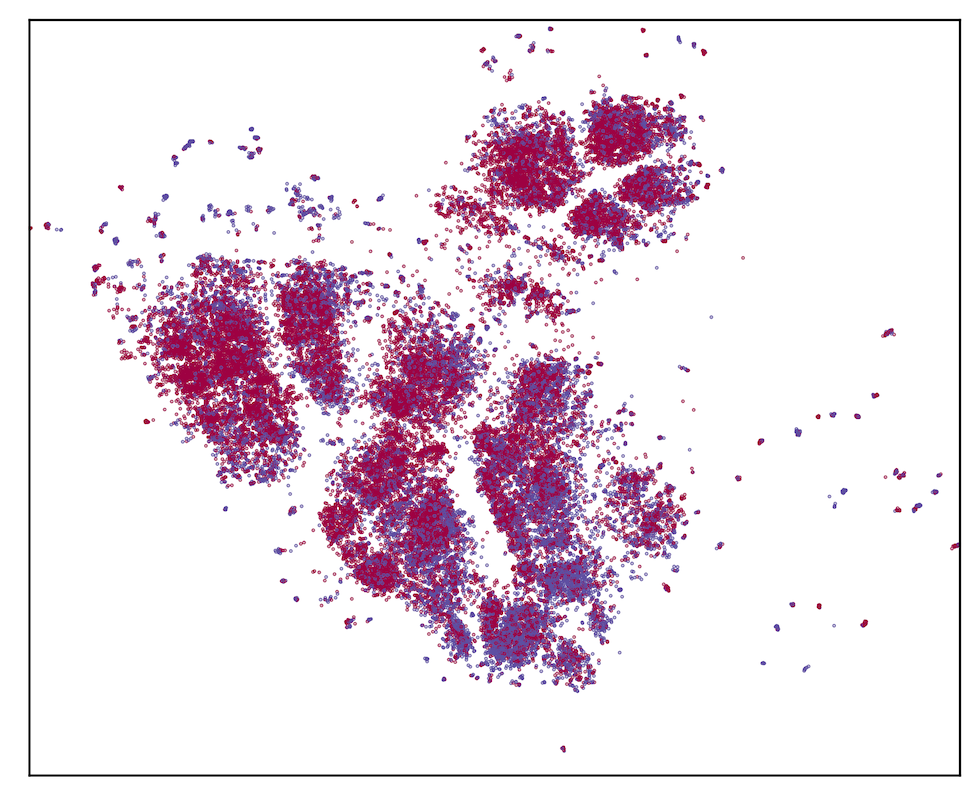} \\
    (a) \hspace{1.45in} (b) \hspace{1.45in} (c) \hspace{1.45in} (d) \\
    
    \caption{Embedding of patient features with `abdominal pain' as the chief complaint. (a) UMAP, (b) P-UMAP-MSE, (c) P-UMAP-CEMSE, (d) P-UMAP-CE.}
    \label{suppfig:abdominal_pain}
\end{figure*}
\begin{table*}[h]
 \caption{Results on Clinical Dataset with `abdominal pain' as the chief complaint.} \label{supptab:yale_data_table_abdominal_pain}
  \centering
  \begin{tabular}{l|l|l|l|l|l|l|l|l|l|}
    \toprule
             & \multicolumn{3}{|c|}{Trustworthiness (Train)} & \multicolumn{3}{|c|}{Trustworthiness (Train+Test)} & \multicolumn{2}{|c|}{k-NN classifier} \\
             \cmidrule(r){2-9}
             & $T_{5}$ & $T_{30}$ & $T_{150}$ & $T_{5}$ & $T_{30}$ & $T_{150}$ & $1$-NN Error & $5$-NN Error  \\
    \midrule
    UMAP ($n_s=5$)  & $0.8423$ & $0.8216$ & $0.8055$ & $0.8338$ & $0.8194$ & $0.8026$ & $41.21\%$ & $36.19\%$ \\
    UMAP ($n_s=3$) & -- & -- & -- & $0.8362$ & $0.8217$ & $0.8050$ & $40.66\%$ & $35.55\%$ \\
    P-UMAP-MSE    & $0.8285$ & $0.8128$ & $0.7973$ & $0.8200$ & $0.8095$ & $0.7940$ & $40.58\%$ & $35.51\%$  \\
    P-UMAP-CEMSE  & $0.8388$ & $0.8290$ & $0.8388$ & $0.8350$ & $0.8300$ & $0.8139$ & $40.03\%$ & $35.21\%$  \\
    P-UMAP-CE     & $\mathbf{0.8798}$ & $\mathbf{0.8577}$ & $\mathbf{0.8404}$ & $\mathbf{0.8730}$ & $\mathbf{0.8569}$ & $\mathbf{0.8404}$ & $\mathbf{39.24}\%$ & $\mathbf{34.78}\%$  \\
    
    \bottomrule
  \end{tabular}
\end{table*}

\begin{figure*} [t]
    \centering
    \includegraphics[width=1.5in]{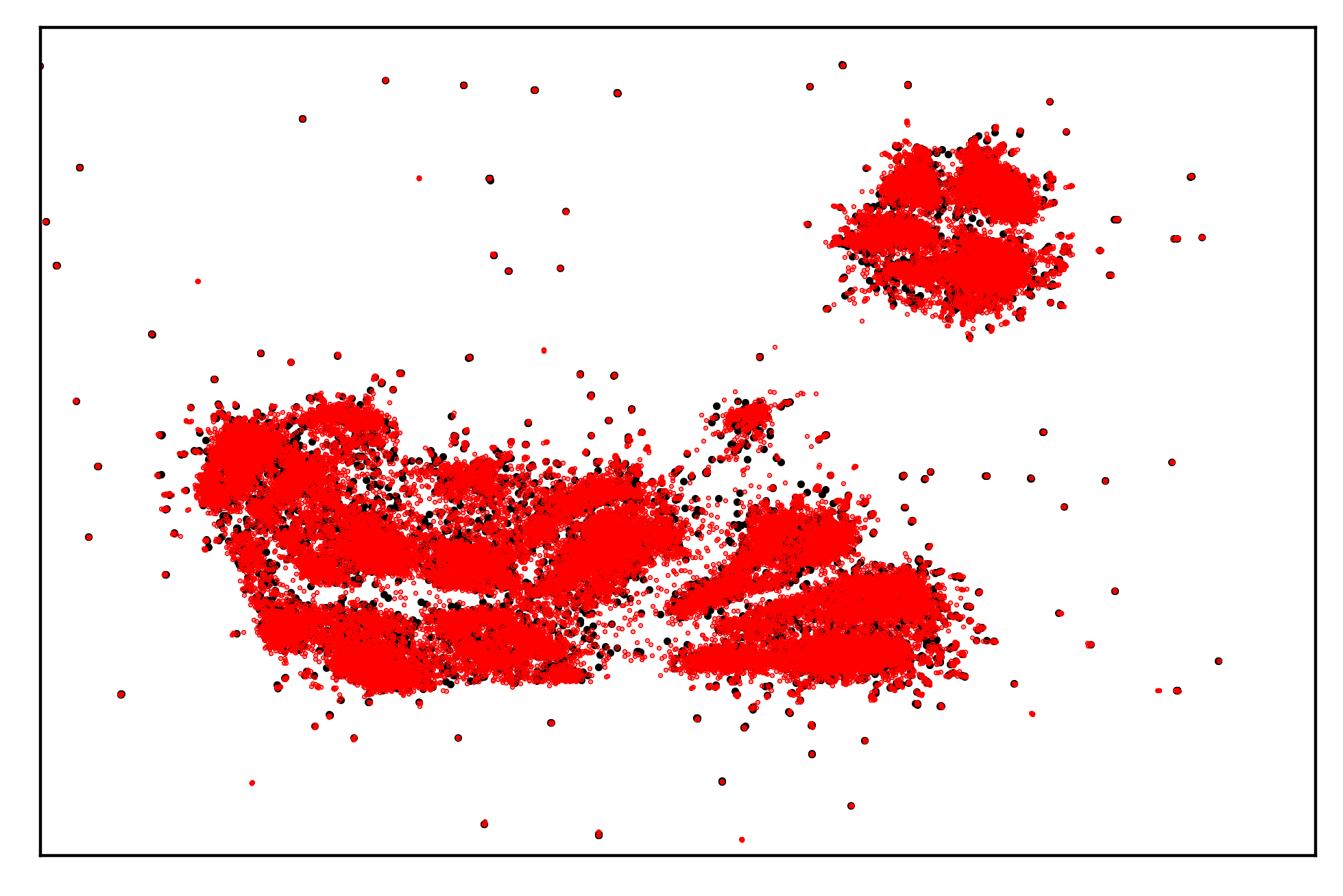}
    \includegraphics[width=1.5in]{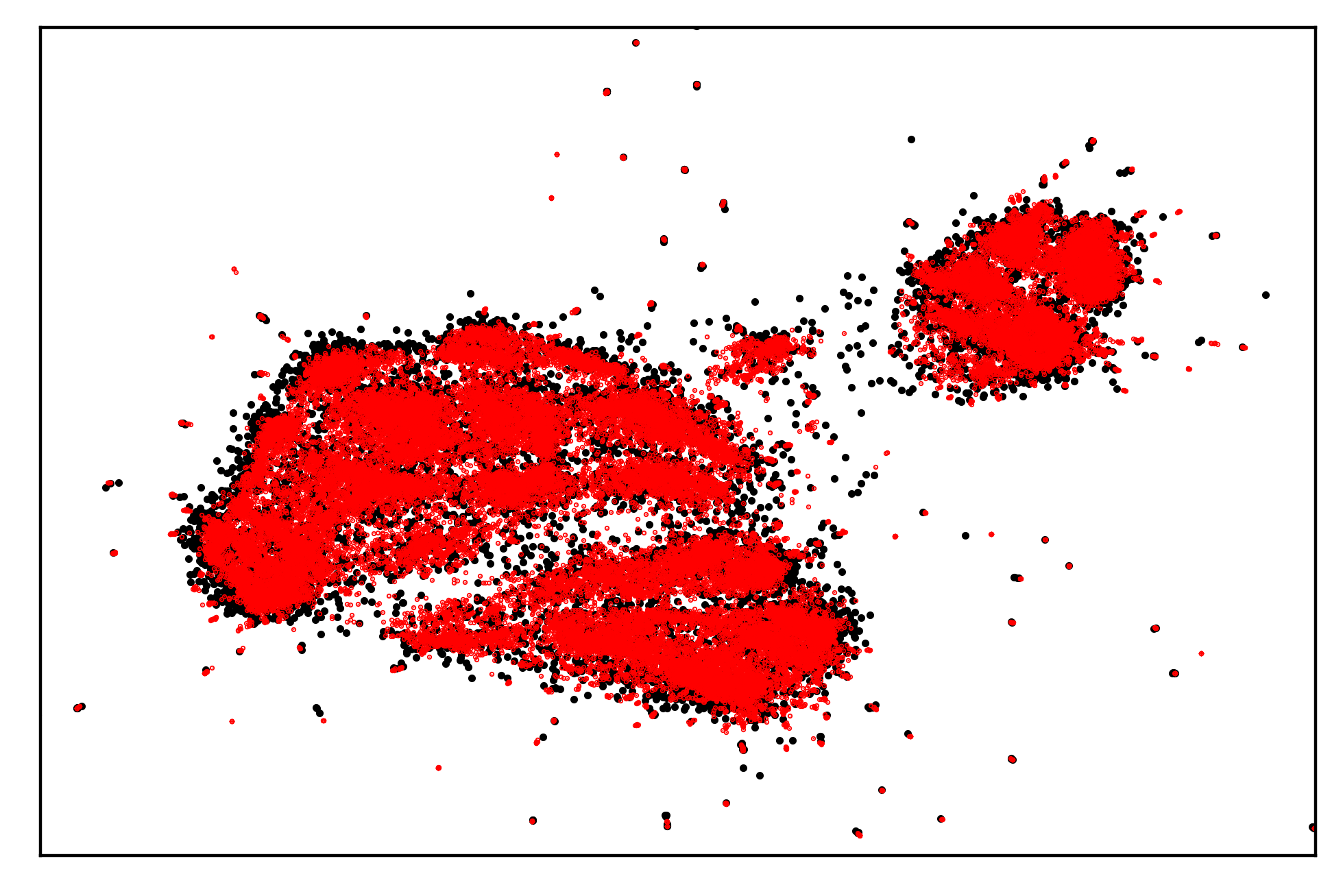}
    \includegraphics[width=1.5in]{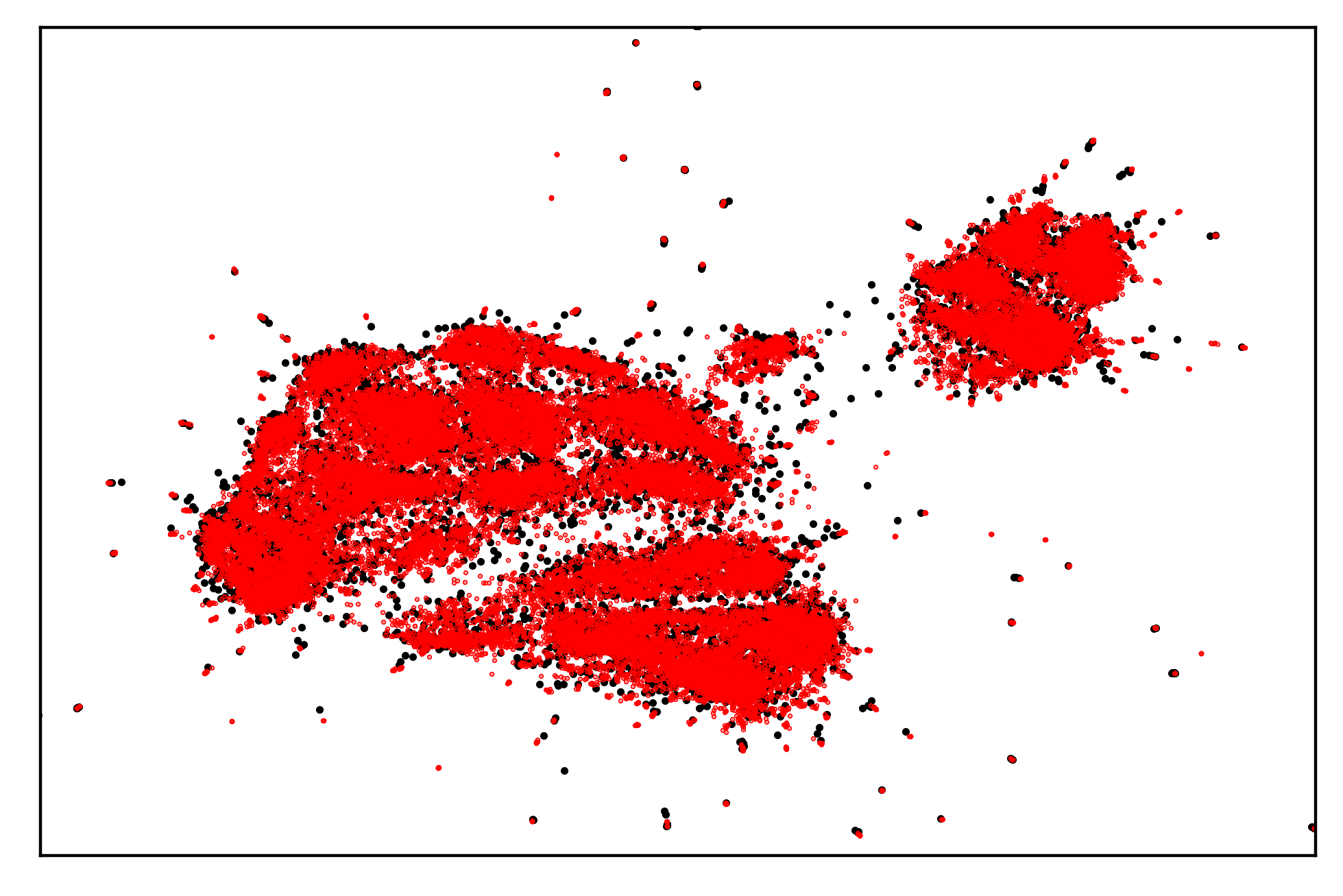}
    \includegraphics[width=1.5in]{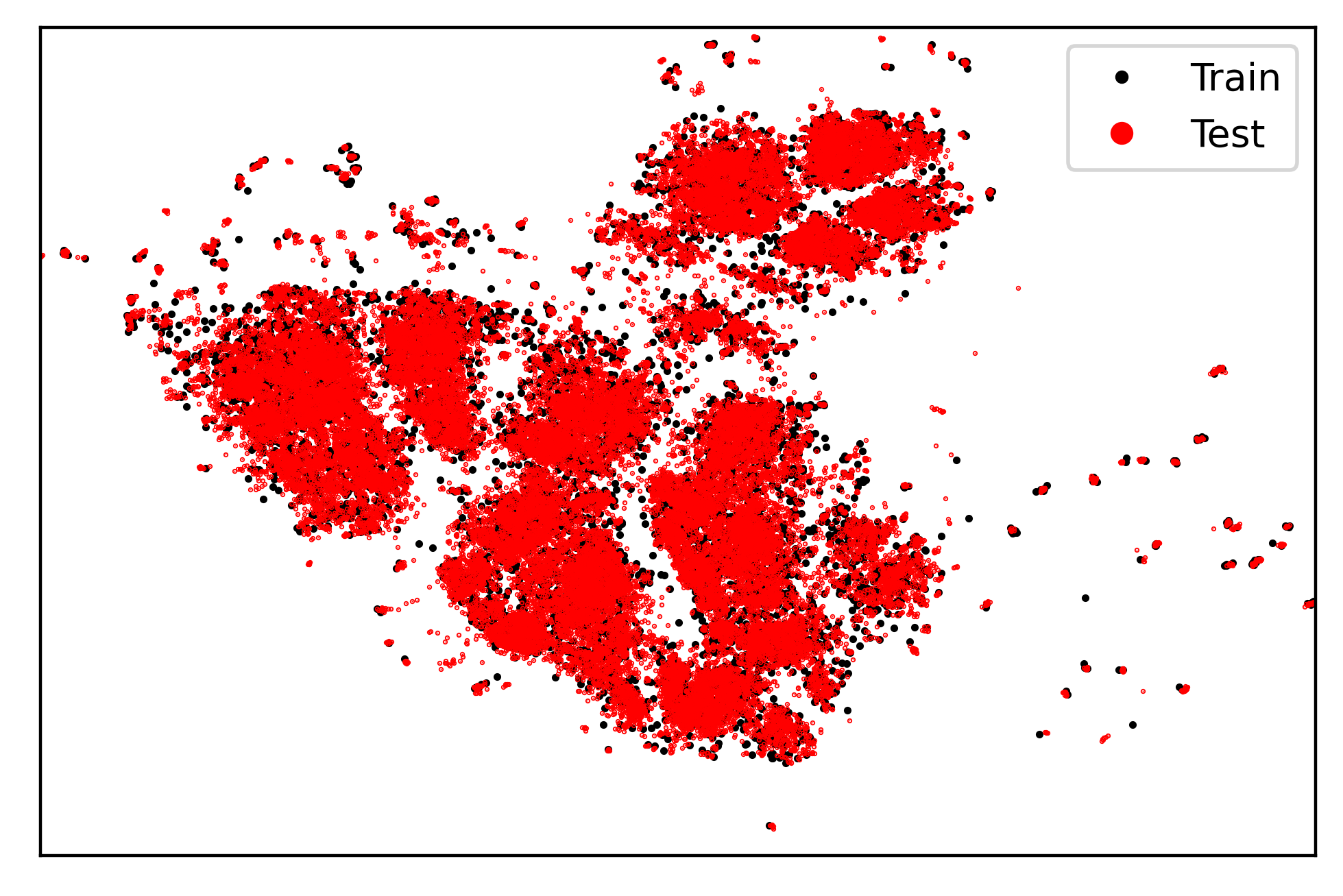} \\
    (a) \hspace{1.45in} (b) \hspace{1.45in} (c) \hspace{1.45in} (d) \\
    \caption{Repulsion effect in the clinical dataset for patients with `abdominal pain' as the chief complaint. (a) UMAP when training and test data are embedded together. (b-d) Embedded using training followed by test data is embedded as out-of-sample data using (b) UMAP showing repulsion effect, (c) UMAP embedding showing reduced repulsion effect when $n_s$ is set to 3 for test points, and (d) P-UMAP-CE. Both (c) and (d) show reduced repulsion effect.} \label{suppfig:repulsion_efect_yale_abdominal}
\end{figure*}

In the main text, we discussed that the minimum distance $m_d$ was set to $0$, $0.0001$, $0.001$, $0.01$, $0.1$, $0.5$, $0.8$, and $1.0$ for the UMAP embedding of the chief complaint `shortness of breath'. The corresponding trustworthiness values for them are shown in Fig.~\ref{suppfig:trustworthienss_yale_shortness}. The highest trustworthiness was obtained for $m_d=0.0001$.

Now, we look into another chief complaint, `Abdominal Pain'. 
Abdominal pain was responsible for $54,315$ hospital visits in the dataset, of which $19,482$ visits resulted in admission. After pre-processing, each visit corresponds to $1016$ dimensional vectors. 
The training data consists of $43,452$ visits, and the test data consists of $10,863$ visits. The 2-dimensional embeddings obtained from different algorithms are shown in Fig.~\ref{suppfig:abdominal_pain}. 
As previously mentioned, all the algorithms show a similar embedding structure. 
The points of admitted cases are clustered (blue) in different parts of the embedding. 
Similar to the `shortness of breath case', the trustworthiness of the embeddings is highest for P-UMAP-CE (Table~\ref{supptab:yale_data_table_abdominal_pain}). 
For the low number of nearest neighbors of 5, the trustworthiness of UMAP embedding is higher than P-UMAP-MSE and P-UMAP-CEMSE. 
The k-NN classifiers also show a similar trend in accuracy, with P-UMAP-CE showing the lowest errors.

Fig.~\ref{suppfig:repulsion_efect_yale_abdominal} shows the effect for `abdominal pain'. The plots are similar to what we obtained for `shortness of breath' in the main text. 
As usual, Fig.~\ref{suppfig:repulsion_efect_yale_abdominal}(a) shows the embedding when the training and test data are embedded together and naturally display no repulsion effect. 
The clumping up of test points in the periphery of the clusters appears when test points are embedded using the rule-based method (Fig.~\ref{suppfig:repulsion_efect_yale_abdominal}(b)), which is visible due to the dark boundary around the training embedding of lighter color. The dark boundary is also absent when a lower value of $n_s$ is used (Fig.~\ref{suppfig:repulsion_efect_yale_abdominal}(c)) as well as when parameterized UMAP is used (Fig.~\ref{suppfig:repulsion_efect_yale_abdominal}(d)).


\section{Additional Discussion Regarding Force Ratios}~\label{suppsec:supp_afr_rfr}

Fig.~\ref{suppfig:choiceofpoints} shows the out-of-sample points selected for analysis in section~4.4. The green region shows points outside the boundary, and the purple region shows points inside the boundary. 

\begin{figure*}[ht]
    \centering
    \includegraphics[width=0.3\linewidth]{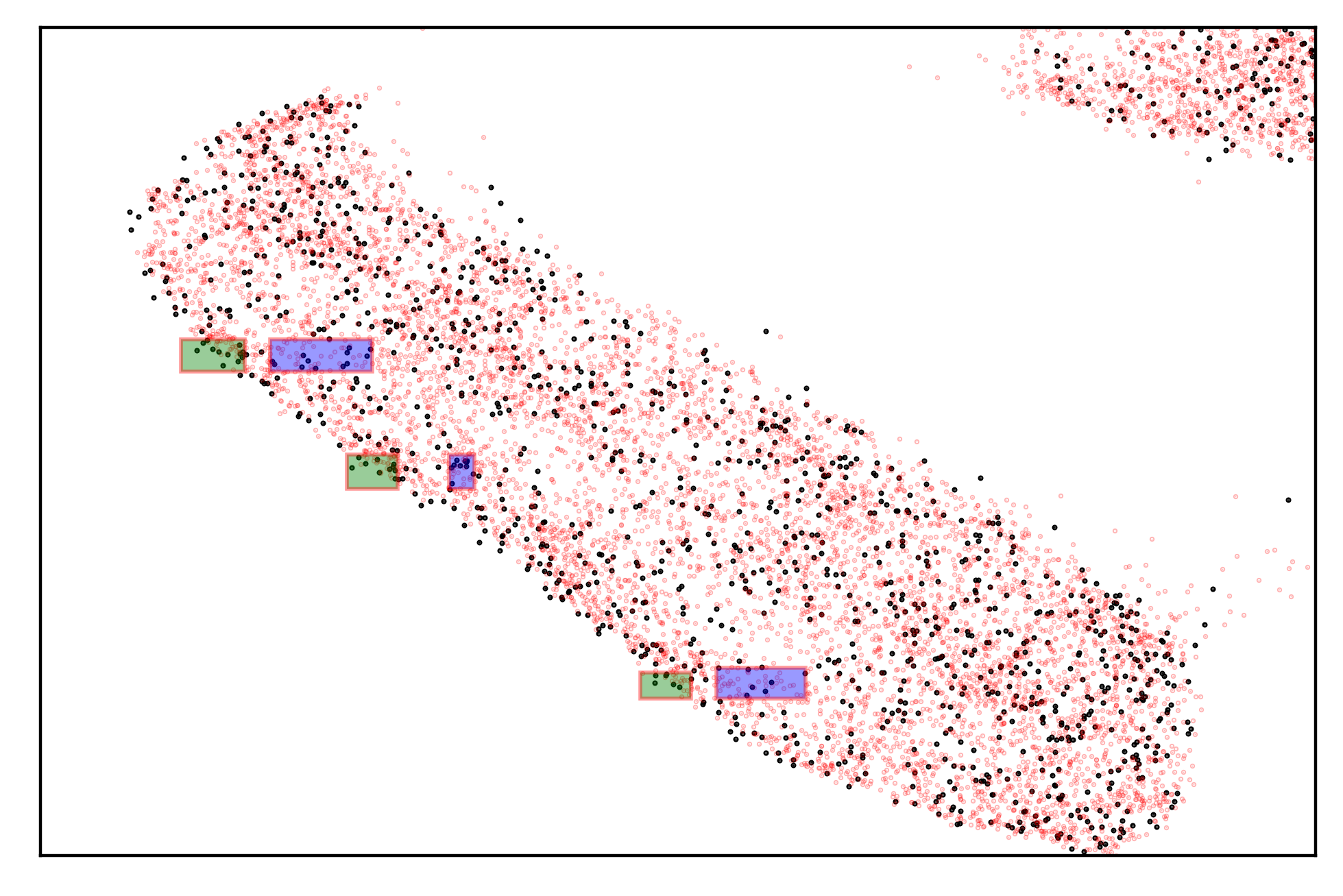}
    \includegraphics[width=0.3\linewidth]{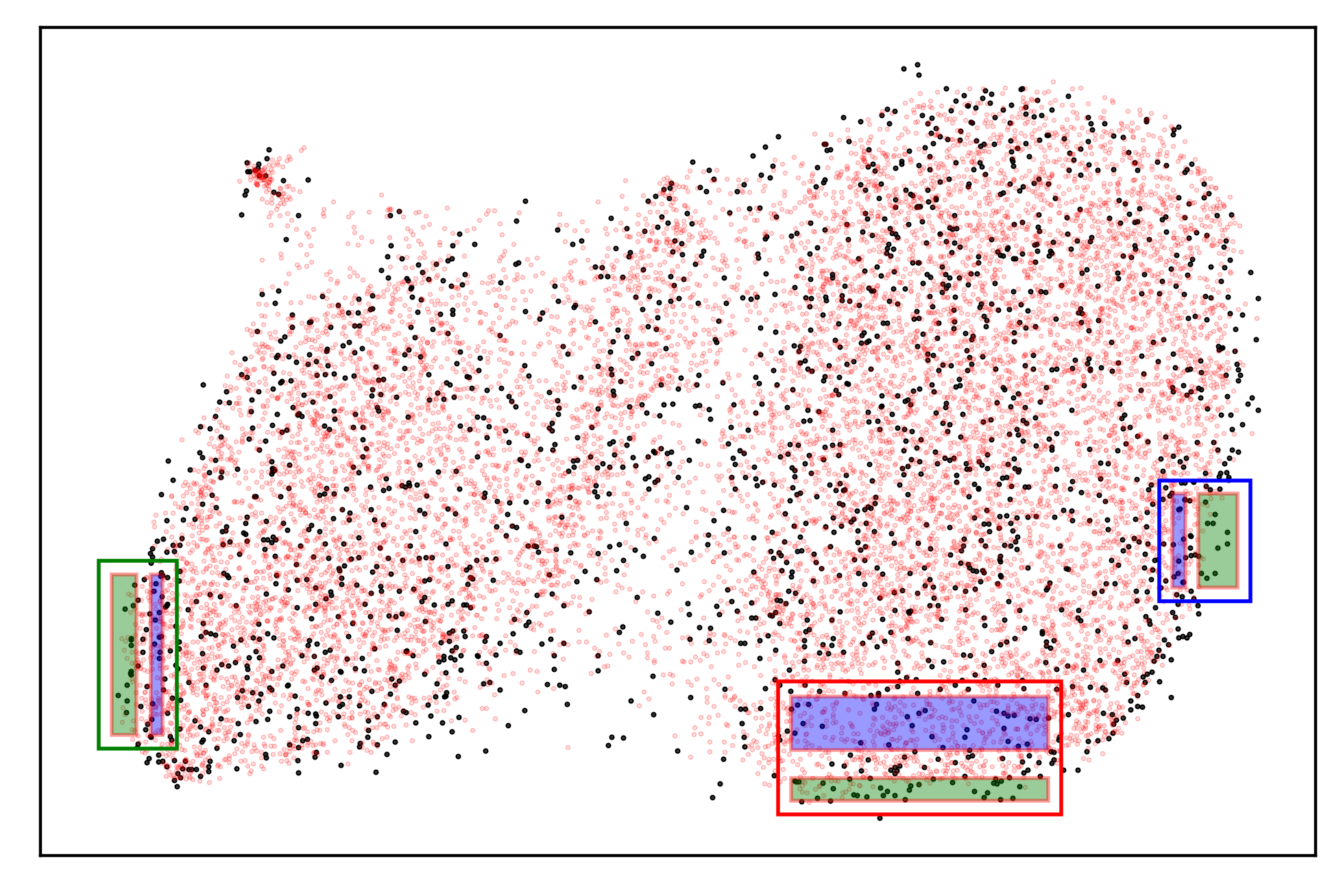}
    \includegraphics[width=0.3\linewidth]{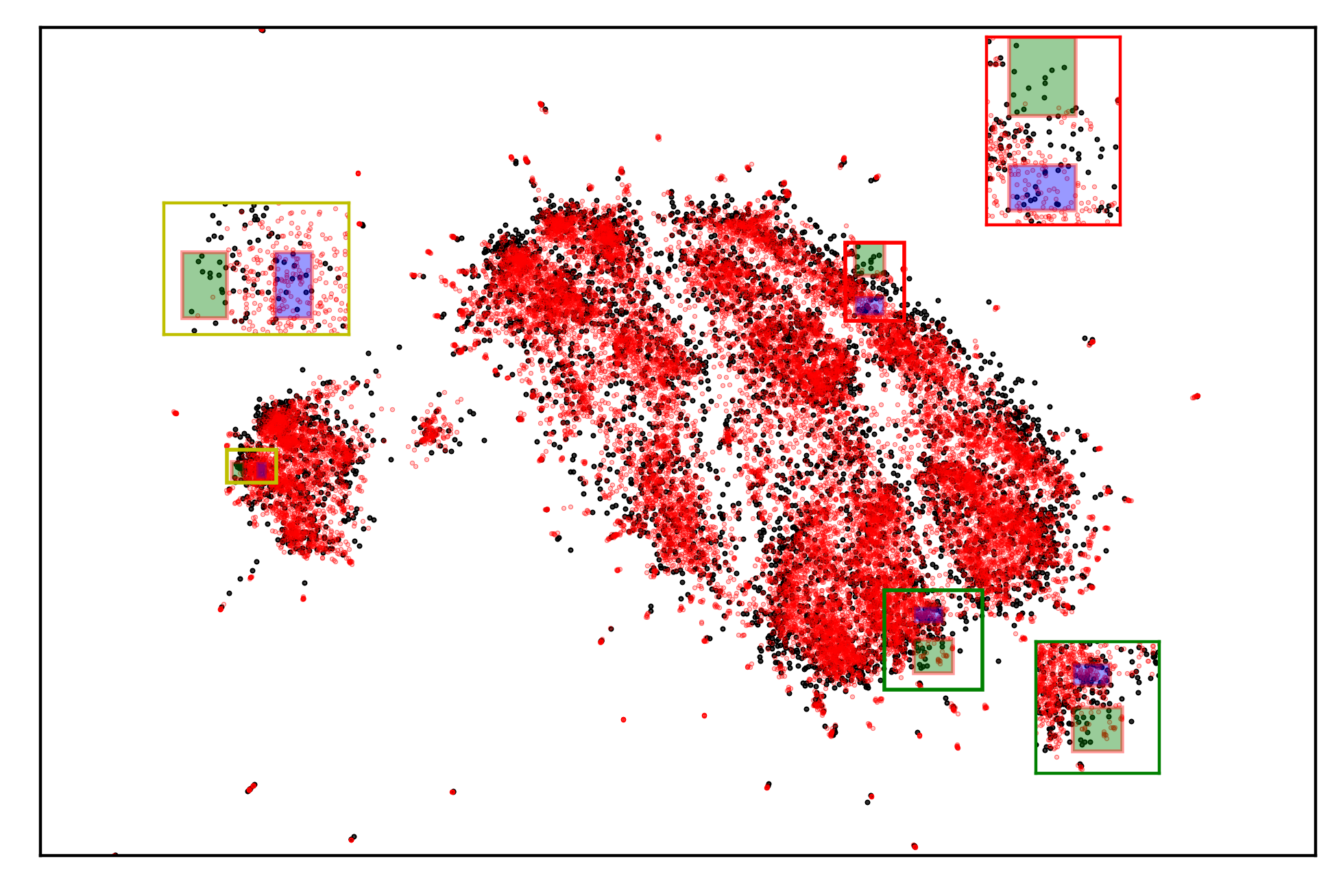} \\
    (a) \hspace{0.3\linewidth} (b) \hspace{0.3\linewidth} (c)
    \caption{Points chosen for analysis in Section~4.4. Green shading indicates points outside the cluster periphery ($S_1$) and purple shading indicates points inside the cluster periphery ($S_2$). (a) 58 out-of-sample points selected from MNIST dataset, (b) 112 out-of-sample points selected from chest x-ray data, (c) 80 out-of-sample points selected from clinical data.}
    \label{suppfig:choiceofpoints}
\end{figure*}

\end{document}